%% file: main.tex
\documentclass[10pt]{article} % For LaTeX2e
\usepackage[ruled,vlined]{algorithm2e}
\usepackage{algorithmic}

\usepackage[accepted]{tmlr}
\input{math_commands.tex}

\usepackage{hyperref}
\usepackage{url}
\usepackage{mathtools}
\usepackage{amsmath}
\usepackage{amsthm}
\usepackage{booktabs}
\usepackage{caption}
\usepackage{subfigure}
\usepackage{comment}

\newtheorem{theorem}{Theorem}
\newtheorem{lemma}{Lemma}
\newtheorem{remark}{Remark}
\newtheorem{corollary}{Corollary}
\newtheorem{assumption}{Assumption}

\usepackage{color}
\usepackage[normalem]{ulem}
\usepackage{xcolor}

\DeclareRobustCommand{\erase}{\bgroup\markoverwith{\textcolor{red}{\rule[.5ex]{2pt}{1.0pt}}}\ULon}

\title{Momentum Tracking: Momentum Acceleration for Decentralized Deep Learning on Heterogeneous Data}

\author{\name Yuki Takezawa \email yuki-takezawa@ml.ist.i.kyoto-u.ac.jp \\
      \addr Kyoto University, Okinawa Institute of Science and Technology
      \ANDTMLR
      \name Han Bao \email bao@i.kyoto-u.ac.jp \\
      \addr Kyoto University, Okinawa Institute of Science and Technology
      \ANDTMLR
      \name Kenta Niwa \email kenta.niwa.bk@hco.ntt.co.jp \\
      \addr NTT Communication Science Laboratories
      \ANDTMLR
      \name Ryoma Sato \email r.sato@ml.ist.i.kyoto-u.ac.jp \\
      \addr Kyoto University, Okinawa Institute of Science and Technology
      \ANDTMLR
      Makoto Yamada \email makoto.yamada@oist.jp \\
      \addr Okinawa Institute of Science and Technology
}

\def\month{9}  % Insert correct month for camera-ready version
\def\year{2023} % Insert correct year for camera-ready version
\def\openreview{\url{https://openreview.net/forum?id=8koy8QuTZD}} % Insert correct link to OpenReview for camera-ready version

\begin{document}

\maketitle

\begin{abstract}
SGD with momentum is one of the key components for improving the performance of neural networks.
For decentralized learning, a straightforward approach using momentum is Distributed SGD (DSGD) with momentum (DSGDm).
However, DSGDm performs worse than DSGD when the data distributions are statistically heterogeneous.
Recently, several studies have addressed this issue and proposed methods with momentum that are more robust to data heterogeneity than DSGDm, although their convergence rates remain dependent on data heterogeneity and
deteriorate when the data distributions are heterogeneous.
In this study, we propose Momentum Tracking, which is a method with momentum whose convergence rate is proven to be independent of data heterogeneity.
More specifically, we analyze the convergence rate of Momentum Tracking in the setting where the objective function is non-convex and the stochastic gradient is used.
Then, we identify that it is independent of data heterogeneity for any momentum coefficient $\beta \in [0, 1)$.
Through experiments, we demonstrate that Momentum Tracking is more robust to data heterogeneity than the existing decentralized learning methods with momentum and can consistently outperform these existing methods when the data distributions are heterogeneous.
\end{abstract}

\section{Introduction}
%Decentralized Learningの導入
Neural networks have achieved remarkable success in various fields 
such as image processing \citep{simonyanZ2014very,chen2020simple} and
natural language processing \citep{devlin2019bert}.
To train neural networks, we need to collect large amounts of training data,
but it is often difficult to collect large amounts of data such as medical images on one server because of privacy concerns.
In such scenarios, decentralized learning has attracted significant attention because it allows us to train neural networks without aggregating all the data onto one server.
Recently, decentralized learning has been studied from various perspectives, including data heterogeneity \citep{tang2018d2,esfandiari2021cross}, communication compression \citep{tang2018communication,lu2020moniqua,liu2021linear,takezawa2022communication}, and network topologies \citep{ying2021exponential,le2023refined}.

%Momentumの重要性
One of the key components for improving the performance of neural networks is SGD with momentum (SGDm).
Whereas SGD updates the model parameters using a stochastic gradient, 
SGDm updates the model parameters using the moving average of the stochastic gradient, which is called the momentum.
Because SGDm can accelerate convergence and improve generalization performance,
SGDm has become an indispensable tool, enabling neural networks to achieve high accuracy \citep{he2016deep,cutkosky2020momentum,karimireddy2021breaking,defazio2021momentum}.
Recently, SGDm has been improved in many studies, and methods such as Adam \citep{kingma2015adam} and RAdam \citep{liu2020On} have been proposed.

% 単純にDecentralized Learningに組み込む(i.e., DSGDm)と悪化する
In decentralized learning, the straightforward approach to using the momentum is Distributed SGD (DSGD) with momentum (DSGDm) \citep{gao2020periodic}.
When the data distributions held by each node (i.e., the server) are statistically homogeneous, 
DSGDm works well and can improve the performance as well as SGDm \citep{lin21quasi}.
However, in real-world decentralized learning settings, the data distributions may be heterogeneous \citep{hsieh2020noniid}.
%In such cases, DSGDm has been shown to underperform DSGD (i.e., without momentum) \citep{yuan202decentlam}.
In such cases, DSGDm performs worse than DSGD (i.e., without momentum) \citep{yuan202decentlam}.

%既存のDecentralized Learning with momentum
This is because, when the data distributions are heterogeneous and we use the momentum instead of the stochastic gradient, 
each model parameter is updated in further different directions and drifts away more easily.
As a result, the convergence rate of DSGDm falls below that of DSGD.
To address this issue,
\citet{lin21quasi} and \citet{yuan202decentlam} modified the update rules of the momentum in DSGDm
and proposed methods that are more robust to data heterogeneity than DSGDm.
However, their convergence rates remain dependent on data heterogeneity,
and our experiments revealed that their performance are degraded when the data distributions are strongly heterogeneous (Sec. \ref{sec:experiment}).
%accuracy decreases when the data distributions are strongly heterogeneous.

% GTとかヘテロに強い手法あるけど、SGDでしか解析しておらず、momentum使ったときどうなるかはわからない
Data heterogeneity for decentralized learning has been well studied from both experimental and theoretical perspectives \citep{hsieh2020noniid,koloskova2020unified}.
Subsequently, many methods including Gradient Tracking \citep{lorenzo2016next,nedic2017achieving} have been proposed
and it has been shown that their convergence rates do not depend on data heterogeneity \citep{tang2018d2,vogels2021relaysum,koloskova2021an}.
However, these studies considered only the case where the momentum was not used,
and it remains unclear whether these methods are robust to data heterogeneity when the momentum is applied.

In the convex optimization literature,
\citet{xin2020distributed} and \citet{carnevale2022gtadam} proposed combining Gradient Tracking with momentum or Adam
and analyzed the convergence rates.
However, they considered only the case where the objective function is strongly convex and the full gradient is used,
which does not hold in the standard deep learning setting, where the objective function is non-convex and only the stochastic gradient is accessible.
Hence, their convergence rates are still unknown 
in the setting where the objective function is non-convex and the stochastic gradient is used
and it remains unclear whether their convergence rates are independent of data heterogeneity.
Furthermore, they did not discuss data heterogeneity, either theoretically or experimentally.

% 本研究では...
In this work, we propose a decentralized learning method with momentum, which we call \textbf{Momentum Tracking}, whose convergence rate is proven to be independent of data heterogeneity 
in the setting where the objective function is non-convex and the stochastic gradient is used.
More specifically, 
we identify that the convergence rate of Momentum Tracking is independent of data heterogeneity for any momentum coefficient $\beta \in [0, 1)$.
In Table~\ref{table:summary}, we compare the convergence rate of Momentum Tracking with those of existing methods.
To the best of our knowledge, Momentum Tracking is the first decentralized learning method with momentum whose convergence rate has been proven to be independent of data heterogeneity 
in the setting where the objective function is non-convex and the stochastic gradient is used.
Experimentally, we demonstrate that Momentum Tracking is more robust to data heterogeneity than the existing decentralized learning methods with momentum
and can consistently outperform these existing methods when the data distributions are heterogeneous.
\begin{table*}[!t]
\vskip -0.1 in
\caption{Comparison of the convergence rates. In the ``Data-Heterogeneity'' column, ``\checkmark'' indicates that the convergence rate is independent of data heterogeneity, and ``(\checkmark)'' indicates that it is independent, but there is no discussion about data heterogeneity either theoretically or experimentally. In the ``Momentum,'' ``Stochastic,'' and ``Non-Convex'' columns, ``\checkmark'' respectively indicates that the method is accelerated using momentum, the convergence rate is provided when the stochastic gradient is used, and the convergence rate is provided when the objective function is non-convex.}
\label{table:summary}
\centering
\vskip -0.1 in
\resizebox{\linewidth}{!}{
\begin{tabular}{ccccc}
\toprule
                                          & Data-Heterogeneity & Momentum & Stochastic & Non-Convex \\
\midrule
DSGD \citep{lian2017can}                  &              &            &\checkmark & \checkmark \\
Gradient Tracking \citep{koloskova2021an} & \checkmark   &            &\checkmark & \checkmark \\ 
DSGDm \citep{gao2020periodic}             &              & \checkmark &\checkmark & \checkmark \\
QG-DSGDm \citep{lin21quasi}               &              & \checkmark &\checkmark & \checkmark \\ 
DecentLaM \citep{yuan202decentlam}        &              & \checkmark &\checkmark & \checkmark \\
ABm \citep{xin2020distributed}            & (\checkmark) & \checkmark &           &            \\
GTAdam \citep{carnevale2022gtadam}        & (\checkmark) & \checkmark &           &            \\
\textbf{Momentum Tracking (our work)}     & \checkmark   & \checkmark &\checkmark & \checkmark \\
\bottomrule
\end{tabular}
}
\vskip -0.1 in
\end{table*}

\section{Preliminaries and Related Work}
\subsection{Decentralized Learning}
\label{sec:decentralized_learning}
Let $G = (V, E)$ be an undirected graph that represents the underlying network topology, where $V$ denotes the set of nodes and $E$ denotes the set of edges. 
Let $N \coloneqq |V|$ be the number of nodes, and we label each node in $V$ by a set of integers $\{ 1, 2, \cdots, N \}$ for simplicity.
We define $\mathcal{N}_i \coloneqq \{ j \in V \mid (i,j) \in E \}$ as the set of neighbor nodes of node $i$ and define $\mathcal{N}_i^{+} \coloneqq \mathcal{N}_i \cup \{ i \}$.
In decentralized learning, node $i$ has a local data distribution $\mathcal{D}_i$ and local objective function $f_i : \mathbb{R}^d \rightarrow \mathbb{R}$,
and can communicate with node $j$ if and only if $(i, j) \in E$.
Then, decentralized learning aims to minimize the average of the local objective functions as follows:
\begin{align*}
    \min_{\vx \in \mathbb{R}^d} \left[ f(\vx) \coloneqq \frac{1}{N} \sum_{i=1}^N f_i (\vx) \right], 
    f_i (\vx) \coloneqq \mathbb{E}_{\xi_i\sim\mathcal{D}_i} \left[ F_i (\vx ; \xi_i) \right],
\end{align*}
where $\vx$ is the model parameter,
$\xi_i$ is the data sample that follows $\mathcal{D}_i$,
and local objective function $f_i (\vx)$ is defined as the expectation
of $F_i(\vx ; \xi_i)$ over data sample $\xi_i$.
In the following, $\nabla F_i(\vx ; \xi_i)$ and $\nabla f_i (\vx) \coloneqq \mathbb{E}_{\xi_i \sim \mathcal{D}_i}[\nabla F_i (\vx;\xi_i)]$ denote the stochastic and full gradient respectively.

Distributed SGD (DSGD) \citep{lian2017can} is one of the most well-known algorithms for decentralized learning.
Formally, the update rules of DSGD are defined as follows:
\begin{align}
    \vx_i^{(r+1)} = \sum_{j \in \mathcal{N}_i^{+}} W_{ij} \left( \vx_j^{(r)} - \eta \nabla F_j (\vx_j^{(r)} ; \xi_j^{(r)}) \right),
\end{align}
where $\eta > 0$ is the step size and $W_{ij} \in [0,1]$ is the weight of edge $(i, j)$.
Let $\mW \in [0, 1]^{N\times N}$ be the matrix whose $(i,j)$-element is $W_{ij}$ if $(i,j) \in E$ and $0$ otherwise.
In general, a mixing matrix is used for $\mW$ (i.e., $\mW = \mW^\top$, $\mW \mathbf{1} = \mathbf{1}$, and $\mW^\top \mathbf{1} = \mathbf{1}$).
\citet{lian2018asynchronous} extended DSGD in the case where each node communicates asynchronously and analyzed the convergence rate.
\citet{koloskova2020unified} analyzed the convergence rate of DSGD when the network topology changes over time.
These results revealed that the convergence rate of DSGD deteriorates
and the performance is degraded when the data distributions held by each node are statistically heterogeneous.
This is because the local gradients $\nabla f_i$ are different across nodes and each model parameter $\vx_i$ tends to drift away when the data distributions are heterogeneous.
To address this issue, $D^2$ \citep{tang2018d2}, Gradient Tracking \citep{lorenzo2016next,nedic2017achieving}, and primal-dual algorithms \citep{niwa2020edge,niwa2021asynchronous,takezawa2022theoretical} were proposed to correct the local gradient $\nabla f_i$ to the global gradient $\nabla f$.
As a different approach, \citet{vogels2021relaysum} proposed a novel averaging method to prevent each model parameter $\vx_i$ from drifting away.
It has been shown that the convergence rates of these methods do not depend on data heterogeneity and do not deteriorate, even when the data distributions are statistically heterogeneous.
However, these methods do not consider the case in which momentum is used.

\subsection{Momentum}
\label{sec:momentum}

% 通常の学習の時の、momentum SGDについて紹介
The methods with momentum were originally proposed by \citet{polyk1941some},
and SGD with momentum (SGDm) has achieved successful results in training neural networks \citep{simonyanZ2014very,he2016deep,Wang2020escaping}.
% Momentum for Decentralized Learning.
In decentralized learning,
a straightforward approach to using the momentum is DSGD with momentum (DSGDm) \citep{gao2020periodic}.
The update rules of DSGDm are defined as follows:
\begin{align}
    \label{eq:dsgdm_1}
    \vu_i^{(r+1)} &= \beta \vu_i^{(r)} + \nabla F_i (\vx_i^{(r)} ; \xi_i^{(r)}), \\
    \label{eq:dsgdm_2}
    \vx_i^{(r+1)} &= \sum_{j \in \mathcal{N}_i^{+}} W_{ij} \left( \vx_j^{(r)} - \eta \vu_j^{(r+1)} \right),
\end{align}
where $\vu_i$ is the local momentum of node $i$ and $\beta \in [0, 1)$ is a momentum coefficient.
%\cite{yu2019linear} proposed to further compute the average for the momentum $\vu_i$ as well as the model parameter $\vx_i$.
In addition, several variants of DSGDm were studied by \citet{yu2019linear,assran2019stochstic,Wang2020SlowMo,singh2021communication}.
When the data distributions held by each node are statistically homogeneous, DSGDm works well and can improve accuracy as well as SGDm. 
However, when the data distributions are statistically heterogeneous, DSGDm leads to poorer performance than DSGD.
This is because when the data distributions held by each node are statistically heterogeneous (i.e., $\nabla f_i$ varies significantly across nodes),
the difference in the updated value of the model parameter across the nodes (i.e., $\eta \vu_i$) is amplified by the momentum \citep{lin21quasi}.

To address this issue, \citet{yuan202decentlam} and \citet{lin21quasi} proposed methods to modify the update rules of the momentum in DSGDm 
such that the momentum of each node has close values,
which are called DecentLaM and QG-DSGDm, respectively. 
They further experimentally demonstrated that these methods are more robust to data heterogeneity than DSGDm. 
However, their convergence rates have been shown to still depend on data heterogeneity
and deteriorate when the data distributions are heterogeneous.

\subsection{Gradient Tracking}
\label{sec:gradient_tracking}
One of the most well-known methods whose convergence rate does not depend on data heterogeneity is Gradient Tracking \citep{lorenzo2016next}.
Whereas DSGD exchanges only the model parameter $\vx_i$, 
Gradient Tracking exchanges the model parameter $\vx_i$ and local (stochastic) gradient $\nabla f_i$
and then updates the model parameters while estimating global gradient $\nabla f$.
\citet{nedic2017achieving} and \citet{qu2018harnessing} analyzed the convergence rate of Gradient Tracking when the objective function is (strongly) convex and the full gradient is used.
\citet{pu2021distributed} analyzed the convergence rate when the objective function is strongly convex and the stochastic gradient is used.
Recently, \citet{koloskova2021an} analyzed the convergence rates of Gradient Tracking 
in the setting where the objective function is non-convex and the stochastic gradient is used.
There is also a line of research to combine Gradient Tracking with variance reduction methods \citep{xin2022fast} and communication compression methods \citep{zhao2022beer}.
They showed that the convergence rate of Gradient Tracking does not depend on data heterogeneity.
However, these studies only consider the case without momentum,
and the convergence analysis for Gradient Tracking with momentum has not been explored thus far in the aforementioned studies.

\section{Proposed Method}
\label{sec:proposed}

In this section, we propose \textbf{Momentum Tracking}, which is a decentralized learning method with momentum whose convergence rate is proven to be independent of the data heterogeneity 
in the setting where the objective function is non-convex and the stochastic gradient is used.

\subsection{Setup}
We assume that the following standard assumptions hold:

\begin{assumption}
\label{assumption:lower_bound} 
There exists a constant $f^\star > - \infty$ that satisfies $f(\vx) \ge f^\star$ for all $\vx \in \mathbb{R}^d$.
\end{assumption}
\begin{assumption}
\label{assumption:mixing}
There exists a constant $p \in (0, 1]$ that satisfies for all $\vx_1, \cdots, \vx_N \in \mathbb{R}^{d}$,
\begin{align}
\label{eq:mixing}
    \| \mX \mW - \bar{\mX} \|^2_F \leq (1 - p) \| \mX - \bar{\mX} \|^2_F,
\end{align}
where $\mX \coloneqq (\vx_1, \cdots, \vx_N) \in \mathbb{R}^{d\times N}$ and $\bar{\mX} \coloneqq \frac{1}{N} \mX \mathbf{1}\mathbf{1}^\top$.
\end{assumption}
\begin{assumption}
\label{assumption:smoothness}
There exists a constant $L>0$ that satisfies for all $i \in V$ and $\vx, \vy \in \mathbb{R}^d$,
\begin{align}
\label{eq:smoothness}
    \| \nabla f_i (\vx) - \nabla f_i (\vy) \| \leq L \| \vx - \vy \|.
\end{align}
\end{assumption}
\begin{assumption}
\label{assumption:stochasic_gradient}
There exists a constant $\sigma^2$ that satisfies for all $i \in V$ and $\vx_i \in \mathbb{R}^d$,
\begin{align}
\label{eq:stochastic_gradient}
    \mathbb{E}_{\xi_i \sim \mathcal{D}_i} \| \nabla F_i(\vx_i ; \xi_i) - \nabla f_i (\vx_i) \|^2 \leq \sigma^2.
\end{align}
\end{assumption}

Assumptions \ref{assumption:lower_bound}, \ref{assumption:mixing}, \ref{assumption:smoothness}, and \ref{assumption:stochasic_gradient} are commonly used for decentralized learning algorithms \citep{lian2017can,yu2019linear,koloskova2021an,lin21quasi,lu2021optimal,yuan2022revisiting}.
Additionally, the following assumption, which represents data heterogeneity,
is commonly used in the convergence analysis of decentralized learning algorithms \citep{lian2017can,yu2019linear,lin21quasi}.
\begin{assumption}
\label{assumption:heterogeneity}
There exists a constant $\zeta^2$ that satisfies for all $\vx \in \mathbb{R}^d$,
\begin{align*}
    \frac{1}{N} \sum_{i=1}^N \left\| \nabla f_i (\vx) - \nabla f (\vx) \right\|^2 \leq \zeta^2.
\end{align*}
\end{assumption}
Under Assumption~\ref{assumption:heterogeneity}, the convergence rates of DSGD~\citep{lian2017can}, DSGDm~\citep{gao2020periodic,yuan202decentlam}, QG-DSGDm~\citep{lin21quasi}, and DecentLaM~\citep{yuan202decentlam} were shown to be dependent on data heterogeneity $\zeta^2$ and deteriorate as $\zeta^2$ increases.
By contrast, in Sec. \ref{sec:convergence_analysis}, we prove that Momentum Tracking converges without Assumption \ref{assumption:heterogeneity}
and the convergence rate is independent of data heterogeneity $\zeta^2$.
In addition, we do not assume the convexity of the objective functions $f(\vx)$ and $f_i (\vx)$. 
Therefore, $f(\vx)$ and $f_i (\vx)$ are potentially non-convex functions (e.g., the loss functions of neural networks).

\subsection{Momentum Tracking}

In this section, we propose \textbf{Momentum Tracking}, which is robust to data heterogeneity and accelerated by the momentum.
The update rules of Momentum Tracking are defined as follows:
\begin{align}
    \label{eq:mt_1}
    \vu_i^{(r+1)} &= \beta \vu_i^{(r)} + \nabla F_i(\vx_i^{(r)} ; \xi_i^{(r)}), \\
    \label{eq:mt_2}
    \vx_i^{(r+1)} &= \sum_{j \in \mathcal{N}_i^+} W_{ij} \vx_j^{(r)} - \eta \left( \vu_i^{(r+1)} - \vc_i^{(r)} \right), \\
    \label{eq:mt_3}
    \vc_i^{(r+1)} &= \sum_{j \in \mathcal{N}_i^{+}} W_{ij} \left( \vc_j^{(r)} - \vu_j^{(r+1)} \right) + \vu_i^{(r+1)},
\end{align}
where $\beta \in [0, 1)$ is a momentum coefficient.
The pseudo-code for Momentum Tracking is presented in Sec.~\ref{sec:psuedo}.
In Momentum Tracking, $\vc_i$ corrects the local momentum $\vu_i$ to the global momentum $\frac{1}{N} \sum_j \vu_j$
and prevents each model parameter $\vx_i$ from drifting, even when the data distributions are statistically heterogeneous (i.e., the local momentum $\vu_i$ varies significantly across nodes).

Because Momentum Tracking is equivalent to Gradient Tracking when $\beta=0$,
Momentum Tracking is a simple extension of Gradient Tracking.
Hence, when $\beta=0$, it has been shown that the convergence rate of Momentum Tracking is independent of data heterogeneity $\zeta^2$ \citep{koloskova2021an}.
However, because data heterogeneity is amplified when the momentum is used instead of the stochastic gradient (i.e., $\beta>0$) \citep{lin21quasi,yuan202decentlam},
it is unclear whether the convergence rate of Momentum Tracking is independent of data heterogeneity $\zeta^2$ for any $\beta \in [0, 1)$
or for only a restricted range of $\beta$.
In Sec.~\ref{sec:convergence_analysis}, we provide the convergence rate of Momentum Tracking and prove that it is independent of $\zeta^2$ for any $\beta \in [0, 1)$.

\subsection{Convergence Analysis}
\label{sec:convergence_analysis}

Under Assumptions \ref{assumption:lower_bound}, \ref{assumption:mixing}, \ref{assumption:smoothness}, and \ref{assumption:stochasic_gradient},
Theorem \ref{theorem:convergence_rate} provides the convergence rate of Momentum Tracking.
%in the standard deep learning setting.
All proofs are presented in Sec. \ref{sec:proof}.

\begin{theorem}
\label{theorem:convergence_rate}
Suppose that Assumptions \ref{assumption:lower_bound}, \ref{assumption:mixing}, \ref{assumption:smoothness}, and \ref{assumption:stochasic_gradient} hold,
each model parameter $\vx_i$ is initialized with the same parameters,
and both $\vu_i$ and $\vc_i$ are initialized as $\frac{1}{1 - \beta} (\nabla F_i (\vx_i^{(0)} ; \xi_i^{(0)}) - \frac{1}{N} \sum_{j=1}^N \nabla F_j (\vx_j^{(0)} ; \xi_j^{(0)}))$. 
Then, for any $\beta \in [0, 1)$ and $R \geq 1$,
there exists a step size $\eta$ such that the average parameter $\frac{1}{R} \sum_{r=0}^{R-1} \mathbb{E} \left\| \nabla f (\bar{\vx}^{(r)}) \right\|^2$ generated by Eqs.~(\ref{eq:mt_1}-\ref{eq:mt_3}) is bounded from above by
\begin{align}
\label{eq:rate_of_mt}
    \mathcal{O} \Biggl( \sqrt{ \frac{r_0 \sigma^2 L }{N R} }
    + \left( \frac{r_0^2 \sigma^2 L^2}{p^4 R^2 (1-\beta)} \left( 1 + \frac{p \beta^2}{1-\beta} \right) \right)^{\frac{1}{3}}
    + \frac{L r_0}{(1-\beta) p^2 R} \sqrt{1 + \frac{\beta^2}{(1-\beta^2)^3 p}} \Biggr),
\end{align}
where $\bar{\vx} \coloneqq \frac{1}{N} \sum_{i=1}^N \vx_i$ and $r_0 \coloneqq f(\bar{\vx}^{(0)}) - f^\star$.
\end{theorem}

\begin{remark}
Theorem \ref{theorem:convergence_rate} assumes that $\vu_i$ and $\vc_i$ are initialized as $\frac{1}{1 - \beta} (\nabla F_i (\vx_i^{(0)} ; \xi_i^{(0)}) - \frac{1}{N} \sum_{j=1}^N \nabla F_j (\vx_j^{(0)} ; \xi_j^{(0)}))$.
Thus, All-Reduce is required only once before starting the training.
If we initialize $\vu_i$ and $\vc_i$ as zeros, data heterogeneity at initial parameters $\frac{1}{N} \sum_i \| \nabla f_i (\vx_i^{(0)}) - \nabla f (\vx_i^{(0)}) \|^2$ appears in the convergence rate,
but the same phenomenon occurs in the analysis of Gradient Tracking by \citet{koloskova2021an} (see Sec. \ref{sec:additional_discussion_of_rate}).
\end{remark}

\begin{remark}
\label{remark:abm}
Combinations of Gradient Tracking with the momentum or Adam have also been proposed by \citet{xin2020distributed} and \citet{carnevale2022gtadam}.
However, they considered only the setting in which the objective function is strongly convex and the full gradient is used.
By contrast, our study focuses on the deep learning setting.
Hence, our proof strategies are completely different from those in these previous studies,
and Theorem \ref{theorem:convergence_rate} provides the convergence rate in the setting where the objective function is non-convex and the stochastic gradient is used.
\end{remark}

\begin{remark}
\citet{koloskova2021an} provided the convergence rate of Gradient Tracking in the setting where the objective function is non-convex and the stochastic gradient is used.
However, they did not consider the case where the momentum is used, and it is not trivial to provide the convergence rate of Momentum Tracking from the results in this previous work.
\end{remark}

\subsection{Discussion}
\label{sec:discussion}

\textbf{Comparison with Gradient Tracking:}
Theorem \ref{theorem:convergence_rate} indicates that the convergence rate of Momentum Tracking does not depend on data heterogeneity $\zeta^2$ for any $\beta \in [0, 1)$ 
and does not deteriorate even when the data distributions are statistically heterogeneous (i.e., $\zeta^2>0$).
Therefore, Theorem~\ref{theorem:convergence_rate} indicates that Momentum Tracking is theoretically robust to data heterogeneity for any $\beta \in [0, 1)$.
Although Momentum Tracking is a simple extension of Gradient Tracking, our work is the first to identify that the combination of Gradient Tracking and the momentum converges without being affected by data heterogeneity $\zeta^2$ for any $\beta \in [0,1)$ 
in the setting where the objective function is non-convex and the stochastic gradient is used.

\begin{comment}
Because the convergence rate of Momentum Tracking Eq.~(\ref{eq:rate_of_mt}) is optimal when $\beta=0$,
Theorem~\ref{theorem:convergence_rate} does not show that the convergence rate is improved by using the momentum.
However, the convergence rates of DSGDm and QG-DSGDm provided by \citet{gao2020periodic} and \citet{lin21quasi} are also optimal when $\beta=0$.
Moreover, they do not provide theoretical results that are consistent with the experimental results that the convergence rates are improved when $\beta>0$.
As in these studies, we experimentally demonstrate that convergence is accelerated when $\beta > 0$ in Sec. \ref{sec:experiment}
and leave for future work to show the theoretical benefits of using $\beta>0$.
\end{comment}
Although the convergence rate of Momentum Tracking Eq.~(\ref{eq:rate_of_mt}) is minimized when $\beta=0$, 
Momentum Tracking does accelerate its convergence with the momentum being used ($\beta>0$), as experimentally demonstrated in Sec.~\ref{sec:experiment}. 
Indeed, the convergence rates of DSGDm \citep{gao2020periodic} and QG-DSGDm \citet{lin21quasi} have the same issue.
Thus, it is still an open question to show the theoretical benefits of using $\beta>0$.

\textbf{Comparison with Existing Algorithms with momentum:}
Next, we compare the convergence rate of Momentum Tracking with those of existing decentralized learning algorithms with momentum:
DSGDm \citep{gao2020periodic}, DecentLaM \citep{yuan202decentlam}, and QG-DSGDm \citep{lin21quasi}.
Here, we only show the convergence rate of QG-DSGDm, but the same discussion holds for the other methods.
The convergence rate of QG-DSGDm is as follows:

\begin{theorem}[\citet{lin21quasi}]
\label{theorem:qg_dsgdm}
Suppose that Assumptions \ref{assumption:lower_bound}, \ref{assumption:mixing}, \ref{assumption:smoothness}, and \ref{assumption:stochasic_gradient} hold, and Assumption \ref{assumption:heterogeneity} also holds.
Then, for any $\beta \in [0, \frac{p}{21+p}]$ and $R \geq 1$, 
there exists a step size $\eta$ such that $\frac{1}{R} \sum_{r=0}^{R-1} \mathbb{E} \left\| \nabla f (\bar{\vx}^{(r)}) \right\|^2$ generated by QG-DSGDm is bounded from above by\footnote{For simplicity, we set the additional hyperparameter $\mu$ for QG-DSGDm to $\beta$.}
\begin{align*}
    \mathcal{O} \Biggl( \sqrt{\frac{r_0 \sigma^2 L}{N R}} 
    + \left( \frac{r_0^2 L^2 (\zeta^2 + \sigma^2)}{p^2 R^2} \right)^{\frac{1}{3}} 
    + \frac{L r_0}{R} \left( \frac{1}{p} + \frac{1}{1-\beta} + \frac{\beta}{(1-\beta)^3} \right) \Biggr),
\end{align*}
where $r_0 \coloneqq f(\bar{\vx}^{(0)}) - f^\star$.
\end{theorem}
Data heterogeneity $\zeta^2$ appears in the second term.
Thus, the convergence rate of QG-DSGDm deteriorates when the data distributions held by each node are statistically heterogeneous.
%and the convergence rate of QG-DSGDm depends on data heterogeneity $\zeta^2$.
%Therefore, the convergence rate of QG-DSGDm decreases when the data distributions held by each node are statistically heterogeneous.
By contrast, the convergence rate of Momentum Tracking Eq.~(\ref{eq:rate_of_mt}) does not depend on data heterogeneity $\zeta^2$.
Therefore, Momentum Tracking is more robust to data heterogeneity than QG-DSGDm.
Because the convergence rates of DSGDm and DecentLaM also depend on $\zeta^2$, the same discussion holds for DSGDm and DecentLaM. 
Hence, Momentum Tracking is more robust to data heterogeneity than these methods.
To the best of our knowledge, Momentum Tracking is the first decentralized learning method with momentum whose convergence rate has been proven to be independent of data heterogeneity $\zeta^2$ 
in the setting where the objective function is non-convex and the stochastic gradient is used.

Next, we discuss the range of $\beta$.
The convergence rates of QG-DSGDm and DecentLaM provided by \citet{lin21quasi} and \citet{yuan202decentlam} hold only when the range of $\beta$ is restricted. 
For instance, Theorem~\ref{theorem:qg_dsgdm} assumes that $\beta \leq \frac{p}{21+p} (< 0.05)$.
However, these restrictions on the range of $\beta$ do not hold in practice.
(Typically, $\beta$ is set to $0.9$.)
Therefore, the convergence rates of QG-DSGDm and DecentLaM are unclear in such practical cases.
By contrast, Theorem~\ref{theorem:convergence_rate} can provide the convergence rate of Momentum Tracking that holds for any $\beta \in [0, 1)$.

\textbf{Comparison with SGDm:}
Next, we compare the convergence rate of Momentum Tracking with that of SGDm.
In a setting where the objective function is non-convex and the stochastic gradient is used, SGDm has been proven to converge to the stationary point with $\mathcal{O}(1 / \sqrt{R})$ \citep{yan2018unified,liu2020improved}.
By contrast, Theorem \ref{theorem:convergence_rate} indicates that if the number of rounds $R$ is sufficiently large, Momentum Tracking converges with $\mathcal{O}(1 / \sqrt{N R})$. Therefore, Momentum Tracking can achieve a linear speedup with respect to the number of nodes $N$,
which is a common and important property in decentralized learning methods \citep{lian2018asynchronous,koloskova2020unified}.

\section{Experiment}
\label{sec:experiment}

In this section, we present the results of an experimental evaluation of Momentum Tracking
and demonstrate that Momentum Tracking is more robust to data heterogeneity than the existing decentralized learning methods with momentum.
In this section, we focus on test accuracy,
and more detailed evaluation about the convergence rate is presented in Sec. \ref{sec:synthetic}.

\subsection{Setup}
\textbf{Comparison Methods:}
(1) DSGD \citep{lian2017can}: the method described in Sec. \ref{sec:decentralized_learning};
(2) DSGDm \citep{gao2020periodic}: the method described in Sec. \ref{sec:momentum};
(3) QG-DSGDm \citep{lin21quasi}: a method in which the update rule of the momentum in DSGDm is modified to be more robust to data heterogeneity than DSGDm;
(4) DecentLaM \citep{yuan202decentlam}: a method in which the update rule of the momentum in DSGDm is modified to be more robust to data heterogeneity;
(5) Gradient Tracking \citep{nedic2017achieving}: a method without momentum that is robust to data heterogeneity;
(6) Momentum Tracking: the proposed method described in Sec. \ref{sec:proposed}. 

\textbf{Dataset and Model:} We evaluated Momentum Tracking using three 10-class image classification tasks: FashionMNIST \citep{xiao2017/online}, SVHN \citep{betzer2011reading}, and CIFAR-10 \citep{Krizhevsky09learningmultiple}.
Following the previous work \citep{niwa2020edge}, we distributed the data to nodes such that each node was given data of randomly selected $k$ classes.
When $k=10$, the data distributions held by each node can be regarded as statistically homogeneous.
When $k<10$, the data distributions are regarded as statistically heterogeneous.
We evaluated the comparison methods by setting $k$ to $\{ 4,6,8,10 \}$ and changing data heterogeneity.
Note that a smaller $k$ indicates that the data distributions are more heterogeneous. 
For the neural network architecture, we used LeNet \citep{lecun1998gradientbased} with group normalization \citep{wu2018group} in Sec. \ref{sec:results}.
In Sec. \ref{sec:architecture}, we present more detailed evaluation by varying the neural network architecture (e.g., VGG-11 \citep{simonyanZ2014very} and ResNet-34 \citep{he2016deep}).
For each comparison method, we used $10 \%$ of the training data for validation and individually tuned the step size.
For DSGDm, QG-DSGDm, DecentLaM, and Momentum Tracking, we set $\beta$ to $0.9$.
All experiments were repeated using three different seed values,
and we report their averages.
More detailed hyperparameter settings are presented in Sec \ref{sec:hyperparameter}.

\textbf{Network Topology and Implementation:}
Communication efficiency is one of the most important factors in decentralized learning
and is determined by the maximum degree of the underlying network topology \citep{neglia2019role,wang2019matcha,ying2021exponential}.
Thus, following these prior works, we present the results of setting the underlying network topology to a ring consisting of eight nodes (i.e., $N=8$) in Secs. \ref{sec:results} and \ref{sec:architecture}.
In Sec. \ref{sec:network}, we present more detailed evaluation by varying the network topology.
All comparison methods were implemented using PyTorch and run on eight GPUs (NVIDIA RTX 3090).

\subsection{Experimental Results}
\label{sec:results}
\begin{figure*}[!t]
\centering
  \subfigure[CIFAR-10 (10-class)]{
    \centering
    \includegraphics[width=0.26 \hsize]{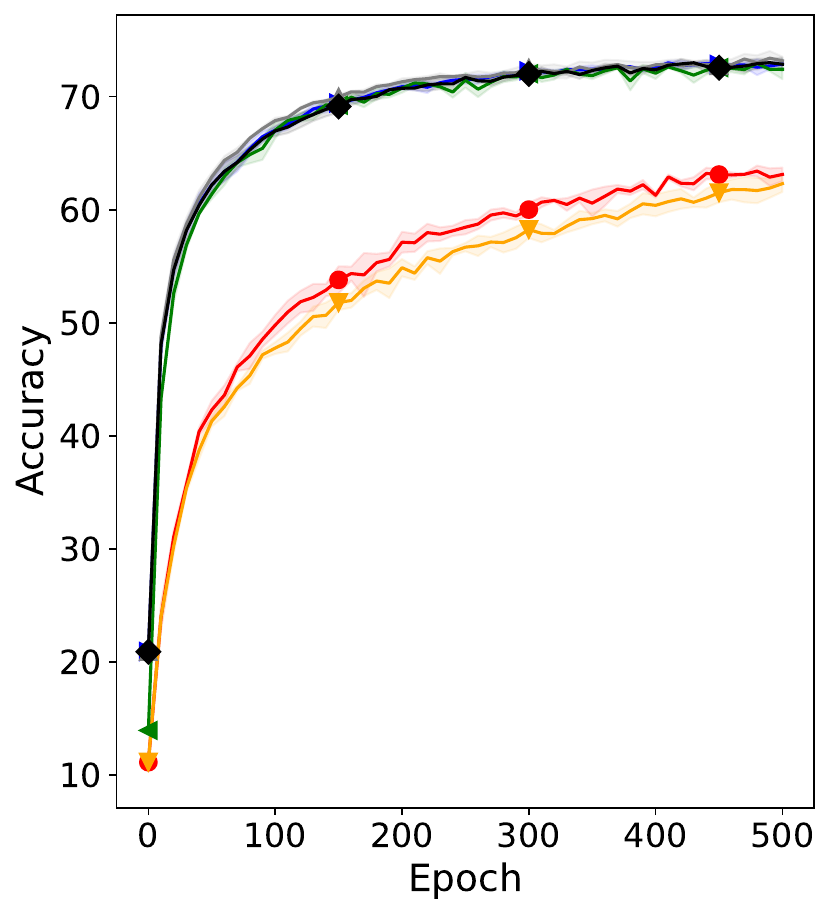}
  }
  \subfigure[CIFAR-10 (4-class)]{
    \centering
    \includegraphics[width=0.26 \hsize]{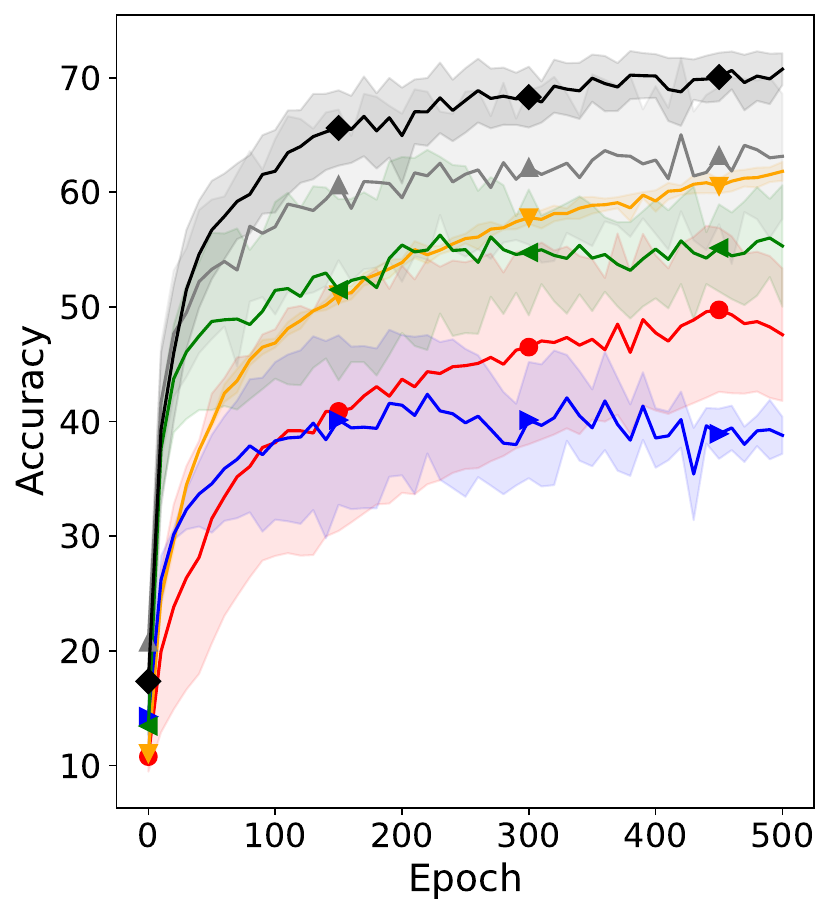}
  }
  \subfigure[Average accuracy for datasets]{
    \centering
    \includegraphics[width=0.26 \hsize]{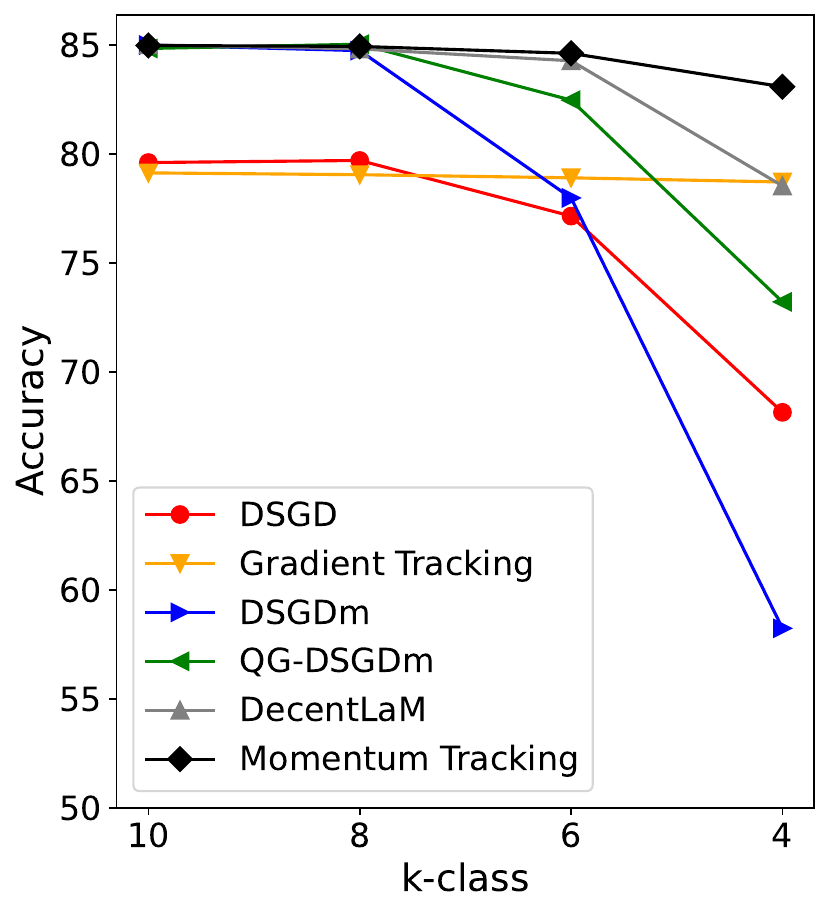}
  }
\vskip -0.1 in
\caption{(a) Learning curve on CIFAR-10 with LeNet in the $10$-class (i.e., homogeneous) setting. We evaluated the test accuracy per 10 epochs. (b) Learning curve in the $4$-class (i.e., heterogeneous) setting. (c) Average test accuracy for all datasets (i.e., FashionMNIST, SVHN, and CIFAR-10).}
\label{fig:cifar_plot}
\vskip -0.1 in
\end{figure*}

\begin{table*}[!t]
\caption{Test accuracy on FashionMNIST, SVHN, and CIFAR-10 with LeNet. ``$k$-class'' means that each node has only the data of randomly selected $k$ classes. Bold font means the highest accuracy.}
\label{table:lenet}
%\vskip -0.1 in
\centering
\begin{tabular}{lccccccc}
\toprule
           & \multicolumn{7}{c}{\textbf{FashionMNIST}} \\
           & 10-class && 8-class && 6-class && 4-class \\
\midrule
DSGD              & $85.6 \pm 0.49$      && $85.6 \pm 0.41$      && $82.7 \pm 1.12$      && $78.1 \pm 1.56$ \\
Gradient Tracking & $85.0 \pm 0.49$      && $85.4 \pm 0.26$      && $85.0 \pm 0.37$      && $84.9 \pm 0.22$ \\
\midrule
DSGDm             & $89.5 \pm 0.15$      && $89.3 \pm 0.21$      && $82.1 \pm 3.23$      && $68.7 \pm 5.02$ \\
QG-DSGDm          & $\bf{89.6 \pm 0.10}$ && $\bf{89.5 \pm 0.47}$ && $86.9 \pm 1.59$      && $80.8 \pm 2.94$ \\
DecentLaM         & $89.5 \pm 0.14$      && $89.3 \pm 0.36$      && $\bf{89.2 \pm 0.41}$ && $84.3 \pm 3.05$ \\
Momentum Tracking & $89.5 \pm 0.36$      && $89.4 \pm 0.05$      && $88.9 \pm 0.47$      && $\bf{86.8 \pm 1.56}$ \\
\bottomrule \\
\toprule
           & \multicolumn{7}{c}{\textbf{SVHN}} \\
           & 10-class && 8-class && 6-class && 4-class \\
\midrule
DSGD              & $90.1 \pm 0.17$ && $89.5 \pm 0.61$ && $87.6 \pm 1.94$ && $78.8 \pm 8.55$ \\
Gradient Tracking & $90.1 \pm 0.30$ && $89.8 \pm 0.38$ && $89.8 \pm 0.39$ && $89.4 \pm 0.47$ \\
\midrule
DSGDm             & $\bf{92.6 \pm 0.35}$ && $92.4 \pm 0.19$      && $88.1 \pm 4.38$      && $67.2 \pm 9.69$ \\
QG-DSGDm          & $92.5 \pm 0.22$      && $\bf{92.5 \pm 0.17}$ && $90.9 \pm 1.67$      && $83.5 \pm 7.14$ \\
DecentLaM         & $92.4 \pm 0.21$      && $92.2 \pm 0.39$      && $92.0 \pm 0.48$      && $88.2 \pm 4.75$ \\
Momentum Tracking & $\bf{92.6 \pm 0.32}$ && $92.4 \pm 0.40$      && $\bf{92.3 \pm 0.23}$ && $\bf{91.7 \pm 0.53}$ \\
\bottomrule \\
\toprule
           & \multicolumn{7}{c}{\textbf{CIFAR-10}} \\
           & 10-class && 8-class && 6-class && 4-class \\
\midrule
DSGD              & $63.1 \pm 0.60$      && $64.1 \pm 0.52$      && $61.2 \pm 1.16$      && $47.6 \pm 5.77$ \\
Gradient Tracking & $62.3 \pm 0.73$      && $62.0 \pm 0.80$      && $61.9 \pm 0.58$      && $61.8 \pm 0.82$ \\
\midrule
DSGDm             & $72.9 \pm 0.41$      && $72.5 \pm 0.20$      && $63.8 \pm 6.24$      && $38.8 \pm 1.61$ \\
QG-DSGDm          & $72.4 \pm 0.87$      && $\bf{73.1 \pm 0.16}$ && $69.6 \pm 2.42$      && $55.3 \pm 5.30$ \\
DecentLaM         & $\bf{73.2 \pm 0.36}$ && $72.9 \pm 0.14$      && $71.7 \pm 1.10$      && $63.1 \pm 5.43$ \\
Momentum Tracking & $72.9 \pm 0.59$      && $73.0 \pm 0.49$      && $\bf{72.6 \pm 0.41}$ && $\bf{70.7 \pm 1.38}$ \\
\bottomrule
\end{tabular}
%\vskip -0.3 in
\end{table*}
Table~\ref{table:lenet} lists the test accuracy for FashionMNIST, SVHN, and CIFAR-10.
Fig.~\ref{fig:cifar_plot} (a) and (b) present the learning curves for CIFAR-10 and Fig.~\ref{fig:cifar_plot} (c) presents the average test accuracy for all datasets.

\textbf{Comparison of Momentum Tracking and Gradient Tracking:}
First, we discuss the results of Momentum Tracking and Gradient Tracking.
Table~\ref{table:lenet} and Fig.~\ref{fig:cifar_plot} indicate that Momentum Tracking achieves a higher accuracy faster than Gradient Tracking and outperforms Gradient Tracking in all settings. 
When the data distributions are homogeneous (i.e., $10$-class), Momentum Tracking outperforms Gradient Tracking by $5.8 \%$ on average.
When the data distributions are heterogeneous (e.g., $4$-class), Momentum Tracking outperforms Gradient Tracking by $4.4 \%$ on average.
Thus, the results show that Momentum Tracking can consistently outperform Gradient Tracking regardless of data heterogeneity.

\textbf{Comparison of Momentum Tracking and DSGDm:}
Next, we discuss the results of Momentum Tracking and DSGDm.
The results show that when the data distributions are homogeneous (i.e., $10$-class),
Momentum Tracking and DSGDm are comparable and outperform DSGD and Gradient Tracking.
However, when the data distributions are heterogeneous (e.g., $4$-class),
the test accuracy of DSGDm decreases even more than that of DSGD,
and DSGDm underperforms DSGD by $9.9 \%$ on average.
By contrast, the results indicate that Momentum Tracking consistently outperforms DSGD and Gradient Tracking by $14.9 \%$ and $4.4 \%$ respectively when the data distributions are heterogeneous.
The results indicate that Momentum Tracking is more robust to data heterogeneity than DSGDm and outperforms DSGDm by $24.9 \%$ on average.

\begin{table*}[!t]
\caption{Test accuracy on CIFAR-10 with VGG-11 and ResNet-34. ``$k$-class'' indicates that each node has only the data of randomly selected $k$ classes, and bold font indicates the highest accuracy.}
\label{table:vgg}
%\vskip -0.1 in
\centering
\resizebox{\linewidth}{!}{
\begin{tabular}{lccccccc}
\toprule
           & \multicolumn{3}{c}{\textbf{CIFAR-10 + VGG-11}} && \multicolumn{3}{c}{\textbf{CIFAR-10 + ResNet-34}} \\
\cmidrule{2-4}
\cmidrule{6-8}
           & 10-class & 4-class & 2-class && 10-class & 4-class & 2-class \\
\midrule
DSGD              & $91.3 \pm 0.12$      & $86.9 \pm 1.75$      & $71.1 \pm 2.82$      && $94.3 \pm 0.13$      & $90.0 \pm 1.65$      & $63.5 \pm 0.90$ \\
Gradient Tracking & $88.1 \pm 0.14$      & $86.3 \pm 0.50$      & $83.0 \pm 0.04$      && $85.9 \pm 0.71$      & $82.6 \pm 0.33$      & $76.2 \pm 0.30$ \\ 
\midrule
DSGDm             & $\bf{92.2 \pm 0.09}$ & $77.3 \pm 4.05$      & $39.6 \pm 5.92$      && $95.8 \pm 0.26$      & $79.0 \pm 3.69$      & $27.7 \pm 2.83$ \\
QG-DSGDm          & $92.0 \pm 0.02$      & $89.5 \pm 1.08$      & $77.8 \pm 1.96$      && $95.8 \pm 0.22$      & $94.3 \pm 1.13$      & $79.9 \pm 1.59$ \\
DecentLaM         & $92.1 \pm 0.09$      & $\bf{90.9 \pm 0.65}$ & $85.2 \pm 0.67$      && $\bf{95.9 \pm 0.04}$ & $\bf{95.2 \pm 0.51}$ & $89.2 \pm 2.26$ \\
Momentum Tracking & $91.9 \pm 0.06$      & $\bf{90.9 \pm 0.60}$ & $\bf{87.0 \pm 0.48}$ && $95.0 \pm 0.13$      & $94.4 \pm 0.52$      & $\bf{89.9 \pm 0.73}$ \\
\bottomrule
\end{tabular}}
%\vskip -0.1 in
\end{table*}

\textbf{Comparison of Momentum Tracking, QG-DSGDm, and DecentLaM:}
When the data distributions are homogeneous (i.e., $10$-class), Momentum Tracking, QG-DSGDm, and DecentLaM are comparable and outperform DSGD and Gradient Tracking.
By contrast, when the data distributions are heterogeneous (e.g., $4$-class), Momentum Tracking consistently outperforms QG-DSGDm and DecentLaM by $9.9 \%$ and $4.5 \%$ respectively, whereas QG-DSGDm and DecentLaM are more robust to data heterogeneity than DSGDm.
Hence, these results are consistent with our theoretical analysis, as discussed in Secs. \ref{sec:convergence_analysis} and \ref{sec:discussion}.

In summary, when the data distributions are homogeneous, 
DSGDm, QG-DSGDm, DecentLaM, and Momentum Tracking are comparable and outperform DSGD and Gradient Tracking.
When the data distributions are heterogeneous, 
Momentum Tracking is more robust to data heterogeneity than DSGDm, QG-DSGDm, and DecentLaM,
and can outperform all comparison methods.

\subsection{Results with Various Neural Network Architectures}
\label{sec:architecture}
\begin{figure*}[!t]
\centering
\includegraphics[width=0.8 \hsize]{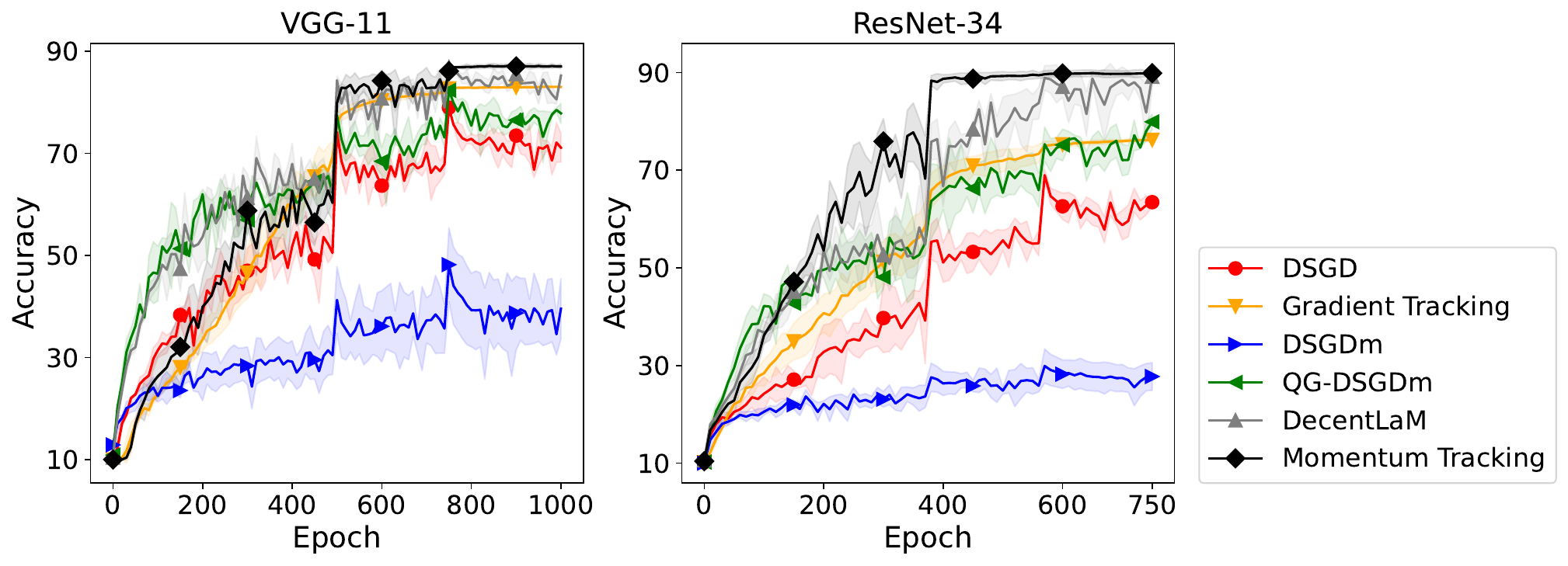}
\vskip -0.1 in
\caption{Learning curves for CIFAR-10 with VGG-11 and ResNet-34 in the $2$-class setting.}
\label{fig:cifar_deep}
\vskip -0.1 in
\end{figure*}

Next, we evaluated Momentum Tracking in more detail by varying the neural network architecture.
Table~\ref{table:vgg} lists the test accuracy with VGG-11 \citep{simonyanZ2014very} and ResNet-34 \citep{he2016deep} when we set $k$ to $\{ 2, 4, 10\}$, and Fig.~\ref{fig:cifar_deep} shows the learning curves.

For both neural network architectures, 
Table~\ref{table:vgg} reveals that when the data distributions are homogeneous (i.e., $10$-class), Momentum Tracking is comparable with DSGDm, QG-DSGDm, and DecentLaM
and outperforms DSGD and Gradient Tracking.
By contrast, when the data distributions are heterogeneous (e.g., $2$-class), Table~\ref{table:vgg} and Fig.~\ref{fig:cifar_deep} reveal that Momentum Tracking outperforms all comparison methods for both neural network architectures.
In particular, Fig.~\ref{fig:cifar_deep} indicates that DSGDm, QG-DSGDm, and DecentLaM are unstable and continue to oscillate in the final training phase,
whereas Momentum Tracking converges stably. 
These results are consistent with those of LeNet presented in Table~\ref{table:lenet}.
%\bluesout{Therefore, the results indicate that Momentum Tracking is more robust to data heterogeneity than DSGDm, QG-DSGDm, and DecentLaM,
%and can outperform these methods regardless of the neural network architecture.}

\section{Conclusion}
In this study, we propose Momentum Tracking, which is a method with momentum whose convergence rate is proven to be independent of data heterogeneity.
More specifically, we provide the convergence rate of Momentum Tracking 
in the setting where the objective function is non-convex and the stochastic gradient is used.
Our theoretical analysis reveals that the convergence rate of Momentum Tracking is independent of data heterogeneity for any $\beta \in [0, 1)$.
Through image classification tasks, we demonstrated that Momentum Tracking can consistently outperform the decentralized learning methods without momentum regardless of data heterogeneity.
Moreover, we showed that Momentum Tracking is more to data heterogeneity than existing decentralized learning methods with momentum
and can consistently outperform these existing methods when the data distributions are heterogeneous.

\section*{Acknowledgments}
Yuki Takezawa, Ryoma Sato, and Makoto Yamada were supported by JSPS KAKENHI Grant Number 23KJ1336, 21J22490, and MEXT KAKENHI Grant Number 20H04243, respectively.

\bibliography{main}
\bibliographystyle{tmlr}

\newpage
\appendix
\section{Pseudo-Codes}
\label{sec:psuedo}
The pseudo-codes for Momentum Tracking, QG-DSGDm, and DecentLaM are given in the following,
where $\textbf{Transmit}_{i\rightarrow j}(\cdot)$ denotes that node $i$ transmits parameters to node $j$ and $\textbf{Receive}_{i\leftarrow j}(\cdot)$ denotes that node $i$ receives parameters from node $j$. 

\begin{algorithm}[h]
   \caption{Update rules of Momentum Tracking at node $i$.}
   \label{alg:momentum_trackng}
\begin{algorithmic}[1]
   \STATE \textbf{Input:} Step size $\eta>0$, $\beta \in (0, 1]$, mixing matrix $\mW$. Initialize $\vc_i$ and $\vu_i$ to $\frac{1}{1 - \beta} (\nabla F_i (\vx_i^{(0)} ; \xi_i^{(0)}) - \frac{1}{N} \sum_{j} \nabla F_j (\vx_j^{(0)} ; \xi_j^{(0)}))$ for all $i\in V$ and $\vx_i$ with the same parameter.
   \FOR{$r = 0, \cdots, R$}
   \STATE $\vu_i^{(r+1)} \leftarrow \beta \vu_i^{(r)} + \nabla F_i(\vx_i^{(r)} ; \xi_i^{(r)})$.
   \FOR{$j \in \mathcal{N}_i$}
   \STATE $\textbf{Transmit}_{i\rightarrow j}(\vx^{(r)}_i)$ and $\textbf{Receive}_{i\leftarrow j}(\vx^{(r)}_j)$.
   \STATE $\textbf{Transmit}_{i\rightarrow j}(\vc^{(r)}_i - \vu_i^{(r+1)})$ and $\textbf{Receive}_{i\leftarrow j}(\vc^{(r)}_j - \vu_j^{(r+1)})$.
   \ENDFOR
   \STATE $\vx_i^{(r+1)} \leftarrow \sum_{j \in \mathcal{N}_i^+} W_{ij} \vx_j^{(r)} - \eta \left( \vu_i^{(r+1)} - \vc_i^{(r)} \right)$.
   \STATE $\vc_i^{(r+1)} \leftarrow \sum_{j \in \mathcal{N}_i^{+}} W_{ij} \left( \vc_j^{(r)} - \vu_j^{(r+1)} \right) + \vu_i^{(r+1)}$.
   \ENDFOR
\end{algorithmic}
\end{algorithm}

\begin{algorithm}[h]
   \caption{Update rules of QG-DSGDm at node $i$.}
\begin{algorithmic}[1]
   \STATE \textbf{Input:} Step size $\eta>0$, $\beta, \mu \in (0, 1]$, mixing matrix $\mW$. Initialize $\hat{\vu}_i$ to zero for all $i\in V$ and $\vx_i$ with the same parameter.
   \FOR{$r = 0, \cdots, R$}
   \STATE $\vu_i^{(r+1)} \leftarrow \beta \hat{\vu}_i^{(r)} + \nabla F_i(\vx_i^{(r)} ; \xi_i^{(r)})$.
   \STATE $\vx_i^{(r+\frac{1}{2})} \leftarrow \vx_i^{(r)} - \eta \vu_i^{(r+1)}$
   \FOR{$j \in \mathcal{N}_i$}
   \STATE $\textbf{Transmit}_{i\rightarrow j}(\vx^{(r+\frac{1}{2})}_i)$ and $\textbf{Receive}_{i\leftarrow j}(\vx^{(r+\frac{1}{2})}_j)$.
   \ENDFOR
   \STATE $\vx_i^{(r+1)} \leftarrow \sum_{j \in \mathcal{N}_i^+} W_{ij} \vx_j^{(r+\frac{1}{2})}$.
   \STATE $\vd_i^{(r+1)} \leftarrow \frac{\vx_i^{(r)} - \vx_i^{(r+1)}}{\eta}$.
   \STATE $\hat{\vu}_i^{(r+1)} \leftarrow \mu \hat{\vu}_i^{(r)} + (1-\mu) \vd_i^{(r+1)}$.
   \ENDFOR
\end{algorithmic}
\end{algorithm}

\begin{algorithm}[h]
   \caption{Update rules of DecentLaM at node $i$.}
\begin{algorithmic}[1]
   \STATE \textbf{Input:} Step size $\eta>0$, $\beta\in (0, 1]$, mixing matrix $\mW$. Initialize $\vu_i$ to zero for all $i\in V$ and $\vx_i$ with the same parameter.
   \FOR{$r = 0, \cdots, R$}
   \STATE $\vx_i^{(r+\frac{1}{2})} \leftarrow \vx_i^{(r)} - \eta \nabla F_i(\vx_i^{(r)} ; \xi_i^{(r)})$.
   \FOR{$j \in \mathcal{N}_i$}
   \STATE $\textbf{Transmit}_{i\rightarrow j}(\vx^{(r+\frac{1}{2})}_i)$ and $\textbf{Receive}_{i\leftarrow j}(\vx^{(r+\frac{1}{2})}_j)$.
   \ENDFOR
   \STATE $\hat{\vg}_i^{(r+1)} \leftarrow \frac{1}{\eta} \vx_i^{(r)} - \frac{1}{\eta} \sum_{j \in \mathcal{N}_i} W_{ij} \vx_j^{(r+\frac{1}{2})}$.
   \STATE $\vu_i^{(r+1)} \leftarrow \beta \vu_i^{(r)} + \hat{\vg}_i^{(r+1)}$.
   \STATE $\vx_i^{(r+1)} \leftarrow \vx_i^{(r)} - \eta \vu_i^{(r+1)}$.
   \ENDFOR
\end{algorithmic}
\end{algorithm}

\section{Additional Discussion about Convergence Rate}
\label{sec:additional_discussion_of_rate}

\subsection{Comparison with Gradient Tracking}
Because Momentum Tracking is equivalent to Gradient Tracking when $\beta=0$, Theorem \ref{theorem:convergence_rate} also provides the convergence rate of Gradient Tracking.
In this section, we compare the convergence rate of Gradient Tracking provided in Theorem \ref{theorem:convergence_rate} to that provided by \citet{koloskova2021an}.

From Theorem \ref{theorem:convergence_rate}, we get the following statement.
\begin{corollary}
\label{cor:convergence_rate}
Suppose that $\beta=0$ and the assumptions of Theorem \ref{theorem:convergence_rate} hold.
Then, for any $R \geq 1$, there exists a step size $\eta$ such that the average parameter $\bar{\vx} \coloneqq \frac{1}{N} \sum_{i} \vx_i$ generated by Eqs.~(\ref{eq:mt_1}-\ref{eq:mt_3}) satisfies
\begin{align}
    \label{eq:our_rate_of_gt}
    \frac{1}{R} \sum_{r=0}^{R-1} \mathbb{E} \left\| \nabla f (\bar{\vx}^{(r)}) \right\|^2 
    \leq 
    \mathcal{O} \left( \sqrt{\frac{r_0 \sigma^2 L}{N R}} 
    + \left( \frac{r_0 \sigma L}{p^2 R} \right)^\frac{2}{3} 
    + \frac{L r_0}{p^2 R} \right),
\end{align}
where $r_0 \coloneqq f(\bar{\vx}^{(0)}) - f^\star$.
\end{corollary}

Then, under Assumptions \ref{assumption:lower_bound}, \ref{assumption:mixing}, \ref{assumption:smoothness}, and \ref{assumption:stochasic_gradient},
\citet{koloskova2021an} provided the convergence rate of Gradient Tracking as follows.

\begin{theorem}[\citet{koloskova2021an}]
\label{theorem:gt}
Suppose that Assumptions \ref{assumption:lower_bound}, \ref{assumption:mixing}, \ref{assumption:smoothness}, and \ref{assumption:stochasic_gradient} hold,
each model parameter $\vx_i$ is initialized with the same parameters,
and $\vc_i$ is initialized as $\mathbf{0}$.
Then, for any round $R > \frac{2}{p} \log (\frac{50}{p}(1 + \log \frac{1}{p}))$,
there exists a step size $\eta$ that satisfies that the average parameter $\bar{\vx} \coloneqq \frac{1}{N} \sum_i \vx_i$ generated by Gradient Tracking satisfies
\begin{align}
    \label{eq:rate_of_gt}
    \frac{1}{R} \sum_{r=0}^{R-1} \left\| \nabla f (\bar{\vx}^{(r)}) \right\|^2 \leq \tilde{\mathcal{O}}\left( \sqrt{\frac{r_0 \sigma^2 L}{N R}} + \left( \frac{r_0 \sigma L }{(\sqrt{p} c + p \sqrt{N}) R}\right)^{\frac{2}{3}}
    + \frac{L (r_0 + L \zeta_0^2)}{p c R} \right),
\end{align}
where $\tilde{O}(\cdot)$ hides the polylogarithmic factors, 
$\zeta_0^2 \coloneqq \frac{1}{N} \sum_i \| \nabla F_i(\bar{\vx}^{(0)} ; \xi_i) - \frac{1}{N} \sum_j \nabla F_j (\bar{\vx}^{(0)} ; \xi_j)\|^2$,
$c \coloneqq 1 - \min \{ \lambda_{\text{min}}, 0 \}^2$, and $\lambda_{\text{min}}$ is the minimum eigenvalue of $\mW$.
\end{theorem}
Comparing the convergence rates in Eqs.~(\ref{eq:our_rate_of_gt}) and (\ref{eq:rate_of_gt}), 
the convergence rate in Eq.~(\ref{eq:rate_of_gt}) is tighter than that in Eq.~(\ref{eq:our_rate_of_gt}) 
because $c \geq p$ for any mixing matrix $\mW$.
However, because the convergence rate in Eq.~(\ref{eq:rate_of_gt}) holds only when the number of round $R$ is larger than $\frac{2}{p} \log (\frac{50}{p}(1 + \log \frac{1}{p}))$, Theorem \ref{theorem:gt} can not describe the behavior of the convergence rate at the beginning of the training.
In contrast, Corollary \ref{cor:convergence_rate} provides the convergence rate for Gradient Tracking that holds for any $R \geq 1$.

\subsection{Comparison with Other Decentralized Learning Methods}
\citet{lu2021optimal} and \citet{yuan2022revisiting} proposed DeTAG and MG-DSGD that can achieve optimal convergence rates by using algorithmic techniques such as gradient accumulation and multiple gossip averaging.
However, Assumption \ref{assumption:heterogeneity} is necessary for both analyses.
Then, as the data heterogeneity becomes large,
the convergence rate of DeTAG deteriorates, and the number of multiple gossip averaging increases. 
The goal of our study is to propose a method with momentum whose convergence rate is independent of data heterogeneity.
Thus, we leave it to future work to compare Momentum Tracking with these methods.

\section{Additional Experiments}
\subsection{Results with Various Network Topologies}
\label{sec:network}
We evaluated Momentum Tracking in more detail by changing the underlying network topology.
Table~\ref{table:network} lists the test accuracy of all comparison methods when we set the underlying network topology to be a hypercube or a semantic exponential graph.

Table~\ref{table:network} indicates that when the data distributions held by each node are statistically homogeneous (i.e., $10$-class),
DSGDm, QG-DSGDm, DecentLaM, and Momentum Tracking are comparable and outperform DSGD and Gradient Tracking for all network topologies.
When the data distributions are heterogeneous (i.e., $4$-class), the results show that Momentum Tracking is more robust to data heterogeneity than DSGDm, QG-DSGDm, and DecentLaM
and outperforms all comparison methods for all network topologies.
Therefore, the results indicate that Momentum Tracking is robust to data heterogeneity regardless of the underlying network topology.

\begin{table}[!h]
\caption{Test accuracy on CIFAR-10 with different underlying network topologies.}
\label{table:network}
\vskip -0.1 in
\centering
\begin{tabular}{lccccccc}
\toprule
           & \multicolumn{7}{c}{\textbf{CIFAR-10}} \\
           & \multicolumn{3}{c}{Hypercube} && \multicolumn{3}{c}{Semantic Exponential Graph} \\
\cmidrule{2-4}
\cmidrule{6-8}
           & 10-class && 4-class && 10-class && 4-class \\
\midrule
DSGD              & $63.3 \pm 0.65$      && $55.9 \pm 4.11$      && $64.0 \pm 0.26$      && $60.7 \pm 1.82$ \\
Gradient Tracking & $61.0 \pm 1.34$      && $60.2 \pm 1.13$      && $62.4 \pm 0.53$      && $62.4 \pm 0.89$ \\
\midrule
DSGDm             & $\bf{73.2 \pm 0.09}$ && $45.0 \pm 5.90$      && $\bf{73.4 \pm 0.13}$ && $51.5 \pm 7.80$ \\
QG-DSGDm          & $73.0 \pm 0.31$      && $62.9 \pm 3.68$      && $\bf{73.4 \pm 0.58}$ && $70.2 \pm 1.09$ \\
DecentLaM         & $72.9 \pm 0.24$      && $69.1 \pm 4.05$      && $72.9 \pm 0.73$      && $71.2 \pm 1.72$ \\
Momentum Tracking & $72.8 \pm 0.15$      && $\bf{72.7 \pm 0.28}$ && $72.7 \pm 0.33$     && $\bf{72.9 \pm 0.07}$ \\
\bottomrule
\end{tabular}
\vskip -0.1 in
\end{table}

\subsection{Results with Other Heterogeneous Setting}

In Sec. \ref{sec:experiment}, we show the results when the data are distributed such that each node had data of randomly selected $k$ classes.
In this section, we show the results in another heterogeneous setting,
where the label distributions of each node are determined by Dirichlet distributions \cite{hsu2019measuring}.

Table \ref{table:dirichlet} lists the results when we distributed data using Dirichlet distributions.
The results indicate that Momentum Tracking is more robust to the data heterogeneity than DSGDm, QG-DSGDm, and DecentLaM in both cases where we use Dirichlet distributions and where we use $k$-class setting.

\begin{table}[h!]
\caption{Test accuracy on CIFAR-10 with different $\alpha$.}
\label{table:dirichlet}
\vskip -0.1 in
    \centering
    \begin{tabular}{lcc}
    \toprule
                 &\multicolumn{2}{c}{\textbf{CIFAR-10 + VGG-11}} \\
                 & $\alpha=10$ (homogeneous case) & $\alpha=0.1$ (heterogeneous case) \\
    \midrule
    QG-DSGDm          & $89.7 \pm 0.07$ & $87.2 \pm 1.54 $\\
    DecentLaM         & $90.1 \pm 0.28$ & $88.6 \pm 0.99$ \\
    Momentum Tracking & $\bf{90.5 \pm 0.01}$ & $\bf{90.2 \pm 0.37}$\\
    \bottomrule
    \end{tabular}
\vskip -0.1 in
\end{table}

\subsection{Initial Value Analysis}
In this section, we discuss the initial values of $\vc_i$ and $\vu_i$.
Table~\ref{table:initial_value} lists the test accuracy for Momentum Tracking when we initialize $\vc_i$ and $\vu_i$ to zero and when we initialize $\vc_i$ and $\vu_i$ as in Theorem \ref{theorem:convergence_rate}.
The results indicate that the test accuracy are almost equivalent on both settings.
Hence, Theorem \ref{theorem:convergence_rate} requires $\vc_i$ and $\vu_i$ to be initialized to $\frac{1}{1 - \beta} (\nabla F_i (\vx_i^{(0)} ; \xi_i^{(0)}) - \frac{1}{N} \sum_{j=1}^N \nabla F_j (\vx_j^{(0)} ; \xi_j^{(0)}))$.
However, in practice, $\vc_i$ and $\vu_i$ can be initialized to zeros without any impact on accuracy.

\begin{table}[!h]
\caption{Test accuracy on FashionMNIST, SVHN, and CIFAR-10 with LeNet. ``$k$-class'' indicates that each node has only the data of randomly selected $k$ classes.}
\label{table:initial_value}
\vskip -0.1 in
\centering
\begin{tabular}{lcccc}
\toprule
           & \multicolumn{4}{c}{\textbf{FashionMNIST}} \\
           & 10-class & 8-class & 6-class & 4-class \\
\midrule
Momentum Tracking                                        & $89.5 \pm 0.36$ & $89.4 \pm 0.05$ & $88.9 \pm 0.47$ & $86.8 \pm 1.56$ \\
Momentum Tracking ($\vc_i^{(0)}=\vu_i^{(0)}=\mathbf{0}$) & $89.5 \pm 0.38$ & $89.4 \pm 0.04$ & $88.7 \pm 0.63$ &  $85.8 \pm 1.53$ \\
\bottomrule \\
\toprule
           & \multicolumn{4}{c}{\textbf{SVHN}} \\
           & 10-class & 8-class & 6-class & 4-class \\
\midrule
Momentum Tracking                                        & $92.6 \pm 0.32$ & $92.4 \pm 0.40$ & $92.3 \pm 0.23$ & $91.7 \pm 0.53$ \\
Momentum Tracking ($\vc_i^{(0)}=\vu_i^{(0)}=\mathbf{0}$) & $92.5 \pm 0.34$ & $92.3 \pm 0.50$ & $92.2 \pm 0.29$ &  $92.0 \pm 0.81$ \\
\bottomrule \\
\toprule
           & \multicolumn{4}{c}{\textbf{CIFAR-10}} \\
           & 10-class & 8-class & 6-class & 4-class \\
\midrule
Momentum Tracking                                        & $72.9 \pm 0.59$ & $73.0 \pm 0.49$ & $72.6 \pm 0.41$ & $70.7 \pm 1.38$ \\
Momentum Tracking ($\vc_i^{(0)}=\vu_i^{(0)}=\mathbf{0}$) & $72.8 \pm 0.35$ & $72.9 \pm 0.32$ & $73.0 \pm 0.41$ & $70.7 \pm 2.00$ \\
\bottomrule
\end{tabular}
\end{table}

\subsection{Comparison with RelaySum}

In this section, we compare Momentum Tracking with RelaySum \cite{vogels2021relaysum},
which is one of the methods that are most robust to data heterogeneity,
and RelaySum with momentum (RelaySumM). 
Table \ref{table:relaysum} lists the accuracy on CIFAR-10 with VGG-11.
The results indicate that RelaySumM is more robust to data heterogeneity than Momentum Tracking
and outperforms Momentum Tracking in the $2$-class setting.
However, the convergence rate of RelaySum is proven to be independent of data heterogeneity only when the momentum is not applied,
and it remains to be unclear whether the convergence rate of RelaySum is independent of data heterogeneity when the momentum is applied.
Thus, it is a clear advantage that the convergence rate of Momentum Tracking is proven to be independent of data heterogeneity for any momentum coefficient $\beta \in [0,1)$.
The main objective and contribution of our study are not to achieve state-of-the-art, 
but to propose a method with momentum whose convergence rate is proven to be independent of the data heterogeneity.
%in the standard deep learning setting. 
We believe that our proof is helpful for future research that will attempt to analyze the convergence rates when the momentum is applied (e.g., the convergence rate of RelaySumM).

\begin{table}[h!]
\caption{Test accuracy on CIFAR-10 with VGG-11.}
\label{table:relaysum}
\vskip -0.1 in
    \centering
    \begin{tabular}{lcc}
    \toprule
                 &\multicolumn{2}{c}{\textbf{CIFAR-10 + VGG-11}} \\
                 & $10$-class & $2$-class \\
    \midrule
    DSGD              & $91.3 \pm 0.12$ & $71.1 \pm 2.82$ \\
    Gradient Tracking & $88.1 \pm 0.14$ & $83.0 \pm 0.04$  \\
    RelaySum          & $91.1 \pm 0.13$ & $89.0 \pm 0.15$ \\
    \midrule
    DSGDm             & $\bf{92.2 \pm 0.09}$ & $39.6 \pm 5.92$  \\
    QG-DSGDm          & $92.0 \pm 0.02$ & $ 77.8 \pm 1.96$  \\
    DecentLaM         & $92.1 \pm 0.09$ & $85.2 \pm 0.67$  \\
    RelaySumM         & $92.1 \pm 0.13$ & $\bf{89.3 \pm 0.76}$ \\
    Momentum Tracking & $91.9 \pm 0.06$ & $87.0 \pm 0.48$ \\
    \bottomrule
    \end{tabular}
\end{table}

\subsection{Comparision with ABm and GTAdam}
In this section, we compared Momentum Tracking with ABm \citep{xin2020distributed} and GTAdam \citep{carnevale2022gtadam}, showing the results in Table \ref{table:abm}.
The update rules of Momentum Tracking are slightly different from that of ABm and GTAdam, 
but the results indicate that they can achieve almost the same accuracy.
However, as we mention in Remark \ref{remark:abm}, 
ABm and GTAdam are proven to be independent of data heterogeneity only in the strongly convex setting.
Thus, Momentum Tracking is the first method with momentum whose convergence rate is proven to be independent of data heterogeneity in non-convex and stochastic settings.

\begin{table}[h]
    \caption{Test accuracy on CIFAR-10 with LeNet.}
    \vskip - 0.1 in
    \label{table:abm}
    \begin{center}
    \centering
    \begin{tabular}{lcc}
        \toprule
         & $10$-class & $4$-class \\
        \midrule
         Momentum Tracking & $72.9 \pm 0.59$ & $70.7 \pm 1.38$ \\
         ABm               & $71.0 \pm 0.20$ & $71.1 \pm 0.20$ \\
         GTAdam            & $71.4 \pm 0.56$ & $67.7 \pm 3.20$ \\
        \bottomrule
    \end{tabular}
    \end{center}
\end{table}

\subsection{Learning Curves}
In this section, we present the learning curves for the results whose final accuracy are presented in Tables~\ref{table:lenet} and \ref{table:vgg}.
Figs.~\ref{fig:learning_curve_fashion}, \ref{fig:learning_curve_svhn}, and \ref{fig:learning_curve_cifar} show the learning curves for FashionMNSIT, SVHN, and CIFAR-10, respectively, with LeNet.
Figs.~\ref{fig:learning_curve_cifar_vgg} and \ref{fig:learning_curve_cifar_resnet} show the learning curves for CIFAR-10 with VGG-11 and ResNet-34, respectively.

When the data distributions are statistically homogeneous (i.e., $10$-class), 
the results indicate that DSGDm, QG-DSGDm, DecentLaM, and Momentum Tracking are comparable
and can achieve high accuracy faster than DSGD and Gradient Tracking.
When the data distributions are statistically heterogeneous (e.g., $2$-class and $4$-class),
the results indicate that the learning curves for Momentum Tracking are more stable than those for DSGDm, QG-DSGDm, and DecentLaM,
and Momentum Tracking outperforms all comparison methods.
In particular, in Figs.~\ref{fig:learning_curve_cifar_vgg}~and~\ref{fig:learning_curve_cifar_resnet}, 
the accuracy of DSGD, DSGDm, QG-DSGDm, and DecentLaM continue to oscillate in the final training phase in the $2$-class setting, 
whereas the accuracy of Momentum Tracking and Gradient Tracking converge in the $2$-class setting as well as in the $10$-class setting.
Therefore, Momentum Tracking is more robust to data heterogeneity than DSGDm, QG-DSGDm, and DecentLaM.

\begin{figure}[!h]
\centering
  \subfigure{
    \centering
    \includegraphics[width=0.23 \columnwidth]{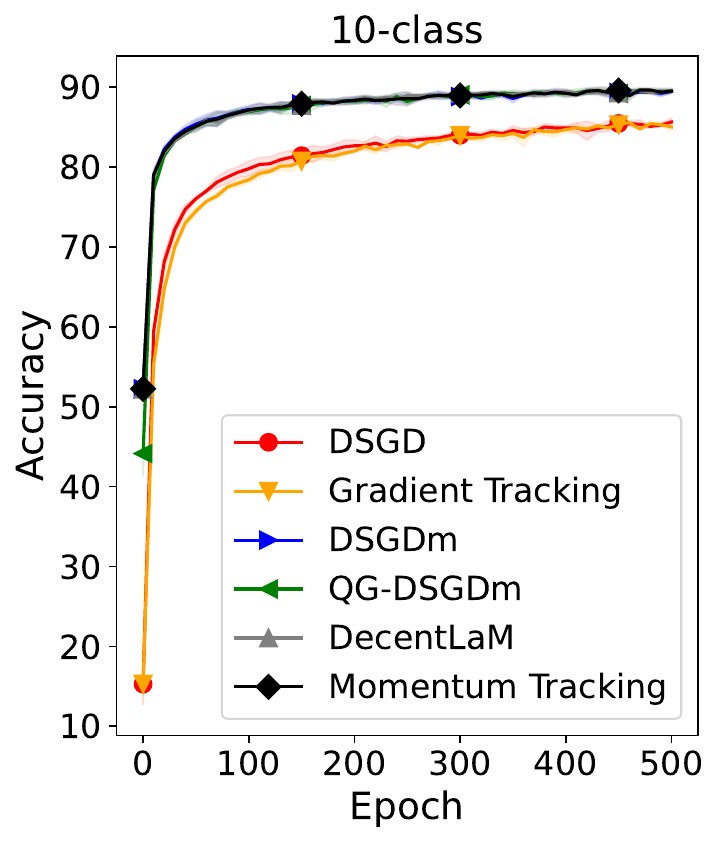}
  }
  \subfigure{
    \centering
    \includegraphics[width=0.23 \columnwidth]{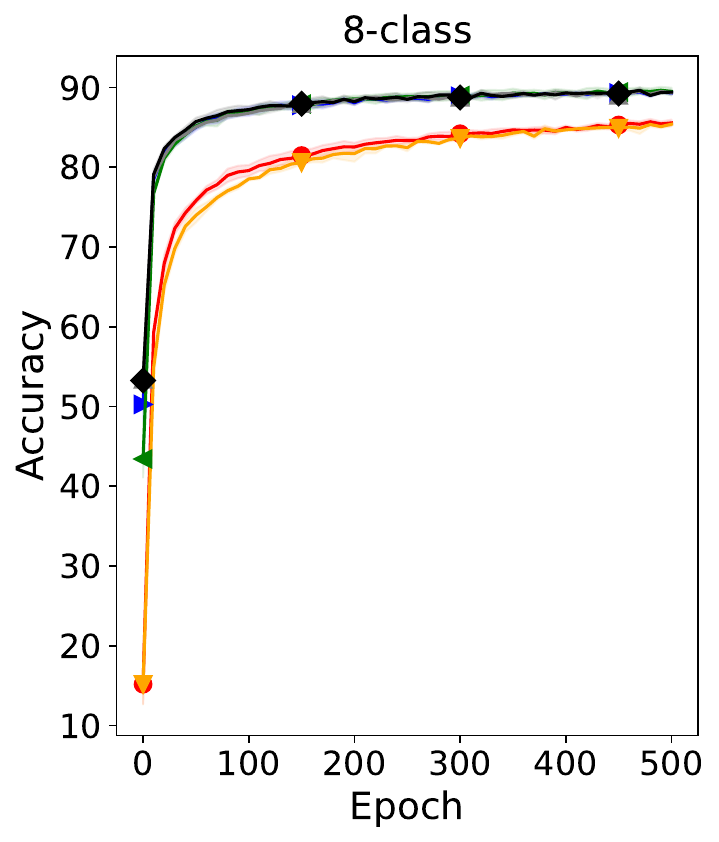}
  }
  \subfigure{
    \centering
    \includegraphics[width=0.23 \columnwidth]{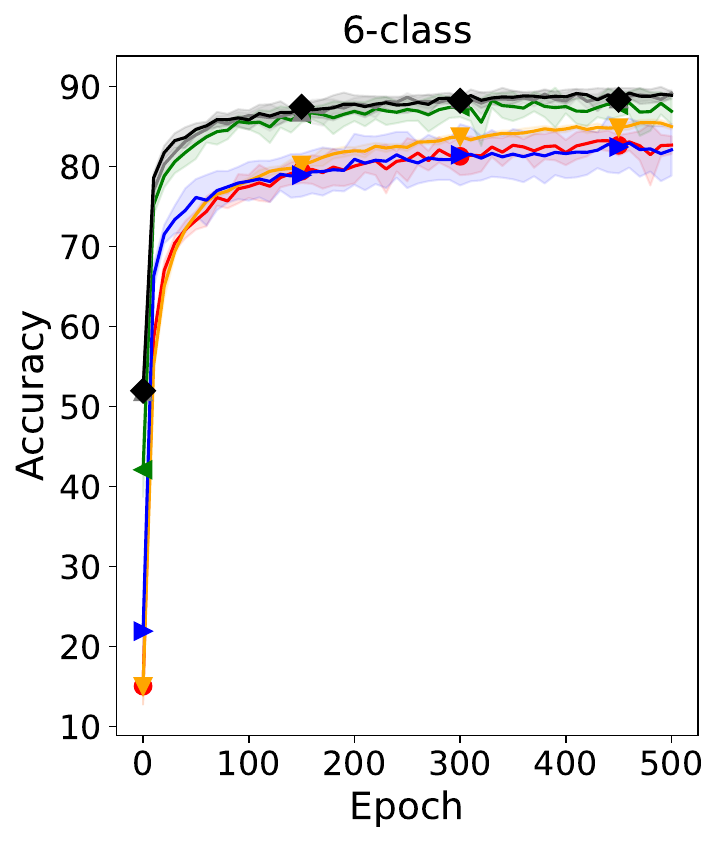}
  }
  \subfigure{
    \centering
    \includegraphics[width=0.23 \columnwidth]{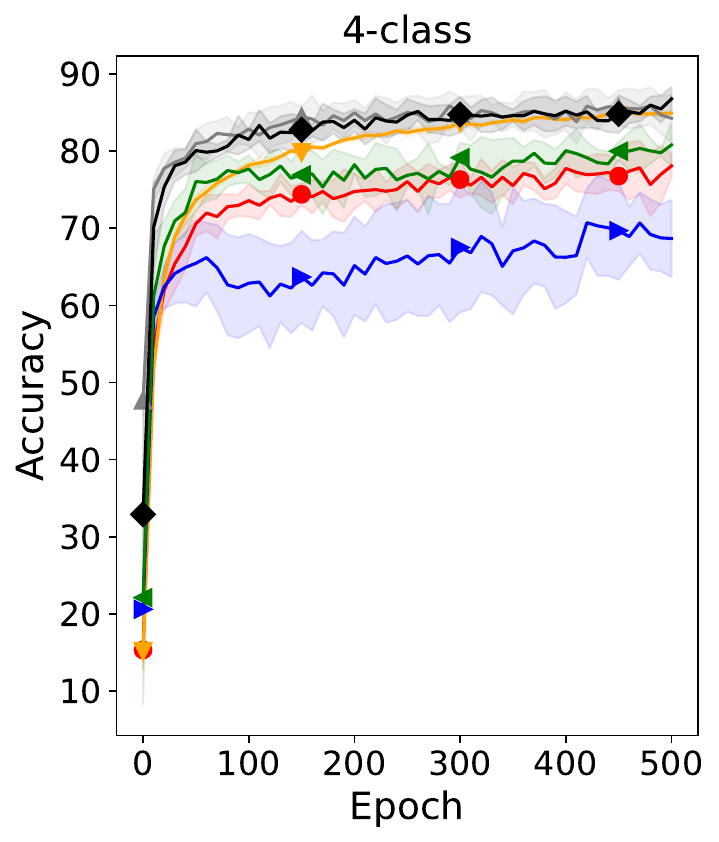}
  }
\caption{Learning curves on FashionMNIST. The accuracy is evaluated per 10 epochs.}
\label{fig:learning_curve_fashion}
\end{figure}

\begin{figure}[!h]
\centering
  \subfigure{
    \centering
    \includegraphics[width=0.23 \columnwidth]{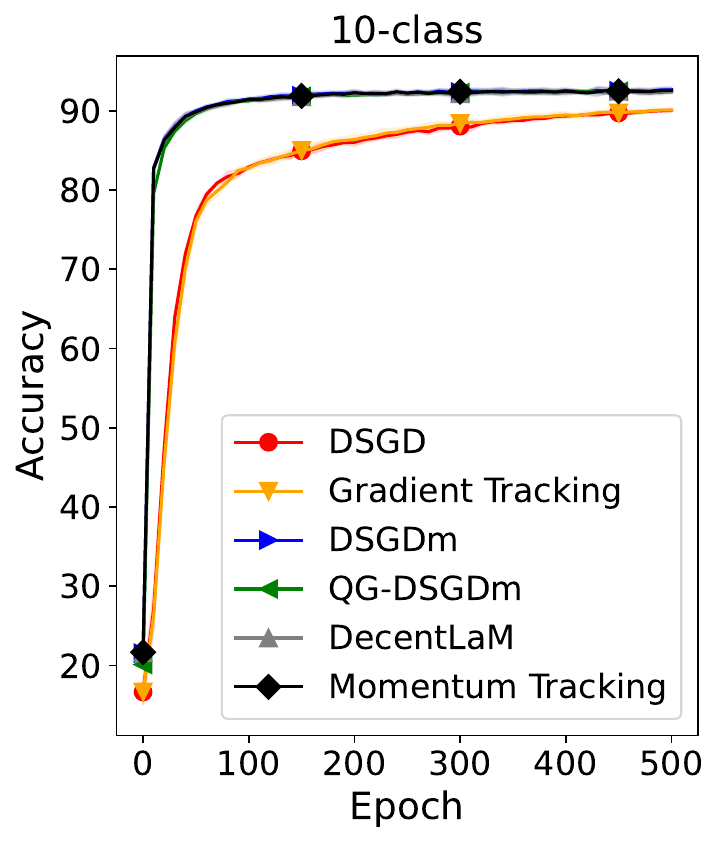}
  }
  \subfigure{
    \centering
    \includegraphics[width=0.23 \columnwidth]{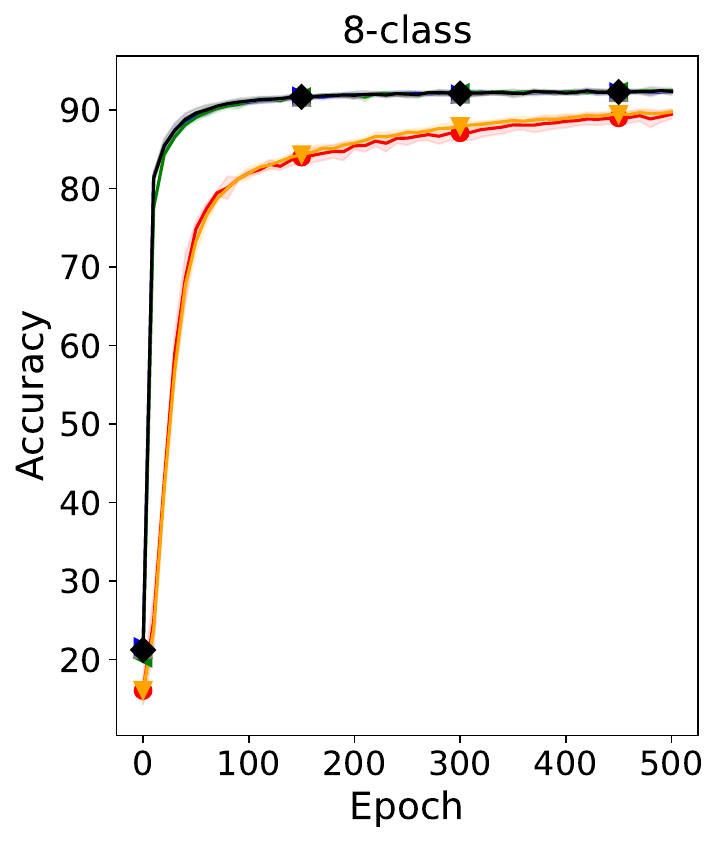}
  }
  \subfigure{
    \centering
    \includegraphics[width=0.23 \columnwidth]{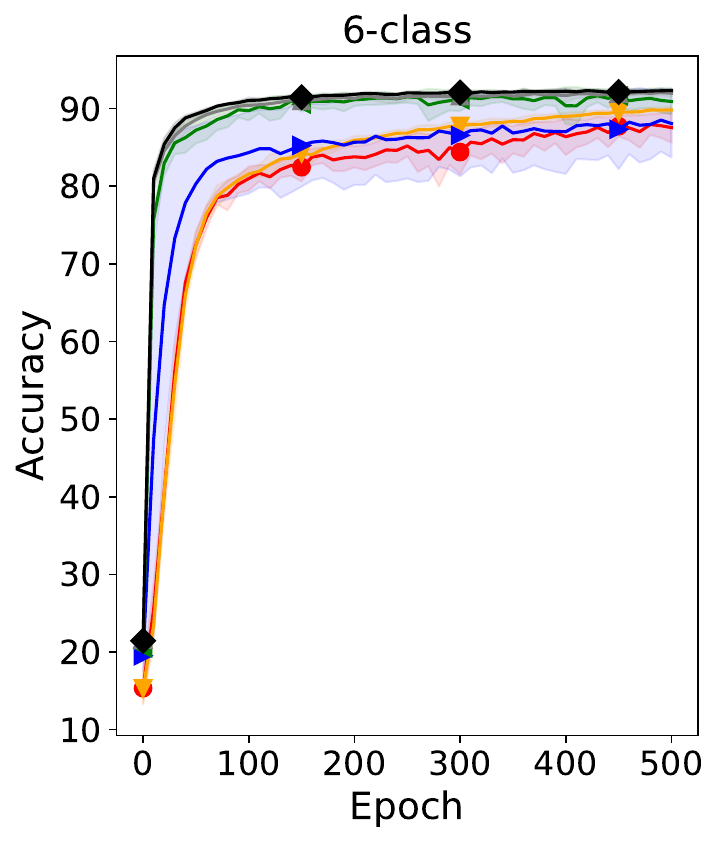}
  }
  \subfigure{
    \centering
    \includegraphics[width=0.23 \columnwidth]{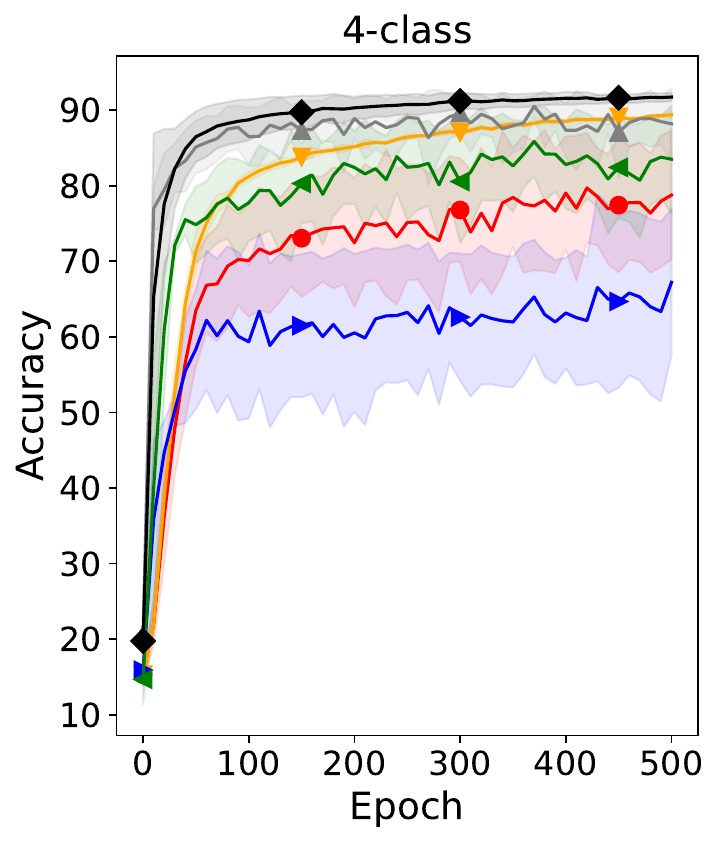}
  }
\caption{Learning curves on SVHN. The accuracy is evaluated per 10 epochs.}
\label{fig:learning_curve_svhn}
\end{figure}

\begin{figure}[!t]
\centering
  \subfigure{
    \centering
    \includegraphics[width=0.23 \columnwidth]{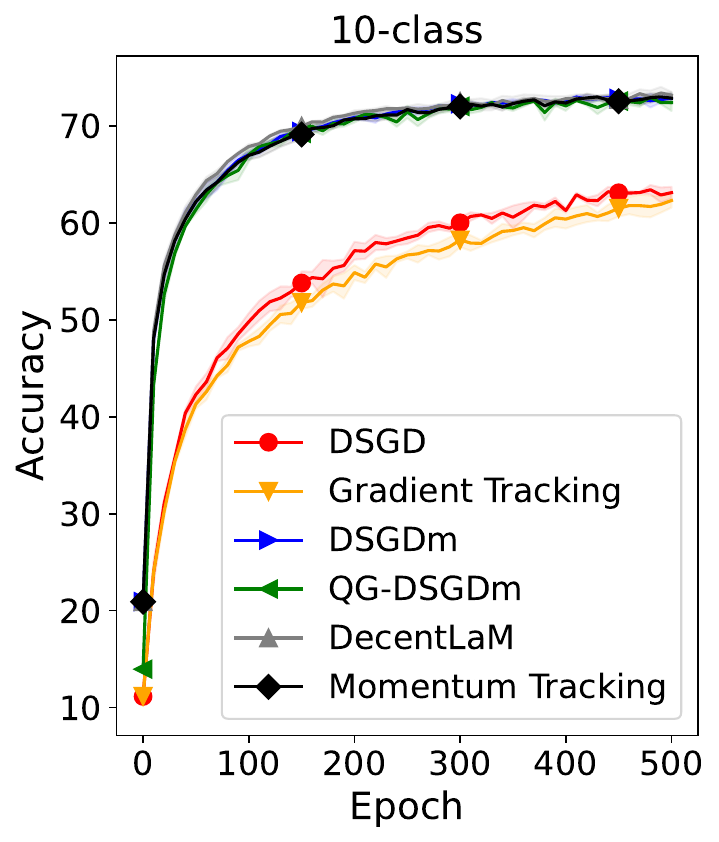}
  }
  \subfigure{
    \centering
    \includegraphics[width=0.23 \columnwidth]{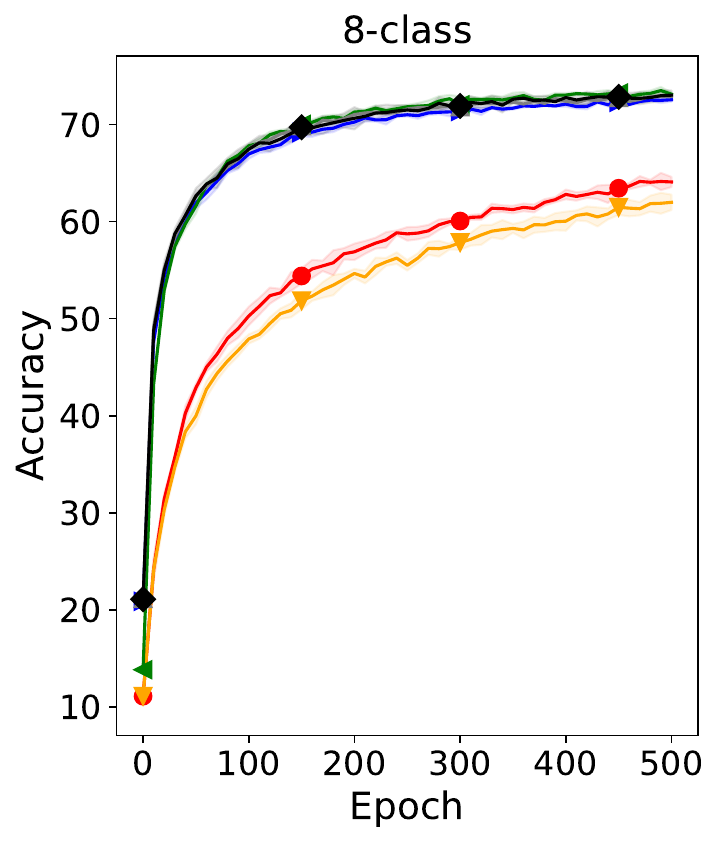}
  }
  \subfigure{
    \centering
    \includegraphics[width=0.23 \columnwidth]{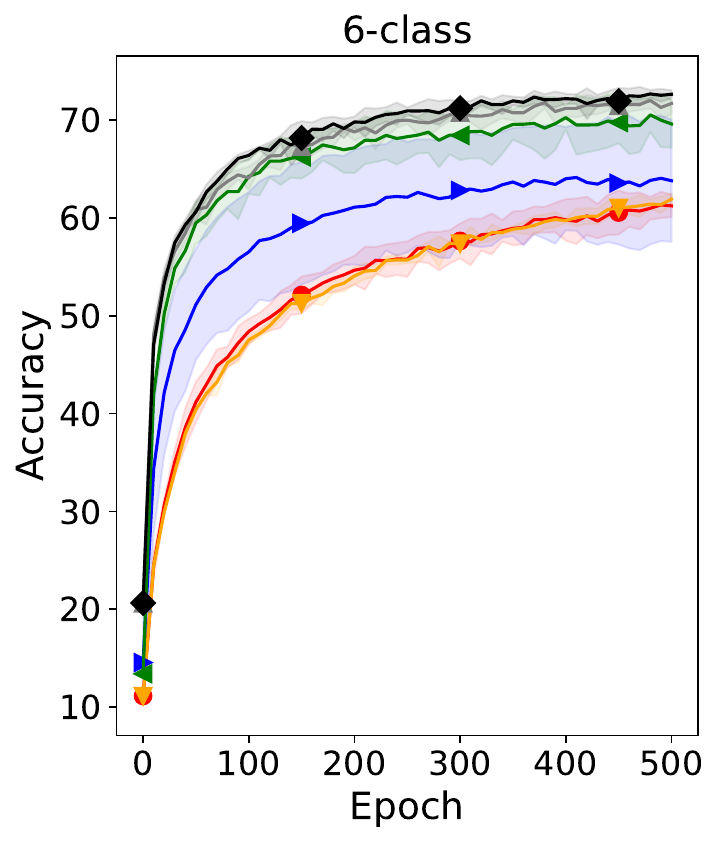}
  }
  \subfigure{
    \centering
    \includegraphics[width=0.23 \columnwidth]{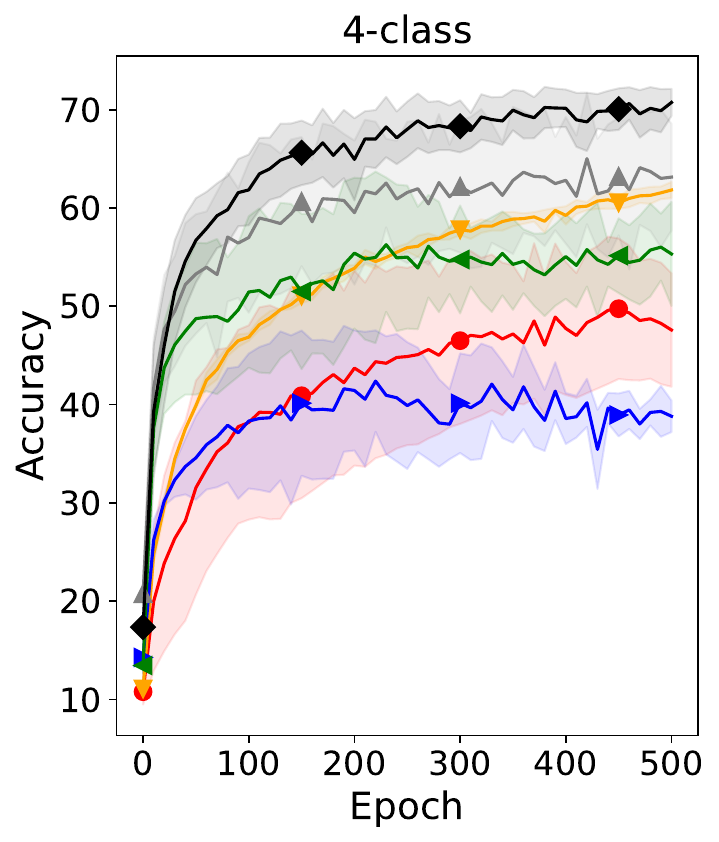}
  }
\caption{Learning curves on CIFAR-10. The accuracy is evaluated per 10 epochs.}
\label{fig:learning_curve_cifar}
\end{figure}

\begin{figure}[!h]
\centering
  \subfigure{
    \centering
    \includegraphics[width=0.31 \columnwidth]{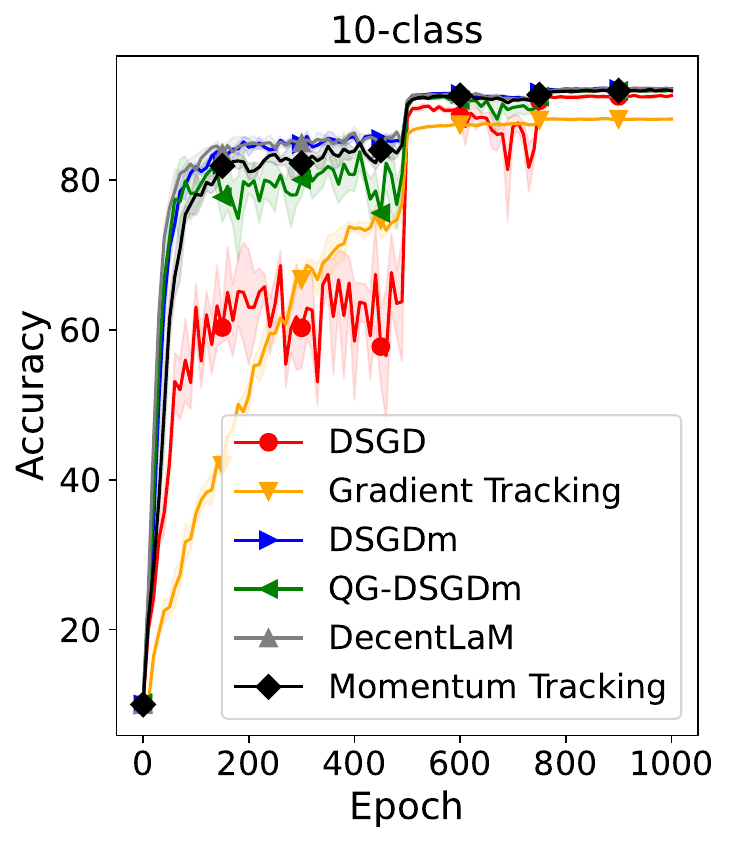}
  }
  \subfigure{
    \centering
    \includegraphics[width=0.31 \columnwidth]{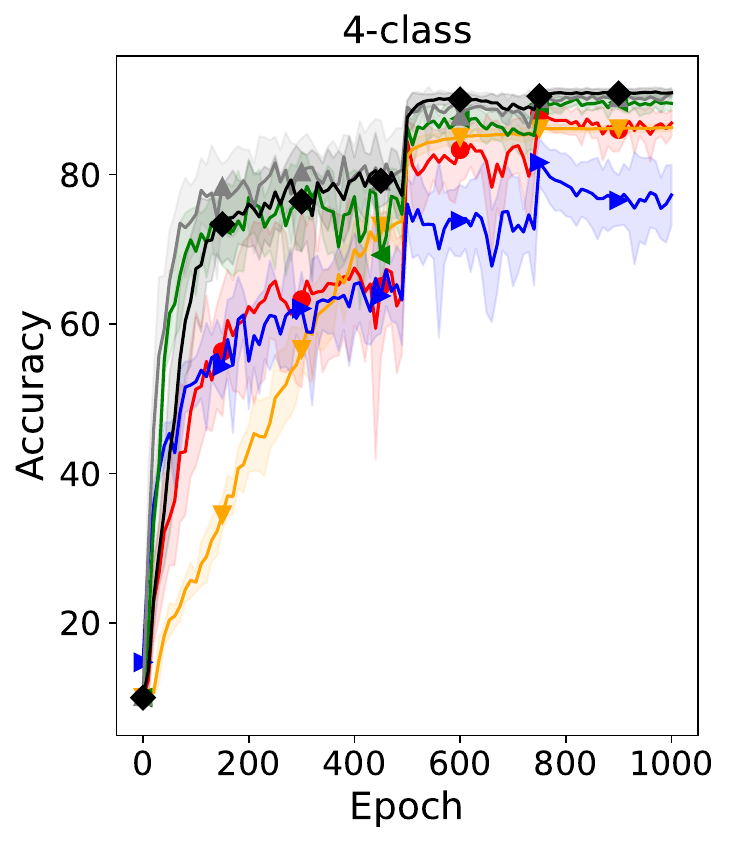}
  }
  \subfigure{
    \centering
    \includegraphics[width=0.31 \columnwidth]{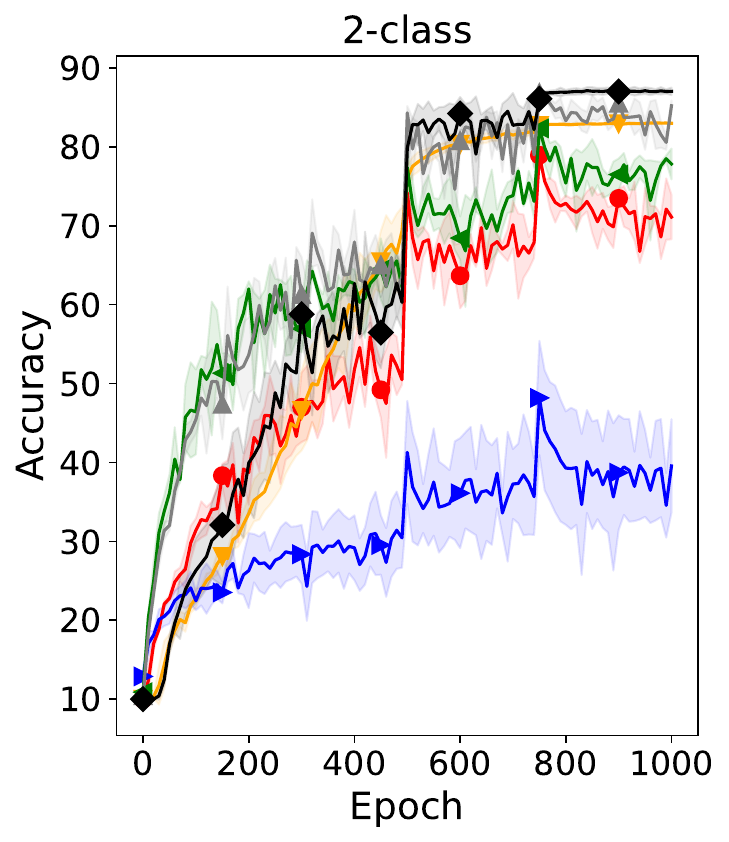}
  }
\vskip -0.1 in
\caption{Learning curves on CIFAR-10 with VGG-11. The accuracy is evaluated per 10 epochs.}
\label{fig:learning_curve_cifar_vgg}
\vskip -0.1 in
\end{figure}

\begin{figure}[!h]
\centering
  \subfigure{
    \centering
    \includegraphics[width=0.31 \columnwidth]{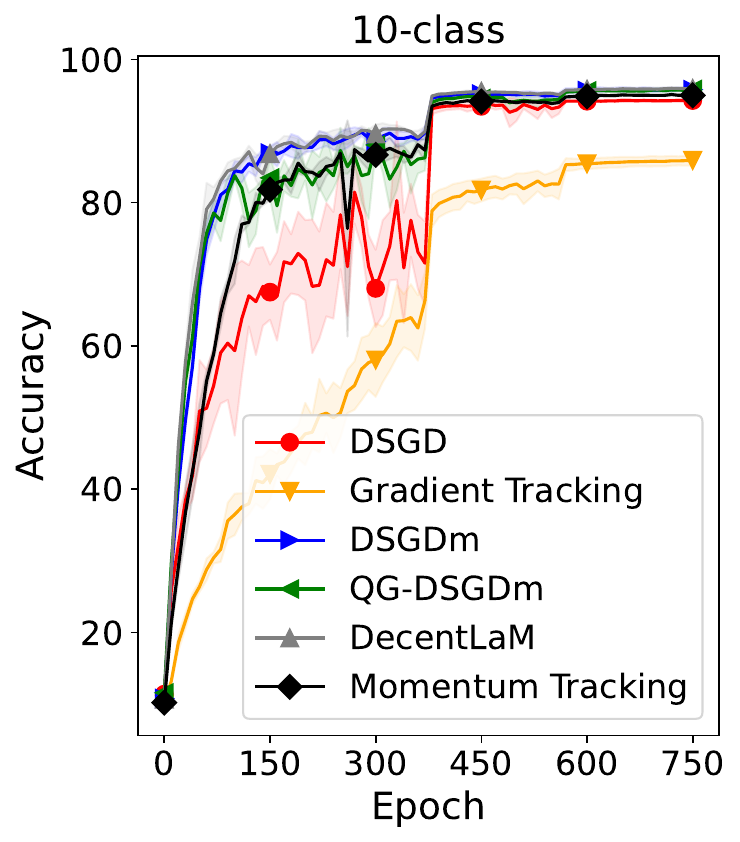}
  }
  \subfigure{
    \centering
    \includegraphics[width=0.31 \columnwidth]{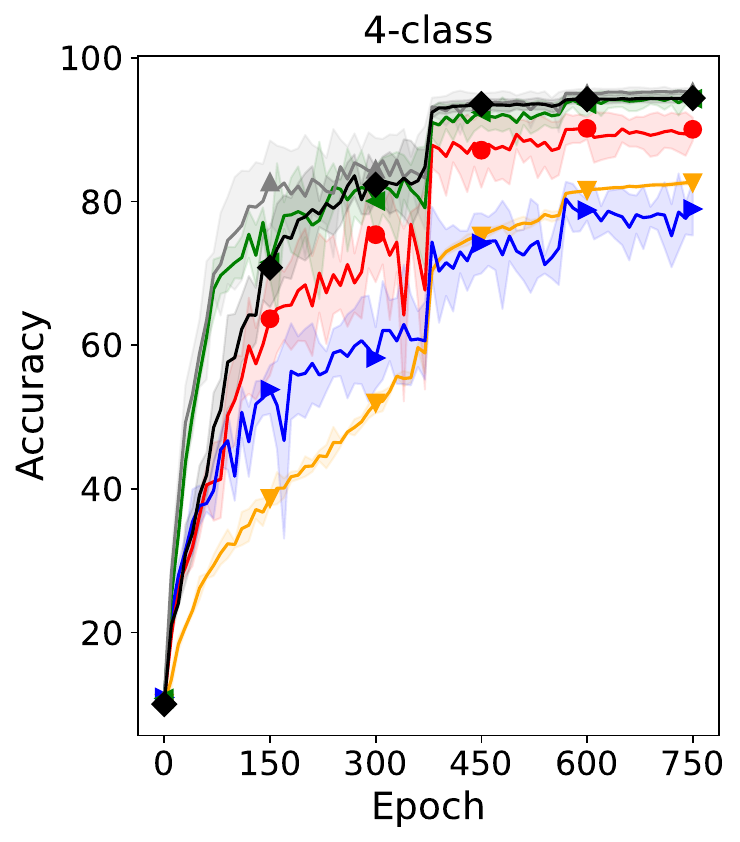}
  }
  \subfigure{
    \centering
    \includegraphics[width=0.31 \columnwidth]{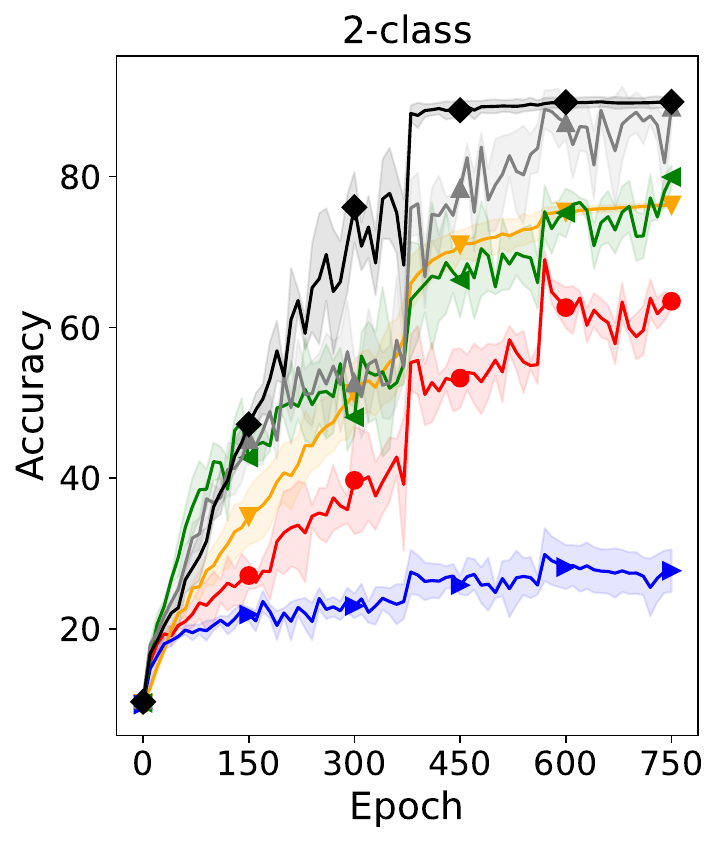}
  }
\vskip -0.1 in
\caption{Learning curves on CIFAR-10 with ResNet-34. The accuracy is evaluated per 10 epochs.}
\label{fig:learning_curve_cifar_resnet}
\vskip -0.1 in
\end{figure}

\newpage~
\newpage
\subsection{Synthetic Experiment}
\label{sec:synthetic}
In this section, we evaluate the convergence rate in more detail using a synthetic dataset.
Following the previous work \citep{koloskova2020unified}, we set the dimension of parameter $d$ to $50$, the number of nodes $N$ to $25$, and the network topology to a ring consisting of $N$ nodes.
We then defined the local objective function $f_i (\vx)$ to be $\frac{1}{2} \| \mA_i \vx - \vb_i \|^2$ where $\mA_i \coloneqq i / \sqrt{N}$ and $\vb_i$ are sampled from $\mathcal{N} (\mathbf{0}, \zeta^2 / i^2 \mathbf{1})$,
and we defined the stochastic gradient $\nabla F_i (\vx ; \xi_i)$ to be $\nabla f_i (\vx) + \epsilon$ where $\epsilon$ is drawn from $\mathcal{N} (\mathbf{0} ;  \sigma^2 / d \mathbf{1})$.
For all comparison methods, we set the step size $\eta$ to $1.0 \times 10^{-4}$.

Figs.~\ref{fig:rate_begining} and \ref{fig:rate} illustrate $\| \nabla f (\bar{\vx}) \|^2$ with respect to the number of rounds $r$ when we vary data heterogeneity $\zeta^2$ as $\{ 0, 25, 50 \}$ and set $\sigma^2$ to one.
The results show that Momentum Tracking converges in the same manner regardless of data heterogeneity $\zeta^2$.
On the other hand, for DSGDm, QG-DSGDm, and DecentLaM,
$\| \nabla f (\bar{\vx}) \|^2$ increases as data heterogeneity $\zeta^2$ increases.
Hence, these results are consistent with our theoretical analysis that the convergence rate of Momentum Tracking is independent of data heterogeneity.

\begin{figure}[!h]
\centering
\vskip -0.1 in
\includegraphics[width=0.95 \columnwidth]{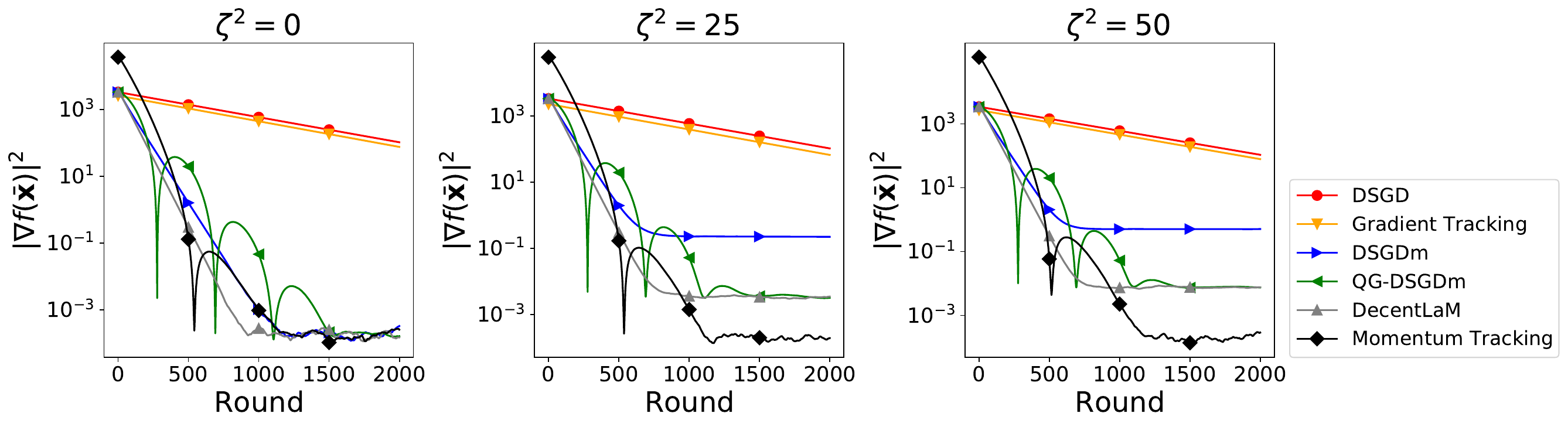}
\vskip -0.1 in
\caption{Comparison of the convergence in the initial training phase in the synthetic experiments.}
\label{fig:rate_begining}
\vskip -0.1 in
\end{figure}

\begin{figure}[!h]
\centering
\vskip -0.1 in
\includegraphics[width=0.95 \columnwidth]{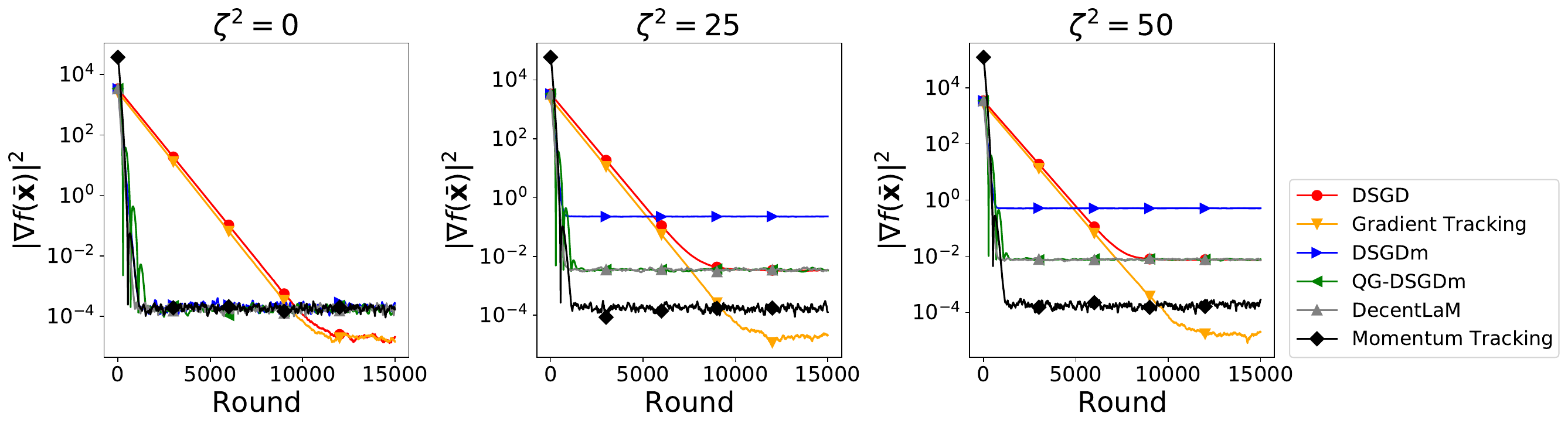}
\vskip -0.1 in
\caption{Comparison of the convergence in the synthetic experiments.}
\label{fig:rate}
\vskip -0.1 in
\end{figure}

\newpage

\section{Proof of Theorem \ref{theorem:convergence_rate}}
\label{sec:proof}

\subsection{Technical Lemma}
\begin{lemma}
For any $\vx, \vy \in \mathbb{R}^d$, $\gamma>0$, it holds that
\begin{align}
\label{eq:triangle}
    \| \vx + \vy \|^2 \leq (1+\gamma) \| \vx \|^2 + (1 + \gamma^{-1}) \| \vy \|^2.
\end{align}
\end{lemma}

\begin{lemma}
For any $\va_1, \cdots, \va_N$, it holds that
\begin{align}
\label{eq:sum}
    \left\| \sum_{i=1}^{N} \va_i \right\|^2 \leq N \sum_{i=1}^N \| \va_i \|^2.
\end{align}
\end{lemma}

\begin{lemma}
For any $\vx, \vy \in \mathbb{R}^d$, $\gamma > 0$, it holds that
\begin{align}
\label{eq:inner_prod}
    2 \langle \vx, \vy \rangle \leq \gamma \| \vx \|^2 + \gamma^{-1} \| \vy \|^2.
\end{align}
\end{lemma}

\subsection{Momentum Tracking in Matrix Notation}
We define $\mU^{(r)}$, $\mX^{(r)}$, $\mC^{(r)}$, $\nabla F (\mX^{(r)} ; \xi^{(r)})$, and $\nabla f (\mX^{(r)})$ as follows:
\begin{gather*}
    \mU^{(r)} \coloneqq \left( \vu_1^{(r)}, \cdots, \vu_N^{(r)} \right), \;\;
    \mX^{(r)} \coloneqq \left( \vx_1^{(r)}, \cdots, \vx_N^{(r)} \right), \;\;
    \mC^{(r)} \coloneqq \left( \vc_1^{(r)}, \cdots, \vc_N^{(r)} \right), \\
    \nabla F (\mX^{(r)} ; \xi^{(r)}) \coloneqq \left( \nabla F_1 (\vx_1^{(r)} ; \xi_1^{(r)}), \cdots, \nabla F_N (\vx_N^{(r)} ; \xi_N^{(r)}) \right), \\
    \nabla f (\mX^{(r)}) \coloneqq \left( \nabla f_1 (\vx_1^{(r)}), \cdots, \nabla f_N (\vx_N^{(r)}) \right).
\end{gather*}
Then, the update rule of Momentum Tracking can then be rewritten as follows:
\begin{align}
    \mU^{(r+1)} &= \beta \mU^{(r)} + \nabla F (\mX^{(r)} ; \xi^{(r)}), \\
    \mX^{(r+1)} &= \mX^{(r)} \mW - \eta (\mU^{(r+1)} - \mC^{(r)}), \\
    \mC^{(r+1)} &= (\mC^{(r)} - \mU^{(r+1)}) \mW + \mU^{(r+1)},
\end{align}
where $\mU^{(0)}$ and $\mC^{(0)}$ are initialized as follows:
\begin{align*}
    \mU^{(0)} &= \frac{1}{1-\beta} (\nabla F (\mX^{(0)} ; \xi^{(0)}) - \frac{1}{N} \nabla F (\mX^{(0)} ; \xi^{(0)}) \mathbf{1}\mathbf{1}^\top ) , \\
    \mC^{(0)} &= \frac{1}{1-\beta} (\nabla F (\mX^{(0)} ; \xi^{(0)}) - \frac{1}{N} \nabla F (\mX^{(0)} ; \xi^{(0)}) \mathbf{1}\mathbf{1}^\top ) .
\end{align*}

\subsection{Additional Notation}
We define the update rules of $\vd_i$ and $\ve_i$ as follows:
\begin{align}
    \vd_i^{(r+1)} &= \beta \vd_i^{(r)} + \nabla f_i(\bar{\vx}^{(r)}), \\
    \ve_i^{(r+1)} &= \beta \ve_i^{(r)} + \nabla f(\bar{\vx}^{(r)}),
\end{align}
where $\vd_i^{(0)} = \frac{1}{1-\beta} (\nabla f_i (\bar{\vx}^{(0)}) - \nabla f (\bar{\vx}^{(0)}))$ and $\ve_i^{(0)} = \mathbf{0}$.
Note that it holds that $\bar{\vd}^{(r)} = \bar{\ve}^{(r)}$ for any round $r \geq 0$.
Then, we define $\mD$ and $\mE$ as follows:
\begin{align*}
    \mD^{(r)} \coloneqq \left( \vd_1^{(r)}, \cdots, \vd_N^{(r)} \right), \;\;
    \mE^{(r)} \coloneqq \left( \ve_1^{(r)}, \cdots, \ve_N^{(r)} \right). \;\;
\end{align*}
The update rules of $\mD$ and $\mE$ can be written as follows:
\begin{align}
    \mD^{(r+1)} &= \beta \mD^{(r)} + \nabla f(\bar{\mX}^{(r)}), \\
    \mE^{(r+1)} &= \beta \mE^{(r)} + \frac{1}{N} \nabla f(\bar{\mX}^{(r)}) \mathbf{1}\mathbf{1}^\top,
\end{align}
where $\mD^{(0)}$ and $\mE^{(0)}$ are initialized as follows:
\begin{align*}
    \mD^{(0)} &= \frac{1}{1-\beta} (\nabla f (\bar{\mX}^{(0)}) - \frac{1}{N} \nabla f (\bar{\mX}^{(0)}) \mathbf{1} \mathbf{1}^\top ), \\
    \mE^{(0)} &= \mathbf{0}.
\end{align*}
Note that $\vd_i$, $\ve_i$, $\mD$, and $\mE$ are the only variables used in the proof that do not need to be computed in practice in Alg. \ref{alg:momentum_trackng}.
We define $\Xi$, $\mathcal{E}$, and $\mathcal{D}$ as follows:
\begin{align*}
    \Xi^{(r)} &\coloneqq \frac{1}{N} \mathbb{E} \left\| \mX^{(r)} - \bar{\mX}^{(r)} \right\|^2_F, \\ 
    \mathcal{E}^{(r)} &\coloneqq \frac{1}{N} \mathbb{E} \left\| \mD^{(r+1)} - \mC^{(r)} - \mE^{(r+1)} \right\|^2_F, \\
    \mathcal{D}^{(r)} &\coloneqq \frac{1}{N} \mathbb{E} \left\| \mD^{(r+1)} - \mD^{(r)} - \mE^{(r+1)} + \mE^{(r)} \right\|^2_F.
\end{align*}

Inspired by \citet{yu2019linear}, we define $\bar{\vz}$ as follows:
\begin{align*}
    \bar{\vz}^{(r)} \coloneqq 
    \begin{cases}
    \bar{\vx}^{(r)}, & \text{if} \;\; r=0 \\
    \frac{1}{1-\beta} \bar{\vx}^{(r)} - \frac{\beta}{1-\beta} \bar{\vx}^{(r-1)}, & \text{otherwise}
    \end{cases}.
\end{align*}
In the following, we define $\pm a \coloneqq a - a = 0$ for any $a$
and $\bar{a} \coloneqq \frac{1}{N} \sum_{i=1}^N a_i$ for any $a_1, \cdots, a_N$.
Then, $\mathbb{E} [\cdot]$ denotes the expectation over all randomness that occurs during training (i.e., $\{ \xi_i^{(r)} \}_{i, r}$),
and $\mathbb{E}_r [\cdot]$ denotes the expectation over the randomness that occurs at round $r$ (i.e., $\{ \xi_i^{(r)} \}_{i}$).

\subsection{Proof Sketch}
In this section, we briefly summarize our proof technique. Our proof is based on the analysis of DSGD \cite{koloskova2020unified} that uses the upper bound of the consensus error $\Xi$. 
We extend their technique to deal with data heterogeneity by decomposing the inequality about the consensus error $\Xi$ into
(i) the inequality of the error between global and corrected local momentum $\mathcal{E}$
and (ii) the inequality of the error between update values of global and uncorrected local momentum $\mathcal{D}$.
While bounding the latter error is rather straightforward, bounding the former error requires our momentum correction term, 
which makes the error recursion contractive; otherwise, the error between global and local momentum results in the heterogeneity term in the final convergence error.
In the following, we show the proof sketch and explain in more detail how to bound $\Xi$ from above without using $\zeta^2$.

We derive the inequality about the consensus error $\Xi$ as follows (See Lemma \ref{lemma:recursion_for_consensus_distance}):
\begin{align*}
    \Xi^{(r+1)} 
    \leq \left(1 - \frac{p}{2}\right) \Xi^{(r)} 
    + \frac{9}{p} \eta^2 \mathcal{E}^{(r)}
    + \frac{9}{N p} \eta^2 \sum_{i=1}^{N} \left( \mathbb{E} \left\|  \mathbf{u}_i^{(r+1)} - \mathbf{d}_i^{(r+1)} \right\|^2 
    +  \mathbb{E} \left\| \mathbf{e}_i^{(r+1)} \right\|^2 \right)
\end{align*}
$\mathcal{E}$ represents the the error between global momentum $\mathbf{e}_i (= \bar{\mathbf{e}})$ and corrected local momentum $(\mathbf{d}_i - \mathbf{c}_i$).
Thus, intuitively, if we remove the tracking term $\mathbf{c}_i$ from Momentum Tracking, 
$\mathcal{E}$ becomes $\frac{1}{N} \sum_i \mathbb{E} \| \mathbf{d}_i  - \mathbf{e}_i \|^2$, which causes the data heterogeneity $\zeta^2$ to appear in the upper bound of $\Xi$.
In the following, we explain how to eliminate $\mathcal{E}^{(r)}$ and $\zeta^2$ from the upper bound of $\Xi^{(r+1)}$,
which is the most important component of our proof.

To bound $\mathcal{E}$ from above, we derive the following two inequalities
(see Lemmas \ref{lemma:recursion_for_e} and \ref{lemma:recursion_for_d}):
\begin{align*}
    \mathcal{E}^{(r+1)} 
    &\leq \left(1 - \frac{p}{2}\right) \mathcal{E}^{(r)} 
    + \frac{18 \beta^2}{p} \mathcal{D}^{(r)}
    + \frac{144 L^4}{p} \eta^2 \Xi^{(r)} 
    + C_1, \\
    \mathcal{D}^{(r+1)} 
    &\leq \frac{2 \beta^2}{1+\beta^2} \mathcal{D}^{(r)}  
    + \frac{32 L^4 \eta^2}{1 - \beta^2} \Xi^{(r)}
    + C_2,
\end{align*}
where $C_1$ and $C_2$ are variables independent of $\Xi$, $\mathcal{D}$, $\mathcal{E}$, and $\zeta^2$.
Here, the most important benefit to adding the tracking term $\mathbf{c}_i$ is that the coefficient of $\mathcal{E}^{(r)}$ becomes less than 1. (i.e., $(1 - \frac{p}{2}) < 1$).
Roughly speaking, the above two inequalities imply that $\mathcal{E}$ and $\mathcal{D}$ become gradually smaller
because $(1 - \frac{p}{2}) < 1$ and $\frac{2 \beta^2}{1 + \beta^2} < 1$ hold for any $\beta \in [0,1)$.

Next, we derive a new inequality by combining the above three inequalities.
We define an auxiliary error term $\Theta$ as follows:
\begin{align*}
    \Theta^{(r)} \coloneqq \Xi^{(r)} +  \frac{36}{p^2} \eta^2 \mathcal{E}^{(r)} + \frac{A \beta^2}{p^3} \eta^2 \mathcal{D}^{(r)},
\end{align*}
where $A > 0$ is defined in Lemma \ref{lemma:recusion_for_three_terms}.
Then, by combining the above three inequalities, we obtain the following (see Lemma \ref{lemma:recusion_for_three_terms}):
\begin{align*}
    \Theta^{(r+1)} 
    &\leq \left(1 - \frac{p}{t}\right) \Theta^{(r)} 
    + C_3,
\end{align*}
where $C_3$ is a variable independent of $\Xi$, $\mathcal{D}$, $\mathcal{E}$ and $\zeta^2$, and $t \geq 4$ is defined in Lemma \ref{lemma:recusion_for_three_terms}.

Using $\Xi \leq \Theta$ and applying the above inequality recursively, we obtain (see Lemma \ref{lemma:simplified_recursion})
\begin{align*}
    \Xi^{(r+1)} \leq \Theta^{(r+1)}
    \leq\left(1 - \frac{p}{t}\right)^{r+1} \Theta^{(0)}
    + C_4,
\end{align*}
where $C_4$ is a variable independent of $\Xi$, $\mathcal{D}$, $\mathcal{E}$, and $\zeta^2$.

Using $(1 - \frac{p}{t}) < 1$, it holds that $\sum_{r=0}^R (1 - \frac{p}{t})^{r+1} \Theta^{(0)} \leq \frac{t}{p} \Theta^{(0)}$.
Finally, using this inequality and the fact that $C_4$ and $\Theta^{(0)}$ are independent of $\zeta^2$ due to the assumption of initial values,
we can eliminate $\mathcal{E}$ from the upper bound of $\Xi$ and derive the upper bound of $\Xi$ that does not contain $\zeta^2$ as follows (see Lemma \ref{lemma:sum_of_consensus}):
\begin{align*}
    \frac{4 L^2}{R+1} \sum_{r=0}^{R} \Xi^{(r)}
    &\leq \frac{1}{2 (R+1)} \sum_{k=0}^{R} \left\| \nabla f(\bar{\mathbf{x}}^{(k)}) \right\|^2 
    + \frac{40 L^2 t}{(1-\beta)^3 p^2} \left( 10 + \frac{29}{p} + \frac{864}{p^2} \right) \sigma^2 \eta^2.
\end{align*}
These are the essential techniques to bound $\sum_{r=0}^R \Xi^{(r)}$ from above without using $\zeta^2$ and make the convergence rate independent of $\zeta^2$.

\newpage

\subsection{Useful Lemma}
\begin{lemma}
\label{lemma:sum_c}
For any round $r \geq 0$, it holds that $\bar{\vc}^{(r)}=\mathbf{0}$.
\end{lemma}
\begin{proof}
For any round $r\geq 0$, we have
\begin{align*}
    \sum_{i=1}^{N} \vc_i^{(r+1)}
    &= \sum_{i=1}^{N} \sum_{j=1}^{N} W_{ij} (\vc_j^{(r)} - \vu_j^{(r+1)}) + \sum_{i=1}^{N} \vu_i^{(r+1)} \\
    &= \sum_{j=1}^{N} (\vc_j^{(r)} - \vu_j^{(r+1)}) \sum_{i=1}^{N} W_{ij} + \sum_{i=1}^{N} \vu_i^{(r+1)}.
\end{align*}
Because $\mW$ is a mixing matrix, we obtain
\begin{align*}
    \sum_{i=1}^{N} \vc_i^{(r+1)}
    &= \sum_{j=1}^{N} \vc_j^{(r)}.
\end{align*}
Since we have
\begin{align*}
    \sum_{i=1}^N \vc_i^{(0)} 
    &= \frac{1}{1-\beta} \sum_{i=1}^N \left( \nabla F_i (\vx_i^{(0)} ; \xi_i^{(0)}) - \frac{1}{N} \sum_{j=1}^N \nabla F_j (\vx_j^{(0)} ; \xi_j^{(0)}) \right)
    = \mathbf{0},
\end{align*}
we obtain the statement.
\end{proof}

\begin{lemma}
\label{lemma:average_property}
For any round $r\geq 0$, it holds that
\begin{align*}
    \bar{\vx}^{(r+1)} = \bar{\vx}^{(r)} - \eta \bar{\vu}^{(r+1)}.
\end{align*}
\end{lemma}
\begin{proof}
We have
\begin{align*}
    \bar{\vx}^{(r+1)} 
    &= \frac{1}{N} \sum_{i=1}^{N} \sum_{j=1}^{N} W_{ij} \vx_j^{(r)} - \eta (\bar{\vu}^{(r+1)} - \bar{\vc}^{(r)}) \\
    &= \frac{1}{N} \sum_{j=1}^{N} \vx_j^{(r)} \sum_{i=1}^{N} W_{ij}  - \eta (\bar{\vu}^{(r+1)} - \bar{\vc}^{(r)}).
\end{align*}
The fact that $\mW$ is a mixing matrix gives us 
\begin{align*}
    \bar{\vx}^{(r+1)} 
    &= \bar{\vx}^{(r)} - \eta (\bar{\vu}^{(r+1)} - \bar{\vc}^{(r)}).
\end{align*}
Then, using Lemma \ref{lemma:sum_c}, we get the statement.
\end{proof}

\begin{lemma}
\label{lemma:z_update}
For any round $r\geq 0$, it holds that
\begin{align*}
    \bar{\vz}^{(r+1)} - \bar{\vz}^{(r)} = - \frac{\eta}{1-\beta} \frac{1}{N} \sum_{i=1}^N \nabla F_i (\vx^{(r)}_i ; \xi_i^{(r)}).
\end{align*}
\end{lemma}
\begin{proof}
For any $r \geq 1$, we have
\begin{align*}
    \bar{\vz}^{(r+1)} - \bar{\vz}^{(r)} 
    &= \frac{1}{1-\beta} (\bar{\vx}^{(r+1)}- \bar{\vx}^{(r)}) - \frac{\beta}{1-\beta} (\bar{\vx}^{(r)} - \bar{\vx}^{(r-1)}) \\
    &= - \frac{\eta}{1-\beta} \bar{\vu}^{(r+1)} + \frac{\eta \beta}{1-\beta} \bar{\vu}^{(r)} \\
    &= - \frac{\eta}{1-\beta} \frac{1}{N} \sum_{i=1}^N \nabla F_i(\vx_i^{(r)} ; \xi_i^{(r)}),
\end{align*}
where we use Lemma \ref{lemma:average_property}.
When $r=0$, we have
\begin{align*}
    \bar{\vz}^{(1)} - \bar{\vz}^{(0)} 
    &= \frac{1}{1-\beta} ( \bar{\vx}^{(1)} - \bar{\vx}^{(0)} ) \\
    &= - \frac{\eta}{1-\beta} \bar{\vu}^{(1)} \\
    &= - \frac{\eta}{1-\beta} \frac{1}{N} \sum_{i=1}^N \nabla F_i(\vx_i^{(0)} ; \xi_i^{(0)}),
\end{align*}
where we use $\bar{\vu}^{(0)} = \mathbf{0}$.
This concludes the proof.
\end{proof}

\begin{lemma}
\label{lemma:x_bar_minus_z_bar}
Suppose that Assumptions \ref{assumption:lower_bound}, \ref{assumption:mixing}, \ref{assumption:smoothness}, and \ref{assumption:stochasic_gradient} hold.
For any round $R \geq 0$, it holds that
\begin{align*}
    \sum_{r=0}^R \mathbb{E} \left\| \bar{\vx}^{(r)} - \bar{\vz}^{(r)} \right\|^2
    &\leq \frac{\beta^2 \eta^2 }{(1-\beta)^4} \sum_{r=0}^{R} \mathbb{E} \left\| \frac{1}{N} \sum_{i=1}^N \nabla f_i (\vx_i^{(r)}) \right\|^2
    + \frac{\beta^2 \sigma^2 \eta^2 }{N (1-\beta)^4} R.
\end{align*}
\end{lemma}
\begin{proof}
From Lemma \ref{lemma:average_property}, we have
\begin{align*}
    \mathbb{E} \left\| \bar{\vx}^{(r)} - \bar{\vz}^{(r)} \right\|^2
    &= \mathbb{E} \left\| \frac{\eta \beta}{1-\beta} \bar{\vu}^{(r)} \right\|^2 \\
    &=  \frac{\beta^2 \eta^2 }{(1-\beta)^2} \mathbb{E} \left\| \sum_{k=0}^{r-1} \beta^{r-k-1} \frac{1}{N} \sum_{i=1}^N \nabla F_i (\vx_i^{(k)} ; \xi_i^{(k)}) \right\|^2,
\end{align*}
for any $r \geq 1$.
Defining $s^{(r)} \coloneqq \sum_{k=0}^{r} \beta^{r-k}$, we obtain
\begin{align*}
    &\mathbb{E} \left\| \bar{\vx}^{(r)} - \bar{\vz}^{(r)} \right\|^2 \\
    &= \frac{\beta^2 \eta^2 }{(1-\beta)^2} {s^{(r-1)}}^2 \mathbb{E} \left\| \sum_{k=0}^{r-1} \frac{\beta^{r-k-1}}{s^{(r-1)}} \frac{1}{N} \sum_{i=1}^N \nabla F_i (\vx_i^{(k)} ; \xi_i^{(k)}) \right\|^2 \\
    &\stackrel{(a)}{\leq} \frac{\beta^2 \eta^2 }{(1-\beta)^2} s^{(r-1)} \sum_{k=0}^{r-1} \beta^{r-k-1} \mathbb{E} \left\| \frac{1}{N} \sum_{i=1}^N \nabla F_i (\vx_i^{(k)} ; \xi_i^{(k)}) \right\|^2 \\
    &\stackrel{(\ref{eq:stochastic_gradient})}{\leq} \frac{\beta^2 \eta^2 }{(1-\beta)^2} s^{(r-1)} \sum_{k=0}^{r-1} \beta^{r-k-1} \mathbb{E} \left\| \frac{1}{N} \sum_{i=1}^N \nabla f_i (\vx_i^{(k)}) \right\|^2
    + \frac{\beta^2 \sigma^2 \eta^2 }{N (1-\beta)^2} s^{(r-1)} \sum_{k=0}^{r-1} \beta^{r-k-1},
\end{align*}
where we use Jensen's inequality in (a).
Using $s^{(r-1)} \leq \frac{1}{1-\beta}$, we obtain
\begin{align*}
    \mathbb{E} \left\| \bar{\vx}^{(r)} - \bar{\vz}^{(r)} \right\|^2
    &\leq \frac{\beta^2 \eta^2 }{(1-\beta)^3} \sum_{k=0}^{r-1} \beta^{r-k-1} \mathbb{E} \left\| \frac{1}{N} \sum_{i=1}^N \nabla f_i (\vx_i^{(k)}) \right\|^2
    + \frac{\beta^2 \sigma^2 \eta^2 }{N (1-\beta)^4}.
\end{align*}
Recursive addition yields
\begin{align*}
    \sum_{r=1}^R \mathbb{E} \left\| \bar{\vx}^{(r)} - \bar{\vz}^{(r)} \right\|^2
    &\leq \frac{\beta^2 \eta^2 }{(1-\beta)^3} \sum_{r=1}^R \sum_{k=0}^{r-1} \beta^{r-k-1} \mathbb{E} \left\| \frac{1}{N} \sum_{i=1}^N \nabla f_i (\vx_i^{(k)}) \right\|^2
    + \frac{\beta^2 \sigma^2 \eta^2 }{N (1-\beta)^4} R \\
    &= \frac{\beta^2 \eta^2 }{(1-\beta)^3} \sum_{k=0}^{R-1} \mathbb{E} \left\| \frac{1}{N} \sum_{i=1}^N \nabla f_i (\vx_i^{(k)}) \right\|^2 \sum_{r=k+1}^R \beta^{r-k-1}
    + \frac{\beta^2 \sigma^2 \eta^2 }{N (1-\beta)^4} R \\
    &\leq \frac{\beta^2 \eta^2 }{(1-\beta)^4} \sum_{k=0}^{R-1} \mathbb{E} \left\| \frac{1}{N} \sum_{i=1}^N \nabla f_i (\vx_i^{(k)}) \right\|^2
    + \frac{\beta^2 \sigma^2 \eta^2 }{N (1-\beta)^4} R,
\end{align*}
where we use $\sum_{r=k+1}^R \beta^{r-k-1} \leq \frac{1}{1-\beta}$ in the last inequality.
From the definition of $\bar{\vz}^{(0)}$, we have $\| \bar{\vx}^{(0)} - \bar{\vz}^{(0)} \|^2=0$. 
This yields the statement.
\end{proof}

\begin{lemma}
\label{lemma:x_bar_update}
Suppose that Assumptions \ref{assumption:lower_bound}, \ref{assumption:mixing}, \ref{assumption:smoothness}, and \ref{assumption:stochasic_gradient} hold.
For any round $r \geq 0$, it holds that
\begin{align*}
    \mathbb{E} \left\| \bar{\vx}^{(r+1)} - \bar{\vx}^{(r)} \right\|^2
    &\leq 
    4 L^2 \eta^2 \Xi^{(r)}
    + 2 \beta^2 \eta^2 \mathbb{E} \left\| \bar{\vu}^{(r)} \right\|^2 
    + 4 \eta^2 \mathbb{E} \left\| \nabla f (\bar{\vx}^{(r)}) \right\|^2
    + \frac{\sigma^2 \eta^2}{N}.
\end{align*}
\end{lemma}
\begin{proof}
From Lemma \ref{lemma:average_property}, we have
\begin{align*}
    \mathbb{E}_r \left\| \bar{\vx}^{(r+1)} - \bar{\vx}^{(r)} \right\|^2
    &= \eta^2 \mathbb{E}_r \left\| \beta \bar{\vu}^{(r)} + \frac{1}{N} \sum_{i=1}^N \nabla F_i (\vx_i^{(r)} ; \xi_i^{(r)}) \right\|^2 \\
    &\stackrel{(\ref{eq:stochastic_gradient})}{\leq} \eta^2 \left\| \beta \bar{\vu}^{(r)} + \frac{1}{N} \sum_{i=1}^N \nabla f_i (\vx_i^{(r)}) \right\|^2
    + \frac{\sigma^2 \eta^2}{N} \\
    &\stackrel{(\ref{eq:triangle})}{\leq} 2 \beta^2 \eta^2 \left\| \bar{\vu}^{(r)} \right\|^2 +  2 \eta^2 \underbrace{\left\| \frac{1}{N} \sum_{i=1}^N \nabla f_i (\vx_i^{(r)}) \right\|^2}_{T}
    + \frac{\sigma^2 \eta^2}{N}.
\end{align*}
Then, $T$ can be bounded from above as follows:
\begin{align*}
    T
    &= \left\| \frac{1}{N} \sum_{i=1}^N \nabla f_i (\vx_i^{(r)}) \pm \nabla f_i (\bar{\vx}^{(r)}) \right\|^2 \\
    &\stackrel{(\ref{eq:triangle})}{\leq} 2 \left\| \frac{1}{N} \sum_{i=1}^N \nabla f_i (\vx_i^{(r)}) - \nabla f_i (\bar{\vx}^{(r)}) \right\|^2 
    + 2 \left\| \nabla f (\bar{\vx}^{(r)}) \right\|^2 \\
    &\stackrel{(\ref{eq:sum})}{\leq} \frac{2}{N} \sum_{i=1}^N \left\| \nabla f_i (\vx_i^{(r)}) - \nabla f_i (\bar{\vx}^{(r)}) \right\|^2 
    + 2 \left\| \nabla f (\bar{\vx}^{(r)}) \right\|^2 \\
    &\stackrel{(\ref{eq:smoothness})}{\leq} \frac{2 L^2}{N} \sum_{i=1}^N \left\| \vx_i^{(r)} - \bar{\vx}^{(r)} \right\|^2 
    + 2 \left\| \nabla f (\bar{\vx}^{(r)}) \right\|^2.
\end{align*}
Then, we obtain the statement.
\end{proof}

\begin{lemma}
\label{lemma:e}
For any round $r \geq 0$, it holds that
\begin{align*}
    \mathbb{E} \left\| \bar{\ve}^{(r+1)} \right\|^2
    &\leq \frac{1}{1-\beta} \sum_{k=0}^{r} \beta^{r-k} \mathbb{E} \left\| \nabla f(\bar{\vx}^{(k)}) \right\|^2.
\end{align*}
\end{lemma}
\begin{proof}
We have
\begin{align*}
    \mathbb{E} \left\| \bar{\ve}^{(r+1)} \right\|^2
    &= \mathbb{E} \left\| \sum_{k=0}^{r} \beta^{r-k} \nabla f(\bar{\vx}^{(k)}) \right\|^2,
\end{align*}
where we use $\bar{\ve}^{(0)} = \mathbf{0}$.
Defining $s^{(r)} \coloneqq \sum_{k=0}^{r} \beta^{r-k}$, we obtain
\begin{align*}
    \mathbb{E} \left\| \bar{\ve}^{(r+1)} \right\|^2
    &= {s^{(r)}}^2 \mathbb{E} \left\| \sum_{k=0}^{r} \frac{\beta^{r-k}}{s^{(r)}} \nabla f(\bar{\vx}^{(k)}) \right\|^2 \\
    &\leq s^{(r)} \sum_{k=0}^{r} \beta^{r-k} \mathbb{E} \left\| \nabla f(\bar{\vx}^{(k)}) \right\|^2,
\end{align*}
where we use Jensen's inequality.
Using $s^{(r)} \leq \frac{1}{1-\beta}$, we obtain the statement.
\end{proof}

\begin{lemma}
\label{lemma:u_minus_d}
Suppose that Assumptions \ref{assumption:lower_bound}, \ref{assumption:mixing}, \ref{assumption:smoothness}, and \ref{assumption:stochasic_gradient} hold.
For any round $r \geq 0$, it holds that
\begin{align*}
    \frac{1}{N} \mathbb{E} \left\| \mD^{(r+1)} - \mU^{(r+1)} \right\|^2_F
    \leq \frac{L^2}{1-\beta} \sum_{k=0}^{r} \beta^{r-k} \Xi^{(k)} 
    + \frac{5 \sigma^2}{(1-\beta)^3}.
\end{align*}
\end{lemma}
\begin{proof}
We have
\begin{align*}
    &\mathbb{E} \left\| \mD^{(r+1)} - \mU^{(r+1)} \right\|^2_F \\
    &= \mathbb{E} \left\| \sum_{k=0}^{r} \beta^{r-k} (\nabla f (\bar{\mX}^{(k)}) -  \nabla F (\mX^{(k)} ; \xi^{(k)}) ) + \beta^{r+1} (\mD^{(0)} - \mU^{(0)}) \right\|^2_F.
\end{align*}
Defining $s^{(r)} \coloneqq \sum_{k=0}^r \beta^{r-k}$, we obtain
\begin{align*}
    &\mathbb{E} \left\| \mD^{(r+1)} - \mU^{(r+1)} \right\|^2_F \\
    &= {s^{(r+1)}}^2 \mathbb{E} \left\| \sum_{k=0}^{r} \frac{\beta^{r-k}}{s^{(r+1)}} (\nabla f (\bar{\mX}^{(k)}) -  \nabla F (\mX^{(k)} ; \xi^{(k)}) ) + \frac{\beta^{r+1}}{s^{(r+1)}} (\mD^{(0)} - \mU^{(0)}) \right\|^2_F \\
    &\stackrel{(a)}{\leq} s^{(r+1)} \sum_{k=0}^{r} \beta^{r-k} \mathbb{E} \left\| \nabla f (\bar{\mX}^{(k)}) -  \nabla F (\mX^{(k)} ; \xi^{(k)}) \right\|^2_F
    + s^{(r+1)} \beta^{r+1} \mathbb{E} \left\| \mD^{(0)} - \mU^{(0)} \right\|^2_F \\
    &\stackrel{(\ref{eq:sum})}{\leq} s^{(r+1)} \sum_{k=0}^{r} \beta^{r-k} \mathbb{E} \left\| \nabla f (\bar{\mX}^{(k)}) -  \nabla F (\mX^{(k)} ; \xi^{(k)}) \right\|^2_F \\
    &\qquad + \frac{2 s^{(r+1)}}{ (1-\beta)^2 } \beta^{r+1} \mathbb{E} \left\| \nabla f (\bar{\mX}^{(0)}) - \nabla F (\mX^{(0)} ; \xi^{(0)}) \right\|^2_F \\
    &\qquad + \frac{2 s^{(r+1)}}{ (1-\beta)^2 } \beta^{r+1} \mathbb{E} \left\| \frac{1}{N} \nabla f (\bar{\mX}^{(0)}) \mathbf{1} \mathbf{1}^\top - \frac{1}{N} \nabla F (\mX^{(0)} ; \xi^{(0)}) \mathbf{1}\mathbf{1}^\top \right\|^2_F \\
    &\stackrel{(\ref{eq:stochastic_gradient})}{\leq} s^{(r+1)} \sum_{k=0}^{r} \beta^{r-k} \mathbb{E} \left\| \nabla f (\bar{\mX}^{(k)}) -  \nabla f (\mX^{(k)}) \right\|^2_F 
    + s^{(r+1)} \sum_{k=0}^{r} \beta^{r-k} N \sigma^2 
    + \frac{4 s^{(r+1)}}{ (1-\beta)^2 } \beta^{r+1} N \sigma^2,
\end{align*}
where we use Jensen's inequality for (a) and use $\mX^{(0)} = \bar{\mX}^{(0)}$ for the last inequality.
Then, using $s^{(r)} \leq \frac{1}{1-\beta}$, we obtain
\begin{align*}
    &\mathbb{E} \left\| \mD^{(r+1)} - \mU^{(r+1)} \right\|^2_F \\
    &\leq \frac{1}{1-\beta} \sum_{k=0}^{r} \beta^{r-k} \mathbb{E} \left\| \nabla f (\bar{\mX}^{(k)}) -  \nabla f (\mX^{(k)}) \right\|^2_F 
    + \frac{N \sigma^2}{(1-\beta)^2} 
    + \frac{4 N \sigma^2}{ (1-\beta)^3 } \beta^{r+1} \\
    &\stackrel{\beta \in [0,1)}{\leq} \frac{1}{1-\beta} \sum_{k=0}^{r} \beta^{r-k} \mathbb{E} \left\| \nabla f (\bar{\mX}^{(k)}) -  \nabla f (\mX^{(k)}) \right\|^2_F 
    + \frac{5 N \sigma^2}{(1-\beta)^3} \\
    &\stackrel{(\ref{eq:smoothness})}{\leq} \frac{L^2}{1-\beta} \sum_{k=0}^{r} \beta^{r-k} \mathbb{E} \left\| \bar{\mX}^{(k)} -  \mX^{(k)} \right\|^2_F 
    + \frac{5 N \sigma^2}{(1-\beta)^3}.
\end{align*}
This concludes the proof.
\end{proof}

\begin{lemma}
\label{lemma:u_bar_minus_d_bar}
Suppose that Assumptions \ref{assumption:lower_bound}, \ref{assumption:mixing}, \ref{assumption:smoothness}, and \ref{assumption:stochasic_gradient} hold.
For any round $r \geq 0$, it holds that
\begin{align*}
    \mathbb{E} \left\| \bar{\vu}^{(r+1)} - \bar{\vd}^{(r+1)} \right\|^2
    \leq \frac{L^2}{1-\beta} \sum_{k=0}^r \beta^{r-k} \Xi^{(k)} 
    + \frac{\sigma^2}{N (1-\beta)^2}
\end{align*}
\end{lemma}
\begin{proof}
We have
\begin{align*}
    &\mathbb{E} \left\| \bar{\vu}^{(r+1)} - \bar{\vd}^{(r+1)} \right\|^2
    = \mathbb{E} \left\| \sum_{k=0}^r \beta^{r-k} (\frac{1}{N} \sum_{i=1}^N \nabla F_i (\vx_i^{(k)} ; \xi_i^{(k)}) - \nabla f (\bar{\vx}^{(k)}) ) \right\|^2,
\end{align*}
where we use $\bar{\vu}^{(0)} = \bar{\vd}^{(0)} = \mathbf{0}$.
Defining $s^{(r)} \coloneqq \sum_{k=0}^r \beta^{r-k}$, we obtain
\begin{align*}
    &\mathbb{E} \left\| \bar{\vu}^{(r+1)} - \bar{\vd}^{(r+1)} \right\|^2 \\
    &= {s^{(r)}}^2 \mathbb{E} \left\| \sum_{k=0}^r \frac{\beta^{r-k}}{s^{(r)}} (\frac{1}{N} \sum_{i=1}^N \nabla F_i (\vx_i^{(k)} ; \xi_i^{(k)}) - \nabla f (\bar{\vx}^{(k)}) ) \right\|^2 \\
    &\stackrel{(a)}{\leq} s^{(r)} \sum_{k=0}^r \beta^{r-k} \mathbb{E} \left\| \frac{1}{N} \sum_{i=1}^N \nabla F_i (\vx_i^{(k)} ; \xi_i^{(k)}) - \nabla f (\bar{\vx}^{(k)}) \right\|^2 \\
    &\stackrel{(\ref{eq:stochastic_gradient})}{\leq} s^{(r)} \sum_{k=0}^r \beta^{r-k} \mathbb{E} \left\| \frac{1}{N} \sum_{i=1}^N \nabla f_i (\vx_i^{(k)}) - \nabla f (\bar{\vx}^{(k)}) \right\|^2 
    + s^{(r)} \sum_{k=0}^r \beta^{r-k} \frac{\sigma^2}{N} \\
    &\stackrel{(\ref{eq:sum})}{\leq} s^{(r)} \sum_{k=0}^r \beta^{r-k} \frac{1}{N} \sum_{i=1}^N \mathbb{E} \left\| \nabla f_i (\vx_i^{(k)}) - \nabla f_i (\bar{\vx}^{(k)}) \right\|^2 
    + s^{(r)} \sum_{k=0}^r \beta^{r-k} \frac{\sigma^2}{N},
\end{align*}
where we use Jensen's inequality in (a).
Then, using $s^{(r)} \leq \frac{1}{1-\beta}$, we obtain
\begin{align*}
    &\mathbb{E} \left\| \bar{\vu}^{(r+1)} - \bar{\vd}^{(r+1)} \right\|^2 \\
    &\leq \frac{1}{1-\beta} \sum_{k=0}^r \beta^{r-k} \frac{1}{N} \sum_{i=1}^N \mathbb{E} \left\| \nabla f_i (\vx_i^{(k)}) - \nabla f_i (\bar{\vx}^{(k)}) \right\|^2 
    + \frac{\sigma^2}{N (1-\beta)^2} \\
    &\stackrel{(\ref{eq:smoothness})}{\leq} \frac{L^2}{1-\beta} \sum_{k=0}^r \beta^{r-k} \frac{1}{N} \sum_{i=1}^N \mathbb{E} \left\| \vx_i^{(k)} - \bar{\vx}^{(k)} \right\|^2 
    + \frac{\sigma^2}{N (1-\beta)^2}.
\end{align*}
This concludes the proof.
\end{proof}

\begin{lemma}
\label{lemma:u_bar}
Suppose that Assumptions \ref{assumption:lower_bound}, \ref{assumption:mixing}, \ref{assumption:smoothness}, and \ref{assumption:stochasic_gradient} hold.
For any round $r \geq 0$, it holds that
\begin{align*}
    \mathbb{E} \left\| \bar{\vu}^{(r+1)} \right\|^2
    \leq \frac{2 L^2}{1-\beta} \sum_{k=0}^r \beta^{r-k} \Xi^{(k)} 
    + \frac{2}{1-\beta} \sum_{k=0}^{r} \beta^{r-k} \left\| \nabla f(\bar{\vx}^{(k)}) \right\|^2
    + \frac{2\sigma^2}{N (1-\beta)^2}.
\end{align*}
\end{lemma}
\begin{proof}
We have
\begin{align*}
    \mathbb{E} \left\| \bar{\vu}^{(r+1)} \right\|^2
    &= \mathbb{E} \left\| \bar{\vu}^{(r+1)} \pm \bar{\vd}^{(r+1)} \right\|^2 \\
    &\stackrel{(\ref{eq:triangle})}{\leq} 2 \mathbb{E} \left\| \bar{\vu}^{(r+1)} - \bar{\vd}^{(r+1)} \right\|^2
    + 2 \mathbb{E} \left\| \bar{\vd}^{(r+1)} \right\|^2.
\end{align*}
From Lemmas \ref{lemma:e} and \ref{lemma:u_bar_minus_d_bar}, we obtain the statement.
\end{proof}

\newpage
\subsection{Main Proof}

\begin{lemma}[Descent Lemma]
\label{lemma:descent_lemma}
Suppose that Assumptions \ref{assumption:lower_bound}, \ref{assumption:mixing}, \ref{assumption:smoothness}, and \ref{assumption:stochasic_gradient} hold.
If the step size $\eta$ satisfies
\begin{align*}
    \eta \leq \frac{1-\beta}{4 L},
\end{align*}
then it holds that for any round $r \geq 0$,
\begin{align*}
    \mathbb{E} f (\bar{\vz}^{(r+1)}) 
    &\leq \mathbb{E} f (\bar{\vz}^{(r)}) 
    + \frac{L^2 \eta}{1-\beta} \mathbb{E} \left\| \bar{\vx}^{(r)} - \bar{\vz}^{(r)} \right\|^2  
    + \frac{L^2 \eta}{1-\beta} \Xi^{(r)} \\
    &\qquad - \frac{\eta}{4 (1 - \beta)}  \mathbb{E} \left\| \frac{1}{N} \sum_{i=1}^N \nabla f_i (\vx^{(r)}_i) \right\|^2
    - \frac{\eta}{4 (1-\beta)} \mathbb{E} \left\| \nabla f (\bar{\vx}^{(r)}) \right\|^2 
    + \frac{L \sigma^2 \eta^2}{2 N (1-\beta)^2}.
\end{align*}
\end{lemma}
\begin{proof}
From Assumption \ref{assumption:smoothness} and Lemma \ref{lemma:z_update}, we have
\begin{align*}
    &\mathbb{E}_r f (\bar{\vz}^{(r+1)}) \\
    &\leq f (\bar{\vz}^{(r)}) + \mathbb{E}_r \langle \nabla  f (\bar{\vz}^{(r)}), \bar{\vz}^{(r+1)} - \bar{\vz}^{(r)} \rangle + \frac{L}{2} \mathbb{E}_r \left\| \bar{\vz}^{(r+1)} - \bar{\vz}^{(r)} \right\|^2 \\
    &= f (\bar{\vz}^{(r)}) - \frac{\eta}{1-\beta} \left\langle \nabla  f (\bar{\vz}^{(r)}), \frac{1}{N} \sum_{i=1}^N \nabla f_i (\vx^{(r)}_i) \right\rangle
    + \frac{L \eta^2}{2 (1-\beta)^2} \mathbb{E}_r \left\| \frac{1}{N} \sum_{i=1}^N \nabla F_i (\vx^{(r)}_i ; \xi_i^{(r)}) \right\|^2 \\
    &\stackrel{(\ref{eq:stochastic_gradient})}{\leq} f (\bar{\vz}^{(r)}) - \frac{\eta}{1-\beta} \left\langle \nabla  f (\bar{\vz}^{(r)}), \frac{1}{N} \sum_{i=1}^N \nabla f_i (\vx^{(r)}_i) \right\rangle \\
    &\qquad + \frac{L \eta^2}{2 (1-\beta)^2} \left\| \frac{1}{N} \sum_{i=1}^N \nabla f_i (\vx^{(r)}_i) \right\|^2 + \frac{L \sigma^2 \eta^2}{2 N (1-\beta)^2} \\
    &= f (\bar{\vz}^{(r)}) 
    + \frac{\eta}{1-\beta} \underbrace{\left\langle \nabla f (\bar{\vx}^{(r)}) - \nabla f (\bar{\vz}^{(r)}), \frac{1}{N} \sum_{i=1}^N \nabla f_i (\vx^{(r)}_i) \right\rangle}_{T_1} \\
    &\qquad - \frac{\eta}{1-\beta} \underbrace{\left\langle \nabla f (\bar{\vx}^{(r)}), \frac{1}{N} \sum_{i=1}^N \nabla f_i (\vx^{(r)}_i) \right\rangle}_{T_2}
    + \frac{L \eta^2}{2 (1-\beta)^2} \underbrace{\left\| \frac{1}{N} \sum_{i=1}^N \nabla f_i (\vx_i^{(r)}) \right\|^2}_{T_3}
    + \frac{L \sigma^2 \eta^2}{2 N (1-\beta)^2}.
\end{align*}
We can bound $T_1$ from above as follows:
\begin{align*}
    T_1 
    &\stackrel{(\ref{eq:inner_prod}), \gamma=2}{\leq} \left\| \nabla f (\bar{\vx}^{(r)}) - \nabla f (\bar{\vz}^{(r)}) \right\|^2 + \frac{1}{4} \left\| \frac{1}{N} \sum_{i=1}^N \nabla f_i (\vx^{(r)}_i) \right\|^2 \\
    &\stackrel{(\ref{eq:smoothness})}{\leq} L^2 \left\| \bar{\vx}^{(r)} - \bar{\vz}^{(r)} \right\|^2 + \frac{1}{4} \left\| \frac{1}{N} \sum_{i=1}^N \nabla f_i (\vx^{(r)}_i) \right\|^2.
\end{align*}
We can bound $-T_2$ from above as follows:
\begin{align*}
    - T_2 
    &= \frac{1}{2} \left\| \nabla f (\bar{\vx}^{(r)}) - \frac{1}{N} \sum_{i=1}^N \nabla f_i (\vx^{(r)}_i) \right\|^2
    - \frac{1}{2} \left\| \nabla f (\bar{\vx}^{(r)}) \right\|^2
    - \frac{1}{2} \left\| \frac{1}{N} \sum_{i=1}^N \nabla f_i (\vx^{(r)}_i) \right\|^2 \\
    &\stackrel{(\ref{eq:sum})}{\leq} \frac{1}{2} \frac{1}{N} \sum_{i=1}^N \left\| \nabla f_i (\bar{\vx}^{(r)}) - \nabla f_i (\vx^{(r)}_i) \right\|^2
    - \frac{1}{2} \left\| \nabla f (\bar{\vx}^{(r)}) \right\|^2
    - \frac{1}{2} \left\| \frac{1}{N} \sum_{i=1}^N \nabla f_i (\vx^{(r)}_i) \right\|^2 \\
    &\stackrel{(\ref{eq:smoothness})}{\leq} \frac{L^2}{2} \frac{1}{N} \sum_{i=1}^N \left\| \bar{\vx}^{(r)} - \vx^{(r)}_i \right\|^2
    - \frac{1}{2} \left\| \nabla f (\bar{\vx}^{(r)}) \right\|^2
    - \frac{1}{2} \left\| \frac{1}{N} \sum_{i=1}^N \nabla f_i (\vx^{(r)}_i) \right\|^2.
\end{align*}
Then, we can bound $T_3$ from above as follows:
\begin{align*}
    T_3
    &= \left\| \frac{1}{N} \sum_{i=1}^N \nabla f_i (\vx^{(r)}_i) \pm \nabla f (\bar{\vx}^{(r)}) \right\|^2 \\
    &\stackrel{(\ref{eq:triangle})}{\leq} 2 \left\| \frac{1}{N} \sum_{i=1}^N \nabla f_i (\vx^{(r)}_i) - \nabla f (\bar{\vx}^{(r)}) \right\|^2
    + 2 \left\| \nabla f (\bar{\vx}^{(r)}) \right\|^2 \\
    &\stackrel{(\ref{eq:sum})}{\leq} \frac{2}{N} \sum_{i=1}^N \left\| \nabla f_i (\vx^{(r)}_i) - \nabla f_i (\bar{\vx}^{(r)}) \right\|^2
    + 2 \left\| \nabla f (\bar{\vx}^{(r)}) \right\|^2 \\
    &\stackrel{(\ref{eq:smoothness})}{\leq} \frac{2 L^2}{N} \sum_{i=1}^N \left\| \vx^{(r)}_i - \bar{\vx}^{(r)} \right\|^2
    + 2 \left\| \nabla f (\bar{\vx}^{(r)}) \right\|^2.
\end{align*}
By combining them, we obtain
\begin{align*}
    &\mathbb{E}_r f (\bar{\vz}^{(r+1)}) \\
    &\leq f (\bar{\vz}^{(r)}) 
    + \frac{L^2 \eta}{1-\beta} \left\| \bar{\vx}^{(r)} - \bar{\vz}^{(r)} \right\|^2  
    - \frac{\eta}{4 (1 - \beta)}  \left\| \frac{1}{N} \sum_{i=1}^N \nabla f_i (\vx^{(r)}_i) \right\|^2 \\
    &\qquad + \frac{L^2}{1-\beta} \left( \frac{1}{2} + \frac{L \eta}{1-\beta} \right) \eta \Xi^{(r)} 
    - \frac{1}{1-\beta} \left( \frac{1}{2} - \frac{L \eta}{1-\beta} \right) \eta \left\| \nabla f (\bar{\vx}^{(r)}) \right\|^2 
    + \frac{L \sigma^2 \eta^2}{2 N (1-\beta)^2}.
\end{align*}
Using $\eta \leq \frac{1-\beta}{4 L}$, we get the statement.
\end{proof}

\begin{lemma}[Recursion for $\Xi$]
\label{lemma:recursion_for_consensus_distance}
Suppose that Assumptions \ref{assumption:lower_bound}, \ref{assumption:mixing}, \ref{assumption:smoothness}, and \ref{assumption:stochasic_gradient} hold.
Then, it holds that for any round $r \geq 0$,
\begin{align*}
    \Xi^{(r+1)} 
    &\leq (1 - \frac{p}{2}) \Xi^{(r)} 
    + \frac{9}{p} \eta^2 \mathcal{E}^{(r)}
    + \frac{9}{N p} \eta^2 \mathbb{E} \left\|  \mU^{(r+1)} - \mD^{(r+1)} \right\|^2_F
    +  \frac{9}{N p} \eta^2 \mathbb{E} \left\| \mE^{(r+1)} \right\|^2_F.
\end{align*}
\end{lemma}
\begin{proof}
Because $\sum_{i=1}^N \| \va_i - \bar{\va} \|^2 \leq \sum_{i=1}^N \| \va_i \|^2$ for any $\va_1, \cdots, \va_N \in \mathbb{R}^d$,
we have
\begin{align*}
    N \Xi^{(r)} = \mathbb{E} \left\| (\mX^{(r)} - \bar{\mX}^{(r-1)}) + (\bar{\mX}^{(r-1)} - \bar{\mX}^{(r)}) \right\|^2_F \leq  \mathbb{E} \left\| \mX^{(r)} - \bar{\mX}^{(r-1)} \right\|^2_F.
\end{align*}
Then, we have
\begin{align*}
    \left\| \mX^{(r+1)} - \bar{\mX}^{(r+1)} \right\|^2_F
    &\leq \left\| \mX^{(r+1)} - \bar{\mX}^{(r)} \right\|^2_F \\
    &= \left\| \mX^{(r)} \mW - \eta (\mU^{(r+1)} - \mC^{(r)}) - \bar{\mX}^{(r)} \right\|^2_F \\
    &\stackrel{(\ref{eq:triangle})}{\leq} (1 + \gamma) \left\| \mX^{(r)} \mW - \bar{\mX}^{(r)} \right\|^2_F + (1 + \gamma^{-1}) \eta^2 \left\|  \mU^{(r+1)} - \mC^{(r)} \right\|^2_F \\
    &\stackrel{(\ref{eq:mixing})}{\leq} (1 + \gamma) (1 - p) \left\| \mX^{(r)} - \bar{\mX}^{(r)} \right\|^2_F + (1 + \gamma^{-1}) \eta^2 \left\|  \mU^{(r+1)} - \mC^{(r)} \right\|^2_F.
\end{align*}
By substituting $\gamma=\frac{p}{2}$ and using $p \leq 1$, we obtain
\begin{align*}
    &\left\| \mX^{(r+1)} - \bar{\mX}^{(r+1)} \right\|^2_F \\
    &\leq (1 - \frac{p}{2}) \left\| \mX^{(r)} - \bar{\mX}^{(r)} \right\|^2_F + \frac{3}{p} \eta^2 \left\|  \mU^{(r+1)} - \mC^{(r)} \right\|^2_F \\
    &= (1 - \frac{p}{2}) \left\| \mX^{(r)} - \bar{\mX}^{(r)} \right\|^2_F + \frac{3}{p} \eta^2 \left\|  \mU^{(r+1)} \pm \mD^{(r+1)} \pm \mE^{(r+1)} - \mC^{(r)} \right\|^2_F \\
    &\stackrel{(\ref{eq:sum})}{\leq} (1 - \frac{p}{2}) \left\| \mX^{(r)} - \bar{\mX}^{(r)} \right\|^2_F \\
    &\qquad + \frac{9}{p} \eta^2 \left\|  \mU^{(r+1)} - \mD^{(r+1)} \right\|^2_F + \frac{9}{p} \eta^2 \left\| \mD^{(r+1)}- \mC^{(r)} - \mE^{(r+1)} \right\|^2_F +  \frac{9}{p} \eta^2 \left\| \mE^{(r+1)} \right\|^2_F.
\end{align*}
This concludes the proof.
\end{proof}

\begin{lemma}[Recursion for $\mathcal{E}$]
\label{lemma:recursion_for_e}
Suppose that Assumptions \ref{assumption:lower_bound}, \ref{assumption:mixing}, \ref{assumption:smoothness}, and \ref{assumption:stochasic_gradient} hold.
Then, it holds that for any round $r \geq 0$,
\begin{align*}
    \mathcal{E}^{(r+1)} 
    &\leq (1 - \frac{p}{2}) \mathcal{E}^{(r)} 
    + \frac{18 \beta^2}{p} \mathcal{D}^{(r)}
    + \frac{24}{N p} \mathbb{E} \left\| \mU^{(r+1)} - \mD^{(r+1)} \right\|^2_F \\
    &\qquad + \frac{144 L^4}{p} \eta^2 \Xi^{(r)}
    + \frac{72 \beta^2 L^2}{p} \eta^2 \mathbb{E} \left\| \bar{\vu}^{(r)} \right\|^2 
    + \frac{144 L^2}{p} \eta^2 \mathbb{E} \left\| \nabla f (\bar{\vx}^{(r)}) \right\|^2
    + \frac{36 L^2 \sigma^2 \eta^2}{N p}.
\end{align*}

\end{lemma}
\begin{proof}
We have
\begin{align*}
    &\mathbb{E} \left\| \mD^{(r+2)} - \mC^{(r+1)} - \mE^{(r+2)} \right\|^2_F \\
    &= \mathbb{E} \left\| \mD^{(r+2)} - (\mC^{(r)} - \mU^{(r+1)}) \mW - \mU^{(r+1)} - \mE^{(r+2)} \pm \mD^{(r+1)} \pm \mD^{(r+1)} \mW \pm \mE^{(r+1)} \right\|^2_F \\
    &\stackrel{(\ref{eq:triangle}), (\ref{eq:sum})}{\leq} (1+\gamma) \mathbb{E} \left\| (\mD^{(r+1)} - \mC^{(r)}) \mW - \mE^{(r+1)} \right\|^2_F \\
    &\qquad + 2 (1+\gamma^{-1}) \mathbb{E} \left\| (\mU^{(r+1)} - \mD^{(r+1)}) (\mW - \mI) \right\|^2_F \\
    &\qquad + 2 (1 + \gamma^{-1}) \mathbb{E} \left\|  \mD^{(r+2)} - \mD^{(r+1)} + \mE^{(r+1)} - \mE^{(r+2)}  \right\|^2_F \\
    &\stackrel{(\ref{eq:mixing})}{\leq} (1+\gamma) (1-p) \mathbb{E} \left\| \mD^{(r+1)} - \mC^{(r)} - \mE^{(r+1)} \right\|^2_F \\
    &\qquad + 2 (1+\gamma^{-1}) \mathbb{E} \left\| (\mU^{(r+1)} - \mD^{(r+1)}) (\mW - \mI) \right\|^2_F \\
    &\qquad + 2 (1 + \gamma^{-1}) \mathbb{E} \left\|  \mD^{(r+2)} - \mD^{(r+1)} + \mE^{(r+1)} - \mE^{(r+2)}  \right\|^2_F,
\end{align*}
where we use Lemma \ref{lemma:sum_c} and $\mE^{(r+1)} = \frac{1}{N} \mD^{(r+1)} \mathbf{1}\mathbf{1}^\top$ in the last inequality.
Then, we have
\begin{align*}
    &\mathbb{E} \left\| \mD^{(r+2)} - \mC^{(r+1)} - \mE^{(r+2)} \right\|^2_F \\
    &\stackrel{(a)}{\leq} (1+\gamma) (1-p) \mathbb{E} \left\| \mD^{(r+1)} - \mC^{(r)} - \mE^{(r+1)} \right\|^2_F \\
    &\qquad + 2 (1+\gamma^{-1}) \mathbb{E} \left\| \mU^{(r+1)} - \mD^{(r+1)} \right\|^2_F \left\| \mW - \mI \right\|^2_{\mathrm{op}}\\
    &\qquad + 2 (1 + \gamma^{-1}) \mathbb{E} \left\|  \mD^{(r+2)} - \mD^{(r+1)} + \mE^{(r+1)} - \mE^{(r+2)}  \right\|^2_F \\
    &\stackrel{(b)}{\leq} (1+\gamma) (1-p) \mathbb{E} \left\| \mD^{(r+1)} - \mC^{(r)} - \mE^{(r+1)} \right\|^2_F 
    + 8 (1+\gamma^{-1}) \mathbb{E} \left\| \mU^{(r+1)} - \mD^{(r+1)} \right\|^2_F \\
    &\qquad + 2 (1 + \gamma^{-1}) \mathbb{E} \left\|  \mD^{(r+2)} - \mD^{(r+1)} + \mE^{(r+1)} - \mE^{(r+2)}  \right\|^2_F,
\end{align*}
where $\| \cdot \|_{\mathrm{op}}$ denotes the operator norm.
In (a), we use the following definition of the operator norm:
$\|\mW - \mI\|_{\mathrm{op}} \coloneqq \sup_{\hat\vv \in \mathbb{R}^{d} \setminus \{\vzero\}} \frac{\|(\mW - \mI)\hat\vv\|}{\|\hat\vv\|} \ge \frac{\|(\mW - \mI)\vv\|}{\|\vv\|}$ for any $\vv \in \mathbb{R}^{d} \setminus \{\vzero\}$.
In (b), we use Gershgorin circle theorem and the fact that $\mW$ is a mixing matrix.
Substituting $\gamma=\frac{p}{2}$, we obtain
\begin{align*}
    &\mathbb{E} \left\| \mD^{(r+2)} - \mC^{(r+1)} - \mE^{(r+2)} \right\|^2_F \\
    &\leq (1 - \frac{p}{2}) \mathbb{E} \left\| \mD^{(r+1)} - \mC^{(r)} - \mE^{(r+1)} \right\|^2_F 
    + \frac{24}{p} \mathbb{E} \left\| \mU^{(r+1)} - \mD^{(r+1)} \right\|^2_F \\
    &\qquad + \frac{6}{p} \underbrace{\mathbb{E} \left\|  \mD^{(r+2)} - \mD^{(r+1)} + \mE^{(r+1)} - \mE^{(r+2)} \right\|^2_F}_{T}.
\end{align*}
Then, we can bound $T$ from above by expanding $\mD^{(r+2)}$, $\mD^{(r+1)}$, $\mE^{(r+2)}$, and $\mE^{(r+1)}$ as follows:
\begin{align*}
    T 
    &\stackrel{(\ref{eq:sum})}{\leq} 3 \beta^2 \mathbb{E} \left\| \mD^{(r+1)} - \mD^{(r)} + \mE^{(r)} - \mE^{(r+1)} \right\|^2_F 
    + 3 \mathbb{E} \left\| \nabla f (\bar{\mX}^{(r+1)}) - \nabla f (\bar{\mX}^{(r)}) \right\|^2_F \\
    &\qquad + 3 \mathbb{E} \left\| \frac{1}{N} \nabla f (\bar{\mX}^{(r)}) \mathbf{1}\mathbf{1}^\top - \frac{1}{N} \nabla f (\bar{\mX}^{(r+1)}) \mathbf{1}\mathbf{1}^\top\right\|^2_F \\
    &\stackrel{(\ref{eq:smoothness})}{\leq} 3 \beta^2 \mathbb{E} \left\| \mD^{(r+1)} - \mD^{(r)} + \mE^{(r)} - \mE^{(r+1)} \right\|^2_F 
    + 6 L^2 \mathbb{E} \left\| \bar{\mX}^{(r+1)} - \bar{\mX}^{(r)} \right\|^2_F.
\end{align*}
Using Lemma \ref{lemma:x_bar_update}, we obtain
\begin{align*}
    T 
    &\leq 3 \beta^2 \mathbb{E} \left\| \mD^{(r+1)} - \mD^{(r)} + \mE^{(r)} - \mE^{(r+1)} \right\|^2_F \\
    &\qquad + N (24 L^4 \eta^2 \Xi^{(r)}
    + 12 \beta^2 L^2 \eta^2 \mathbb{E} \left\| \bar{\vu}^{(r)} \right\|^2 
    + 24 L^2 \eta^2 \mathbb{E} \left\| \nabla f (\bar{\vx}^{(r)}) \right\|^2
    + \frac{6 L^2 \sigma^2 \eta^2}{N}).
\end{align*}
This concludes the proof.
\end{proof}

\begin{lemma}[Recursion for $\mathcal{D}$]
\label{lemma:recursion_for_d}
Suppose that Assumptions \ref{assumption:lower_bound}, \ref{assumption:mixing}, \ref{assumption:smoothness}, and \ref{assumption:stochasic_gradient} hold.
Then, it holds that for any round $r \geq 0$,
\begin{align*}
    \mathcal{D}^{(r+1)} 
    &\leq \frac{2 \beta^2}{1+\beta^2} \mathcal{D}^{(r)}  
    + \frac{32 L^4 \eta^2}{1 - \beta^2} \Xi^{(r)}
    + \frac{16 L^2 \beta^2 \eta^2}{1 - \beta^2} \mathbb{E} \left\| \bar{\vu}^{(r)} \right\|^2 \\
    &\qquad + \frac{32 L^2 \eta^2}{1 - \beta^2} \mathbb{E} \left\| \nabla f (\bar{\vx}^{(r)}) \right\|^2
    + \frac{8 L^2 \sigma^2 \eta^2}{N (1 - \beta^2)}.
\end{align*}
\end{lemma}
\begin{proof}
We have
\begin{align*}
    &\mathbb{E} \left\| \mD^{(r+2)} - \mD^{(r+1)} - \mE^{(r+2)} + \mE^{(r+1)} \right\|^2_F \\
    &\stackrel{(\ref{eq:triangle}), (\ref{eq:sum})}{\leq} (1+\gamma) \beta^2 \mathbb{E} \left\| \mD^{(r+1)} - \mD^{(r)} + \mE^{(r)} - \mE^{(r+1)} \right\|^2_F \\
    &\qquad + 2 (1 + \gamma^{-1}) \mathbb{E} \left\| \nabla f (\bar{\mX}^{(r+1)}) - \nabla f (\bar{\mX}^{(r)}) \right\|^2_F \\
    &\qquad + 2 (1 + \gamma^{-1}) \mathbb{E} \left\| \frac{1}{N} \nabla f (\bar{\mX}^{(r)}) \mathbf{1}\mathbf{1}^\top - \frac{1}{N} \nabla f (\bar{\mX}^{(r+1)}) \mathbf{1}\mathbf{1}^\top\right\|^2_F \\
    &\stackrel{(\ref{eq:smoothness})}{\leq} (1+\gamma) \beta^2 \mathbb{E} \left\| \mD^{(r+1)} - \mD^{(r)} + \mE^{(r)} - \mE^{(r+1)} \right\|^2_F 
    + 4 L^2 (1 + \gamma^{-1}) \mathbb{E} \left\| \bar{\mX}^{(r+1)} - \bar{\mX}^{(r)} \right\|^2_F.
\end{align*}
Substituting $\gamma = \frac{1 - \beta^2}{1 + \beta^2}$, we obtain
\begin{align*}
    &\mathbb{E} \left\| \mD^{(r+2)} - \mD^{(r+1)} - \mE^{(r+2)} + \mE^{(r+1)} \right\|^2_F \\
    &\leq \frac{2 \beta^2}{1+\beta^2} \mathbb{E} \left\| \mD^{(r+1)} - \mD^{(r)} + \mE^{(r)} - \mE^{(r+1)} \right\|^2_F 
    + \frac{8 L^2}{1-\beta^2} \mathbb{E} \left\| \bar{\mX}^{(r+1)} - \bar{\mX}^{(r)} \right\|^2_F.
\end{align*}
Using Lemma \ref{lemma:x_bar_update}, we obtain
\begin{align*}
    &\mathbb{E} \left\| \mD^{(r+2)} - \mD^{(r+1)} - \mE^{(r+2)} + \mE^{(r+1)} \right\|^2_F \\
    &\leq \frac{2 \beta^2}{1+\beta^2} \mathbb{E} \left\| \mD^{(r+1)} - \mD^{(r)} + \mE^{(r)} - \mE^{(r+1)} \right\|^2_F \\
    &\qquad + N (\frac{32 L^4 \eta^2}{1 - \beta^2} \Xi^{(r)}
    + \frac{16 L^2 \beta^2 \eta^2}{1 - \beta^2} \mathbb{E} \left\| \bar{\vu}^{(r)} \right\|^2 
    + \frac{32 L^2 \eta^2}{1 - \beta^2} \mathbb{E} \left\| \nabla f (\bar{\vx}^{(r)}) \right\|^2
    + \frac{8 L^2 \sigma^2 \eta^2}{N (1 - \beta^2)}).
\end{align*}
This concludes the proof.

\end{proof}

\begin{lemma}[Recursion for $\Xi$, $\mathcal{E}$, and $\mathcal{D}$]
\label{lemma:recusion_for_three_terms}
We define $t \in \mathbb{R}$ and $A \in \mathbb{R}$ as follows:
\begin{align*}
    t \coloneqq \frac{2 \beta^2 p}{1 - \beta^2} + 4, \;\;
    A \coloneqq \frac{648}{1 - \frac{p}{t} - \frac{2\beta^2}{1 + \beta^2}}.    
\end{align*}
Note that it holds that $t \geq 4$ and $A > 0$.
Suppose that Assumptions \ref{assumption:lower_bound}, \ref{assumption:mixing}, \ref{assumption:smoothness}, and \ref{assumption:stochasic_gradient} hold,
and step size $\eta$ satisfies
\begin{align*}  
    \eta \leq \frac{p}{8 L \sqrt{ 324 + \frac{2 A \beta^2}{1 - \beta^2} }}.
\end{align*}
Then, it holds that
\begin{align*}
    &\Xi^{(r+1)} +  \frac{36}{p^2} \eta^2 \mathcal{E}^{(r+1)} + \frac{A \beta^2}{p^3} \eta^2 \mathcal{D}^{(r+1)} \\
    &\leq (1 - \frac{p}{t}) (\Xi^{(r)} +  \frac{36}{p^2} \eta^2 \mathcal{E}^{(r)} + \frac{A \beta^2}{p^3} \eta^2 \mathcal{D}^{(r)}) \\
    &\qquad + \frac{1}{p} \eta^2 \mathbb{E} \left\| \nabla f (\bar{\vx}^{(r)}) \right\|^2
    + \frac{1}{N p} \left( 9 + \frac{864}{p^2} \right) \eta^2 \mathbb{E} \left\|  \mU^{(r+1)} - \mD^{(r+1)} \right\|^2_F \\
    &\qquad + \frac{L^2}{p^3} \left( 2592 \beta^2 + \frac{16 A \beta^4}{1 - \beta^2} \right) \eta^4 \mathbb{E} \left\| \bar{\vu}^{(r)} \right\|^2 
    + \frac{9}{N p} \eta^2 \mathbb{E} \left\| \mE^{(r+1)} \right\|^2_F
    + \frac{\sigma^2 \eta^2}{p}.
\end{align*}
\end{lemma}
\begin{proof}
From Lemmas \ref{lemma:recursion_for_consensus_distance} and \ref{lemma:recursion_for_e}, we have
\begin{align*}
    \Xi^{(r+1)} 
    &\leq (1 - \frac{p}{2}) \Xi^{(r)} 
    + \frac{9}{p} \eta^2 \mathcal{E}^{(r)}
    + \frac{9}{N p} \eta^2 \mathbb{E} \left\|  \mU^{(r+1)} - \mD^{(r+1)} \right\|^2_F
    +  \frac{9}{N p} \eta^2 \mathbb{E} \left\| \mE^{(r+1)} \right\|^2_F, \\
    \frac{36}{p^2} \eta^2 \mathcal{E}^{(r+1)} 
    &\leq (1 - \frac{p}{2}) \frac{36}{p^2} \eta^2 \mathcal{E}^{(r)} 
    + \frac{648 \beta^2}{p^3} \eta^2 \mathcal{D}^{(r)}
    + \frac{5184 L^4}{p^3} \eta^4 \Xi^{(r)}
    + \frac{2592 \beta^2 L^2}{p^3} \eta^4 \mathbb{E} \left\| \bar{\vu}^{(r)} \right\|^2  \\
    &\qquad + \frac{864}{N p^3} \eta^2 \mathbb{E} \left\| \mU^{(r+1)} - \mD^{(r+1)} \right\|^2_F
    + \frac{5184 L^2}{p^3} \eta^4 \mathbb{E} \left\| \nabla f (\bar{\vx}^{(r)}) \right\|^2
    + \frac{1296 L^2 \sigma^2 \eta^4}{N p^3}.
\end{align*}
Then, from Lemma \ref{lemma:recursion_for_d}, we have 
\begin{align*}
    \frac{A \beta^2}{p^3} \eta^2 \mathcal{D}^{(r+1)} 
    &\leq \frac{2 A \beta^4}{(1+\beta^2) p^3} \eta^2 \mathcal{D}^{(r)}  
    + \frac{32 A \beta^2 L^4}{(1 - \beta^2) p^3} \eta^4 \Xi^{(r)}
    + \frac{16 A L^2 \beta^4}{(1 - \beta^2) p^3} \eta^4 \mathbb{E} \left\| \bar{\vu}^{(r)} \right\|^2 \\
    &\qquad + \frac{32 A \beta^2 L^2}{(1 - \beta^2) p^3} \eta^4 \mathbb{E} \left\| \nabla f (\bar{\vx}^{(r)}) \right\|^2
    + \frac{8 A \beta^2 L^2 \sigma^2}{N (1 - \beta^2) p^3} \eta^4.
\end{align*}
Using $\eta^2 \leq \frac{p^2}{L^2}$ and $\eta^2 \leq \frac{p^2}{128 L^2 (162 + \frac{A \beta^2}{1 - \beta^2})}$, we have
\begin{align*}
    \left( (1 - \frac{p}{2}) + \frac{5184 L^4}{p^3} \eta^4 + \frac{32 A \beta^2 L^4}{(1 - \beta^2) p^3} \eta^4 \right) \Xi^{(r)}
    \leq \left( (1 - \frac{p}{2}) + \frac{L^2}{4p}\eta^2 \right) \Xi^{(r)}
    \leq (1 - \frac{p}{4}) \Xi^{(r)}.
\end{align*}
In addition, we have
\begin{align*}
    \left( \frac{9}{p} \eta^2 + (1 - \frac{p}{2}) \frac{36}{p^2} \eta^2\right) \mathcal{E}^{(r)}
    &= (1 - \frac{p}{4}) \frac{36}{p^2} \eta^2 \mathcal{E}^{(r)}, \\
    \left( \frac{648 \beta^2}{p^3} \eta^2 + \frac{2 A \beta^4}{(1+\beta^2) p^3} \eta^2 \right) \mathcal{D}^{(r)}
    &= \left( \frac{648}{A} + \frac{2 \beta^2}{1+\beta^2}\right) \frac{A \beta^2}{p^3} \eta^2 \mathcal{D}^{(r)}
    = ( 1 - \frac{p}{t} ) \frac{A \beta^2}{p^3} \eta^2 \mathcal{D}^{(r)}.
\end{align*}
Then, using $t \geq 4$, we obtain
\begin{align*}
    &\Xi^{(r+1)} +  \frac{36}{p^2} \eta^2 \mathcal{E}^{(r+1)} + \frac{A \beta^2}{p^3} \eta^2 \mathcal{D}^{(r+1)} \\
    &\leq (1 - \frac{p}{t}) (\Xi^{(r)} +  \frac{36}{p^2} \eta^2 \mathcal{E}^{(r)} + \frac{A \beta^2}{p^3} \eta^2 \mathcal{D}^{(r)}) \\
    &\qquad + \frac{L^2}{p^3} \left( \frac{32 A \beta^2}{1 - \beta^2} + 5184 \right) \eta^4 \mathbb{E} \left\| \nabla f (\bar{\vx}^{(r)}) \right\|^2 \\
    &\qquad + \frac{1}{N p} \left( 9 + \frac{864}{p^2} \right) \eta^2 \mathbb{E} \left\|  \mU^{(r+1)} - \mD^{(r+1)} \right\|^2_F \\
    &\qquad + \frac{9}{N p} \eta^2 \mathbb{E} \left\| \mE^{(r+1)} \right\|^2_F \\
    &\qquad + \frac{L^2}{p^3} \left( 2592 \beta^2 + \frac{16 A \beta^4}{1 - \beta^2} \right) \eta^4 \mathbb{E} \left\| \bar{\vu}^{(r)} \right\|^2 \\
    &\qquad + \frac{L^2 \sigma^2}{N p^3} \left( 1296 + \frac{8 A \beta^2}{1 - \beta^2} \right) \eta^4.
\end{align*}
Using $\eta^2 \leq \frac{p^2}{32 L^2 (162 + \frac{A \beta^2}{1 - \beta^2})}$, we have
\begin{align*}
    \frac{L^2}{p^3} \left( \frac{32 A \beta^2}{1 - \beta^2} + 5184 \right) \eta^4 \mathbb{E} \left\| \nabla f (\bar{\vx}^{(r)}) \right\|^2
    &\leq \frac{1}{p} \eta^2 \mathbb{E} \left\| \nabla f (\bar{\vx}^{(r)}) \right\|^2.
\end{align*}
Using $\eta^2 \leq \frac{p^2}{8 L^2 (162 + \frac{A \beta^2}{1 - \beta^2})}$, we obtain
\begin{align*}
    \frac{L^2 \sigma^2}{N p^3} \left( 1296 + \frac{8 A \beta^2}{1 - \beta^2} \right) \eta^4
    \leq \frac{\sigma^2 \eta^2}{N p} \leq \frac{\sigma^2 \eta^2}{p}.
\end{align*}
This concludes the proof.
\end{proof}

\begin{lemma}
\label{lemma:simplified_recursion}
We define $t \in \mathbb{R}$ and $A \in \mathbb{R}$ as follows:
\begin{align*}
    t \coloneqq \frac{2 \beta^2 p}{1 - \beta^2} + 4, \;\;
    A \coloneqq \frac{648}{1 - \frac{p}{t} - \frac{2\beta^2}{1 + \beta^2}}.    
\end{align*}
Under the same assumptions as those in Lemma \ref{lemma:recusion_for_three_terms}, 
it holds that
\begin{align*}
    \sum_{r=0}^{R} \Xi^{(r)}
    &\leq \frac{t}{p} \sum_{k=0}^{R} \Psi^{(k)} 
    + \frac{145 t \sigma^2 \eta^2}{(1-\beta)^2 p^3} R.
\end{align*}
where $\Psi^{(r)}$ is defined as follows:
\begin{align*}
    \Psi^{(r)} 
    &\coloneqq \frac{1}{p} \eta^2 \mathbb{E} \left\| \nabla f (\bar{\vx}^{(r)}) \right\|^2
    + \frac{1}{N p} \left( 9 + \frac{864}{p^2} \right) \eta^2 \mathbb{E} \left\|  \mU^{(r+1)} - \mD^{(r+1)} \right\|^2_F \\
    &\qquad + \frac{L^2}{p^3} \left( 2592 \beta^2 + \frac{16 A \beta^4}{1 - \beta^2} \right) \eta^4 \mathbb{E} \left\| \bar{\vu}^{(r)} \right\|^2 
    + \frac{9}{N p} \eta^2 \mathbb{E} \left\| \mE^{(r+1)} \right\|^2_F.
\end{align*}
\end{lemma}
\begin{proof}
We define $\Theta^{(r)} \coloneqq \Xi^{(r)} + \frac{36}{p^2} \eta^2 \mathcal{E}^{(r)} + \frac{A \beta^2}{p^3} \eta^2 \mathcal{D}^{(r)}$.
From Lemma \ref{lemma:recusion_for_three_terms}, we obtain
\begin{align*}
    \Theta^{(r+1)}
    &\leq (1 - \frac{p}{t}) \Theta^{(r)}
    + \Psi^{(r)} 
    + \frac{\sigma^2 \eta^2}{p} \\
    &\leq (1 - \frac{p}{t})^{r+1} \Theta^{(0)}
    + \sum_{k=0}^r (1 - \frac{p}{t})^{r-k} \Psi^{(k)} 
    + \sum_{k=0}^r (1 - \frac{p}{t})^{r-k} \frac{\sigma^2 \eta^2}{p}.
\end{align*}
Using $\sum_{k=0}^r (1 - \frac{p}{t})^{r-k} \leq \frac{t}{p}$, we obtain
\begin{align*}
    \Theta^{(r+1)}
    &\leq (1 - \frac{p}{t})^{r+1} \Theta^{(0)}
    + \sum_{k=0}^r (1 - \frac{p}{t})^{r-k} \Psi^{(k)} 
    + \frac{t \sigma^2 \eta^2}{p^2}.
\end{align*}
Then, for any $R \geq 1$, we obtain
\begin{align*}
    \sum_{r=1}^{R} \Theta^{(r)}
    &\leq \sum_{r=1}^{R} (1 - \frac{p}{t})^{r} \Theta^{(0)}
    + \sum_{r=1}^{R} \sum_{k=0}^{r-1} (1 - \frac{p}{t})^{r-k-1} \Psi^{(k)} 
    + \frac{t \sigma^2 \eta^2}{p^2} R \\
    &= \sum_{r=1}^{R} (1 - \frac{p}{t})^{r} \Theta^{(0)}
    + \sum_{k=0}^{R-1} \Psi^{(k)} \sum_{r=k+1}^{R} (1 - \frac{p}{t})^{r-k-1} 
    + \frac{t \sigma^2 \eta^2}{p^2} R \\
    &\leq \frac{t}{p} \Theta^{(0)}
    + \frac{t}{p} \sum_{k=0}^{R-1} \Psi^{(k)} 
    + \frac{t \sigma^2 \eta^2}{p^2} R,
\end{align*}
where we use $\sum_{r=1}^{R} (1 - \frac{p}{t})^{r}  \leq \frac{t}{p}$ and $\sum_{r=k+1}^{R} (1 - \frac{p}{t})^{r-k-1} \leq \frac{t}{p}$ in the last inequality.
Then, from the definition of $\mathcal{E}^{(r)}$, we have
\begin{align*}
    \mathcal{E}^{(0)} 
    &= \frac{1}{N} \mathbb{E} \left\| \mD^{(1)} - \mC^{(0)} - \mE^{(1)} \right\|^2_F \\
    &= \frac{1}{(1-\beta)^2} \frac{1}{N} \mathbb{E} \left\| \nabla f (\bar{\mX}^{(0)}) - \frac{1}{N} \nabla f (\bar{\mX}^{(0)}) \mathbf{1} \mathbf{1}^\top - \nabla F (\mX^{(0)} ; \xi^{(0)}) + \frac{1}{N} \nabla F (\mX^{(0)} ; \xi^{(0)}) \mathbf{1}\mathbf{1}^\top \right\|^2_F \\
    &\stackrel{(\ref{eq:sum})}{\leq} \frac{2}{(1-\beta)^2} \frac{1}{N} \mathbb{E} \left\| \nabla f (\bar{\mX}^{(0)}) - \nabla F (\mX^{(0)} ; \xi^{(0)}) \right\|^2_F \\
    &\qquad + \frac{2}{(1-\beta)^2} \frac{1}{N} \mathbb{E} \left\| \frac{1}{N} \nabla f (\bar{\mX}^{(0)}) \mathbf{1} \mathbf{1}^\top - \frac{1}{N} \nabla F (\mX^{(0)} ; \xi^{(0)}) \mathbf{1}\mathbf{1}^\top \right\|^2_F \\
    &\stackrel{(\ref{eq:stochastic_gradient})}{\leq} \frac{4}{(1-\beta)^2} \sigma^2,
\end{align*}
where we use $\mX^{(0)} = \bar{\mX}^{(0)}$ in the last inequality.
From the definition of $\mathcal{D}^{(r)}$, we have
\begin{align*}
    \mathcal{D}^{(0)} &= \frac{1}{N} \mathbb{E} \left\| (\beta-1) \mD^{(0)} + \nabla f (\bar{\mX}^{(0)}) - \frac{1}{N} \nabla f (\bar{\mX}^{(0)}) \mathbf{1}\mathbf{1}^\top \right\|^2_F = 0.
\end{align*}
Then using $\mX^{(0)} = \bar{\mX}^{(0)}$ (i.e., $\Xi^{(0)}=0$), we have
\begin{align*}
    \Theta^{(0)} \leq \frac{144 \sigma^2 \eta^2}{(1-\beta)^2 p^2}.
\end{align*}
Here, the above upper bounds of $\mathcal{E}^{(0)}$ and $\mathcal{D}^{(0)}$ are attributed to how we choose the initial values $\vu_i^{(0)}$, $\vc_i^{(0)}$, $\vd_i^{(0)}$, and $\ve_i^{(0)}$ for $i \in V$.
Then, combining them, we obtain
\begin{align*}
    \sum_{r=1}^{R} \Theta^{(r)}
    &\leq \frac{t}{p} \sum_{k=0}^{R-1} \Psi^{(k)} 
    + \frac{144 t \sigma^2 \eta^2}{(1-\beta)^2 p^3} 
    + \frac{t \sigma^2 \eta^2}{p^2} R \\
    &\leq \frac{t}{p} \sum_{k=0}^{R-1} \Psi^{(k)} 
    + \frac{145 t \sigma^2 \eta^2}{(1-\beta)^2 p^3} R,
\end{align*}
where we use $p \in (0, 1]$, $\beta \in [0, 1)$, and $R \geq 1$ in the last inequality.
Then, using $\Theta^{(r)} \geq \Xi^{(r)}$ and $\Xi^{(0)}=0$, we obtain the statement.
\end{proof}

\begin{lemma}
\label{lemma:sum_of_consensus}
We define $t \in \mathbb{R}$ and $A \in \mathbb{R}$ as follows:
\begin{align*}
    t \coloneqq \frac{2 \beta^2 p}{1 - \beta^2} + 4, \;\;
    A \coloneqq \frac{648}{1 - \frac{p}{t} - \frac{2\beta^2}{1 + \beta^2}}.    
\end{align*}
Note that it holds that $t \geq 4$ and $A > 0$.
Suppose that the same assumptions as those in Lemma \ref{lemma:recusion_for_three_terms} hold.
Then, if step size $\eta$ satisfies
\begin{align*}
    \eta \leq \min\{ \frac{p}{4 L \sqrt{324 + \frac{2 A \beta^2}{1 - \beta^2}}},
                     \frac{(1-\beta) p}{2 L \sqrt{t (5 + \frac{432}{p^2})}}, 
                     \frac{(1-\beta) p}{8 L \sqrt{5 t}} \},
\end{align*}
it holds that
\begin{align*}
    4 L^2 \sum_{r=0}^{R} \Xi^{(r)}
    &\leq \frac{1}{2} \sum_{k=0}^{R} \left\| \nabla f(\bar{\vx}^{(k)}) \right\|^2 
    + \frac{40 L^2 t}{(1-\beta)^3 p^2} \left( 10 + \frac{29}{p} + \frac{864}{p^2} \right) \sigma^2 \eta^2 (R + 1).
\end{align*}
\end{lemma}
\begin{proof}
We define $\Psi^{(r)}$ as follows:
\begin{align*}
    \Psi^{(r)} 
    &\coloneqq \frac{1}{p} \eta^2 \mathbb{E} \left\| \nabla f (\bar{\vx}^{(r)}) \right\|^2
    + \frac{1}{N p} \left( 9 + \frac{864}{p^2} \right) \eta^2 \mathbb{E} \left\|  \mU^{(r+1)} - \mD^{(r+1)} \right\|^2_F \\
    &\qquad + \frac{L^2}{p^3} \left( 2592 \beta^2 + \frac{16 A \beta^4}{1 - \beta^2} \right) \eta^4 \mathbb{E} \left\| \bar{\vu}^{(r)} \right\|^2 
    + \frac{9}{N p} \eta^2 \mathbb{E} \left\| \mE^{(r+1)} \right\|^2_F.
\end{align*}
Using $\bar{\ve}^{(r)} = \ve_i^{(r)}$ and Lemmas \ref{lemma:e}, \ref{lemma:u_minus_d}, and \ref{lemma:u_bar}, we obtain
\begin{align*}
    \Psi^{(r)} 
    &\leq \frac{1}{p} \eta^2 \mathbb{E} \left\| \nabla f (\bar{\vx}^{(r)}) \right\|^2 \\
    &\qquad + \eta^2 \frac{9}{(1-\beta) p} \sum_{k=0}^{r} \beta^{r-k} \left\| \nabla f(\bar{\vx}^{(k)}) \right\|^2 \\
    &\qquad + \frac{2 L^2}{(1-\beta) p^3} \left( 2592 \beta^2 + \frac{16 A \beta^4}{1 - \beta^2} \right) \eta^4 \left( \sum_{k=0}^{r-1} \beta^{r-k-1} \left\| \nabla f(\bar{\vx}^{(k)}) \right\|^2 \right) \\
    &\qquad + \frac{L^2}{(1-\beta) p} \left( 9 + \frac{864}{p^2} \right) \eta^2 \left( \sum_{k=0}^{r} \beta^{r-k} \Xi^{(k)} \right) \\
    &\qquad + \frac{2 L^4}{(1-\beta) p^3} \left( 2592 \beta^2 + \frac{16 A \beta^4}{1 - \beta^2} \right) \eta^4 \left( \sum_{k=0}^{r-1} \beta^{r-k-1} \Xi^{(k)} \right) \\
    &\qquad + \frac{5}{(1-\beta)^3 p} \left( 9 + \frac{864}{p^2} \right) \sigma^2 \eta^2 \\
    &\qquad + \frac{2 L^2}{N (1-\beta)^2 p^3} \left( 2592 \beta^2 + \frac{16 A \beta^4}{1 - \beta^2} \right) \sigma^2 \eta^4,
\end{align*}
for any round $r \geq 1$.
Then, we obtain
\begin{align*}
    \sum_{r=1}^{R} \Psi^{(r)} 
    &\leq \frac{1}{p} \eta^2 \sum_{r=1}^{R} \mathbb{E} \left\| \nabla f (\bar{\vx}^{(r)}) \right\|^2 \\
    &\qquad + \eta^2 \frac{9}{(1-\beta) p} \sum_{r=1}^{R} \sum_{k=0}^{r} \beta^{r-k} \left\| \nabla f(\bar{\vx}^{(k)}) \right\|^2 \\
    &\qquad + \frac{2 L^2}{(1-\beta) p^3} \left( 2592 \beta^2 + \frac{16 A \beta^4}{1 - \beta^2} \right) \eta^4 \left( \sum_{r=1}^{R} \sum_{k=0}^{r-1} \beta^{r-k-1} \left\| \nabla f(\bar{\vx}^{(k)}) \right\|^2 \right) \\
    &\qquad + \frac{L^2}{(1-\beta) p} \left( 9 + \frac{864}{p^2} \right) \eta^2 \left( \sum_{r=1}^{R} \sum_{k=0}^{r} \beta^{r-k} \Xi^{(k)} \right) \\
    &\qquad + \frac{2 L^4}{(1-\beta) p^3} \left( 2592 \beta^2 + \frac{16 A \beta^4}{1 - \beta^2} \right) \eta^4 \left( \sum_{r=1}^{R} \sum_{k=0}^{r-1} \beta^{r-k-1} \Xi^{(k)} \right) \\
    &\qquad + \frac{5}{(1-\beta)^3 p} \left( 9 + \frac{864}{p^2} \right) \sigma^2 \eta^2 R \\
    &\qquad + \frac{2 L^2}{N (1-\beta)^2 p^3} \left( 2592 \beta^2 + \frac{16 A \beta^4}{1 - \beta^2} \right) \sigma^2 \eta^4 R.
\end{align*}
Then, we obtain
\begin{align*}
    \sum_{r=1}^{R} \Psi^{(r)} 
    &\leq \frac{1}{p} \eta^2 \sum_{r=1}^{R} \mathbb{E} \left\| \nabla f (\bar{\vx}^{(r)}) \right\|^2 \\
    &\qquad + \eta^2 \frac{9}{(1-\beta) p} \sum_{k=0}^{R} \left\| \nabla f(\bar{\vx}^{(k)}) \right\|^2 \sum_{r= \max \{1, k \}}^{R} \beta^{r-k} \\
    &\qquad + \frac{2 L^2}{(1-\beta) p^3} \left( 2592 \beta^2 + \frac{16 A \beta^4}{1 - \beta^2} \right) \eta^4 \left( \sum_{k=0}^{R-1} \left\| \nabla f(\bar{\vx}^{(k)}) \right\|^2 \sum_{r=k+1}^{R} \beta^{r-k-1} \right) \\
    &\qquad + \frac{L^2}{(1-\beta) p} \left( 9 + \frac{864}{p^2} \right) \eta^2 \left( \sum_{k=0}^{R} \Xi^{(k)} \sum_{r= \max \{1, k \}}^{R} \beta^{r-k} \right) \\
    &\qquad + \frac{2 L^4}{(1-\beta) p^3} \left( 2592 \beta^2 + \frac{16 A \beta^4}{1 - \beta^2} \right) \eta^4 \left( \sum_{k=0}^{R-1} \Xi^{(k)} \sum_{r=k+1}^{R} \beta^{r-k-1} \right) \\
    &\qquad + \frac{5}{(1-\beta)^3 p} \left( 9 + \frac{864}{p^2} \right) \sigma^2 \eta^2 R \\
    &\qquad + \frac{2 L^2}{N (1-\beta)^2 p^3} \left( 2592 \beta^2 + \frac{16 A \beta^4}{1 - \beta^2} \right) \sigma^2 \eta^4 R.
\end{align*}
Using $\sum_{r=k+1}^R \beta^{r-k-1} \leq \frac{1}{1-\beta}$ and $\sum_{r=\max \{ 1, k\}}^R \beta^{r-k} \leq \frac{1}{1-\beta}$, we obtain
\begin{align*}
    \sum_{r=1}^{R} \Psi^{(r)} 
    &\leq \frac{1}{p} \eta^2 \sum_{r=1}^{R} \mathbb{E} \left\| \nabla f (\bar{\vx}^{(r)}) \right\|^2 \\
    &\qquad + \eta^2 \frac{9}{(1-\beta)^2 p} \sum_{k=0}^{R} \left\| \nabla f(\bar{\vx}^{(k)}) \right\|^2 \\
    &\qquad + \frac{2 L^2 \beta^2}{(1-\beta)^2 p^3} \left( 2592 + \frac{16 A \beta^2}{1 - \beta^2} \right) \eta^4 \left( \sum_{k=0}^{R-1} \left\| \nabla f(\bar{\vx}^{(k)}) \right\|^2  \right) \\
    &\qquad + \frac{L^2}{(1-\beta)^2 p} \left( 9 + \frac{864}{p^2} \right) \eta^2 \left( \sum_{k=0}^{R} \Xi^{(k)} \right) \\
    &\qquad + \frac{2 L^4 \beta^2}{(1-\beta)^2 p^3} \left( 2592 + \frac{16 A \beta^2}{1 - \beta^2} \right) \eta^4 \left( \sum_{k=0}^{R-1} \Xi^{(k)} \right) \\
    &\qquad + \frac{5}{(1-\beta)^3 p} \left( 9 + \frac{864}{p^2} \right) \sigma^2 \eta^2 R \\
    &\qquad + \frac{2 L^2 \beta^2}{N (1-\beta)^2 p^3} \left( 2592 + \frac{16 A \beta^2}{1 - \beta^2} \right) \sigma^2 \eta^4 R.
\end{align*}
Then, using $\bar{\vu}^{(0)} = \mathbf{0}$ and Lemmas \ref{lemma:e} and \ref{lemma:u_minus_d}, we have
\begin{align*}
    \Psi^{(0)} 
    &\leq \frac{1}{p} \eta^2 \mathbb{E} \left\| \nabla f (\bar{\vx}^{(0)}) \right\|^2
    + \frac{9}{(1-\beta) p} \eta^2 \left\| \nabla f(\bar{\vx}^{(0)}) \right\|^2
    + \frac{5}{(1-\beta)^3 p} \left( 9 + \frac{864}{p^2} \right) \sigma^2 \eta^2. 
\end{align*}
Then, we obtain
\begin{align*}
    \sum_{r=0}^{R} \Psi^{(r)} 
    &\leq \frac{1}{p} \eta^2 \sum_{r=0}^{R} \mathbb{E} \left\| \nabla f (\bar{\vx}^{(r)}) \right\|^2 \\
    &\qquad + \eta^2 \frac{18}{(1-\beta)^2 p} \sum_{k=0}^{R} \left\| \nabla f(\bar{\vx}^{(k)}) \right\|^2 \\
    &\qquad + \frac{2 L^2 \beta^2}{(1-\beta)^2 p^3} \left( 2592 + \frac{16 A \beta^2}{1 - \beta^2} \right) \eta^4 \left( \sum_{k=0}^{R-1} \left\| \nabla f(\bar{\vx}^{(k)}) \right\|^2  \right) \\
    &\qquad + \frac{L^2}{(1-\beta)^2 p} \left( 9 + \frac{864}{p^2} \right) \eta^2 \left( \sum_{k=0}^{R} \Xi^{(k)} \right) \\
    &\qquad + \frac{2 L^4 \beta^2}{(1-\beta)^2 p^3} \left( 2592 + \frac{16 A \beta^2}{1 - \beta^2} \right) \eta^4 \left( \sum_{k=0}^{R-1} \Xi^{(k)} \right) \\
    &\qquad + \frac{5}{(1-\beta)^3 p} \left( 9 + \frac{864}{p^2} \right) \sigma^2 \eta^2 (R + 1) \\
    &\qquad + \frac{2 L^2 \beta^2}{N (1-\beta)^2 p^3} \left( 2592 + \frac{16 A \beta^2}{1 - \beta^2} \right) \sigma^2 \eta^4 R.
\end{align*}
Using $\eta^2 \leq \frac{p^2}{32 L^2 (162 + \frac{A \beta^2}{1 - \beta^2})}$, we have
\begin{align*}
    \frac{2 L^2 \beta^2}{(1-\beta)^2 p^3} \left( 2592 + \frac{16 A \beta^2}{1 - \beta^2} \right) \eta^4 \left( \sum_{k=0}^{R-1} \left\| \nabla f(\bar{\vx}^{(k)}) \right\|^2  \right)
    \leq \frac{\beta^2}{(1-\beta)^2 p} \eta^2 \left( \sum_{k=0}^{R-1} \left\| \nabla f(\bar{\vx}^{(k)}) \right\|^2  \right).
\end{align*}
Using $\eta^2 \leq \frac{p^2}{32 L^2 (162 + \frac{A \beta^2}{1 - \beta^2})}$, we have
\begin{align*}
    \frac{2 L^4 \beta^2}{(1-\beta)^2 p^3} \left( 2592 + \frac{16 A \beta^2}{1 - \beta^2} \right) \eta^4 \left( \sum_{k=0}^{R-1} \Xi^{(k)} \right)
    \leq \frac{L^2 \beta^2}{(1-\beta)^2 p} \eta^2 \left( \sum_{k=0}^{R-1} \Xi^{(k)} \right).
\end{align*}
Using $\eta^2 \leq \frac{p^2}{32 L^2 (162 + \frac{A \beta^2}{1 - \beta^2})}$, we have
\begin{align*}
    \frac{2 L^2 \beta^2}{N (1-\beta)^2 p^3} \left( 2592 + \frac{16 A \beta^2}{1 - \beta^2} \right) \sigma^2 \eta^4 R
    \leq \frac{\beta^2}{N (1-\beta)^2 p} \sigma^2 \eta^2 R.
\end{align*}
Then, using $\beta \in [0,1)$ and $N \geq 1$, we obtain
\begin{align*}
    \sum_{r=0}^{R} \Psi^{(r)} 
    &\leq \frac{1}{p} \eta^2 \sum_{r=0}^{R} \mathbb{E} \left\| \nabla f (\bar{\vx}^{(r)}) \right\|^2 \\
    &\qquad + \eta^2 \frac{19}{(1-\beta)^2 p} \sum_{k=0}^{R} \left\| \nabla f(\bar{\vx}^{(k)}) \right\|^2 \\
    &\qquad + \frac{L^2}{(1-\beta)^2 p} \left( 10 + \frac{864}{p^2} \right) \eta^2 \left( \sum_{k=0}^{R} \Xi^{(k)} \right) \\
    &\qquad + \frac{5}{(1-\beta)^3 p} \left( 10 + \frac{864}{p^2} \right) \sigma^2 \eta^2 (R + 1).
\end{align*}
Using $\beta \in [0, 1)$ and Lemma \ref{lemma:simplified_recursion}, we obtain
\begin{align*}
    \sum_{r=0}^{R} \Xi^{(r)}
    &\leq \frac{20 t}{(1-\beta)^2 p^2} \eta^2 \sum_{k=0}^{R} \left\| \nabla f(\bar{\vx}^{(k)}) \right\|^2 \\
    &\qquad + \frac{t L^2}{(1-\beta)^2 p^2} \left( 10 + \frac{864}{p^2} \right) \eta^2 \left( \sum_{k=0}^{R} \Xi^{(k)} \right) \\
    &\qquad + \frac{5 t}{(1-\beta)^3 p^2} \left( 10 + \frac{864}{p^2} \right) \sigma^2 \eta^2 (R + 1)
    + \frac{145 t \sigma^2 \eta^2}{(1-\beta)^2 p^3} R \\
    &\leq \frac{20 t}{(1-\beta)^2 p^2} \eta^2 \sum_{k=0}^{R} \left\| \nabla f(\bar{\vx}^{(k)}) \right\|^2 \\
    &\qquad + \frac{t L^2}{(1-\beta)^2 p^2} \left( 10 + \frac{864}{p^2} \right) \eta^2 \left( \sum_{k=0}^{R} \Xi^{(k)} \right) \\
    &\qquad + \frac{5 t}{(1-\beta)^3 p^2} \left( 10 + \frac{29}{p} + \frac{864}{p^2} \right) \sigma^2 \eta^2 (R + 1).
\end{align*}
Then, using $\eta^2 \leq \frac{(1-\beta)^2 p^2}{4 t L^2 (5 + \frac{432}{p^2})}$, we obtain
\begin{align*}
    \frac{1}{2} \sum_{r=0}^{R} \Xi^{(r)}
    &\leq \frac{20 t}{(1-\beta)^2 p^2} \eta^2 \sum_{k=0}^{R} \left\| \nabla f(\bar{\vx}^{(k)}) \right\|^2 
    + \frac{5 t}{(1-\beta)^3 p^2} \left( 10 + \frac{29}{p} + \frac{864}{p^2} \right) \sigma^2 \eta^2 (R + 1).
\end{align*}
Multiplying $8 L^2$, we obtain
\begin{align*}
    4 L^2 \sum_{r=0}^{R} \Xi^{(r)}
    &\leq \frac{160 L^2 t}{(1-\beta)^2 p^2} \eta^2 \sum_{k=0}^{R} \left\| \nabla f(\bar{\vx}^{(k)}) \right\|^2 
    + \frac{40 L^2 t}{(1-\beta)^3 p^2} \left( 10 + \frac{29}{p} + \frac{864}{p^2} \right) \sigma^2 \eta^2 (R + 1).
\end{align*}
Using $\eta^2 \leq \frac{(1-\beta)^2 p^2}{320 L^2 t}$, we obtain the statement.
\end{proof}

\begin{lemma}
\label{lemma:sum_of_gradient}
We define $t \in \mathbb{R}$ as follows:
\begin{align*}
    t \coloneqq \frac{2 \beta^2 p}{1 - \beta^2} + 4.    
\end{align*}
Suppose that the assumptions of Lemma \ref{lemma:sum_of_consensus} hold.
Then, if step size $\eta$ satisfies
\begin{align*}
    \eta \leq  \frac{(1-\beta)^2}{2 \sqrt{2} L },
\end{align*}
it holds that
\begin{align*}
    \frac{1}{2 (R+1)} \sum_{r=0}^R \mathbb{E} \left\| \nabla f (\bar{\vx}^{(r)}) \right\|^2
    &\leq  \frac{4 (1-\beta)}{\eta (R+1)} \left( f (\bar{\vz}^{(0)}) - f^\star \right) 
    + \frac{2 L \sigma^2 \eta}{N (1-\beta)} \\
    &\qquad + \frac{L^2}{(1-\beta)^3} \left( \frac{40 t}{p^2} \left( 10 + \frac{29}{p} + \frac{864}{p^2} \right)
    + \frac{4 \beta^2}{N (1-\beta)} \right) \sigma^2 \eta^2.
\end{align*}
\end{lemma}
\begin{proof}
Using Lemma \ref{lemma:descent_lemma} and Assumption \ref{assumption:lower_bound}, we have
\begin{align*}
    \sum_{r=0}^R \mathbb{E} \left\| \nabla f (\bar{\vx}^{(r)}) \right\|^2
    &\leq  \frac{4 (1-\beta)}{\eta} \left( f (\bar{\vz}^{(0)}) 
    - f^\star \right) 
    + 4 L^2 \sum_{r=0}^R  \mathbb{E} \left\| \bar{\vx}^{(r)} - \bar{\vz}^{(r)} \right\|^2  
    + 4 L^2 \sum_{r=0}^R \Xi^{(r)} \\
    &\qquad - \sum_{r=0}^R \mathbb{E} \left\| \frac{1}{N} \sum_{i=1}^N \nabla f_i (\vx^{(r)}_i) \right\|^2
    + \frac{2 L \sigma^2 \eta}{N (1-\beta)} (R+1).
\end{align*}
From Lemma \ref{lemma:x_bar_minus_z_bar}, we have
\begin{align*}
    4 L^2 \sum_{r=0}^R \mathbb{E} \left\| \bar{\vx}^{(r)} - \bar{\vz}^{(r)} \right\|^2
    &\leq \frac{4 L^2 \beta^2 \eta^2 }{(1-\beta)^4} \sum_{r=0}^{R} \mathbb{E} \left\| \frac{1}{N} \sum_{i=1}^N \nabla f_i (\vx_i^{(r)}) \right\|^2
    + \frac{4 L^2 \beta^2 \sigma^2 \eta^2 }{N (1-\beta)^4} R.
\end{align*}
Combining them yields
\begin{align*}
    &\sum_{r=0}^R \mathbb{E} \left\| \nabla f (\bar{\vx}^{(r)}) \right\|^2 \\
    &\leq  \frac{4 (1-\beta)}{\eta} \left( f (\bar{\vz}^{(0)}) 
    - f^\star \right) 
    + 4 L^2 \sum_{r=0}^R \Xi^{(r)} 
    - \left( 1 - \frac{4 L^2 \beta^2 \eta^2 }{(1-\beta)^4}\right) \sum_{r=0}^R \mathbb{E} \left\| \frac{1}{N} \sum_{i=1}^N \nabla f_i (\vx^{(r)}_i) \right\|^2 \\
    &\qquad + \frac{2 L \sigma^2 \eta}{N (1-\beta)} (R+1)
    + \frac{4 L^2 \beta^2 \sigma^2 \eta^2 }{N (1-\beta)^4} R.
\end{align*}
Using $\eta^2 \leq \frac{(1-\beta)^4}{8 L^2}$ and $\beta < 1$, we obtain
\begin{align*}
    &\sum_{r=0}^R \mathbb{E} \left\| \nabla f (\bar{\vx}^{(r)}) \right\|^2 \\
    &\leq  \frac{4 (1-\beta)}{\eta} \left( f (\bar{\vz}^{(0)}) - f^\star \right) 
    + 4 L^2 \sum_{r=0}^R \Xi^{(r)}
    + \frac{2 L \sigma^2 \eta}{N (1-\beta)} (R+1)
    + \frac{4 L^2 \beta^2 \sigma^2 \eta^2 }{N (1-\beta)^4} R.
\end{align*}
Using Lemma \ref{lemma:sum_of_consensus}, we obtain
\begin{align*}
    &\frac{1}{2} \sum_{r=0}^R \mathbb{E} \left\| \nabla f (\bar{\vx}^{(r)}) \right\|^2 \\
    &\leq  \frac{4 (1-\beta)}{\eta} \left( f (\bar{\vz}^{(0)}) - f^\star \right) 
    + \frac{2 L \sigma^2 \eta}{N (1-\beta)} (R+1) \\
    &\qquad + \frac{L^2}{(1-\beta)^3} \left( \frac{40 t}{p^2} \left( 10 + \frac{29}{p} + \frac{864}{p^2} \right)
    + \frac{4 \beta^2}{N (1-\beta)} \right) \sigma^2 \eta^2 (R + 1).
\end{align*}
This concludes the proof.
\end{proof}

\begin{lemma}
\label{lemma;learning_rate}
We define $t \in \mathbb{R}$ and $A \in \mathbb{R}$ as follows:
\begin{align*}
    t \coloneqq \frac{2 \beta^2 p}{1 - \beta^2} + 4, \;\;
    A \coloneqq \frac{648}{1 - \frac{p}{t} - \frac{2\beta^2}{1 + \beta^2}}.    
\end{align*}
Then, it holds that
\begin{align*}
    \frac{(1-\beta)^2 p^2}{16 L \sqrt{\frac{7836 \beta^2}{(1-\beta^2)^3 p} + 282}}
    &\leq \min \left\{ \frac{1-\beta}{4 L}, 
                 \frac{p}{8 L \sqrt{324 + \frac{2 A \beta^2}{1 - \beta^2}}},
                 \frac{(1 - \beta) p^2}{2 L \sqrt{t (5 p^2 + 432)}}
                 \frac{(1-\beta) p}{8 L \sqrt{5 t}},
                 \frac{(1-\beta)^2}{2 \sqrt{2} L}
                \right\}.
\end{align*}
\end{lemma}
\begin{proof}
Because $\sqrt{\frac{7836 \beta^2}{(1-\beta^2)^3 p} + 282} > 1$, $p \in (0, 1]$, and $\beta \in [0, 1)$, we have
\begin{align*}
    \frac{(1-\beta)^2 p^2}{16 L \sqrt{\frac{7836 \beta^2}{(1-\beta^2)^3 p} + 282}}
    &\leq \min \left\{ \frac{1-\beta}{4 L}, 
                 \frac{(1-\beta)^2}{2 \sqrt{2} L}
                \right\}.
\end{align*}
From $p \leq 1$, we have
\begin{align*}
    t - 3 = \frac{2 \beta^2 p}{1 - \beta^2} + 1 
    \geq \frac{1 + \beta^2}{1 - \beta^2} p 
    = \frac{p}{1 - \frac{2 \beta^2}{1 + \beta^2}}.
\end{align*}
Then, we obtain
\begin{align*}
    1 - \frac{p}{t} - \frac{2 \beta^2}{1 + \beta^2}
    \geq \frac{p}{t-3} - \frac{p}{t}
    = \frac{3 p}{t (t-3)} 
    \geq \frac{3 p}{t^2}.
\end{align*}
Then, we obtain
\begin{align*}
    A \leq \frac{216 t^2}{p}.
\end{align*}
Using the above inequality, we obtain
\begin{align*}
    \frac{A \beta^2}{1 - \beta^2} + 162 \leq \frac{216 \beta^2 t^2}{p (1 - \beta^2)} + 162.
\end{align*}
From the definition of $t$, we obtain
\begin{align*}
    \frac{A \beta^2}{1 - \beta^2} + 30 t + 162
    &\leq \frac{216 \beta^2 t^2}{p (1 - \beta^2)} + 30 t + 162 \\
    &= \frac{216 \beta^2}{p (1 - \beta^2)} \left( \frac{2 \beta^2 p}{1 - \beta^2} + 4 \right)^2 + 30 \left( \frac{2 \beta^2 p}{1 - \beta^2} + 4 \right) + 162 \\
    &= \frac{216 \beta^2}{p (1 - \beta^2)} \left( \frac{4 \beta^4 p^2}{(1 - \beta^2)^2} + \frac{16 \beta^2 p}{1 - \beta^2} + 16 \right) + 30 \left( \frac{2 \beta^2 p}{1 - \beta^2} + 4 \right) + 162 \\
    &\leq \frac{7836 \beta^2}{p (1 - \beta^2)^3} + 282,
\end{align*}
where we use $\beta \in [0, 1)$ and $p \in (0, 1]$ in the last inequality.
Then, we obtain
\begin{align*}
    \frac{(1-\beta)^2 p^2}{16 L \sqrt{\frac{7836 \beta^2}{(1-\beta^2)^3 p} + 282}} 
    &\leq \frac{(1-\beta)^2 p^2}{16 L \sqrt{\frac{A \beta^2}{1 - \beta^2} + 30 t + 162}} \\
    &\leq \min \left\{
                 \frac{p}{8 L \sqrt{324 + \frac{2 A \beta^2}{1 - \beta^2}}},
                 \frac{(1 - \beta) p^2}{2 L \sqrt{t (5 p^2 + 432)}}
                 \frac{(1-\beta) p}{8 L \sqrt{5 t}},
               \right\}.
\end{align*}
This concludes the proof.
\end{proof}

\begin{lemma}[Convergence Rate for Non-convex Case]
Suppose that Assumptions \ref{assumption:lower_bound}, \ref{assumption:mixing}, \ref{assumption:smoothness}, and \ref{assumption:stochasic_gradient} hold.
Then, for any $R \geq 1$, there exists a step size $\eta$ such that it holds that
\begin{align*}
    &\frac{1}{R} \sum_{r=0}^{R-1} \mathbb{E} \left\| \nabla f (\bar{\vx}^{(r)}) \right\|^2 \\
    &\leq \mathcal{O} \left( \sqrt{ \frac{r_0 \sigma^2 L }{N R} }
    + \left( \frac{r_0^2 \sigma^2 L^2}{p^4 R^2 (1-\beta)} \left( 1 + \frac{p \beta^2}{1-\beta} \right) \right)^{\frac{1}{3}}  
    + \frac{L r_0}{(1-\beta) p^2 R} \sqrt{1 + \frac{\beta^2}{(1-\beta^2)^3 p}} \right),
\end{align*}
where $r_0 \coloneqq f(\bar{\vx}^{(0)}) - f^\star$.
\end{lemma}
\begin{proof}
From Lemmas \ref{lemma:sum_of_gradient} and \ref{lemma;learning_rate}, 
if the step size $\eta$ satisfies the following:
\begin{align*}
    \eta \leq \frac{(1-\beta)^2 p^2}{16 L \sqrt{\frac{7836 \beta^2}{(1-\beta^2)^3 p} + 282}},
\end{align*}
then we have
\begin{align*}
    &\frac{1}{2 (R+1)} \sum_{r=0}^R \mathbb{E} \left\| \nabla f (\bar{\vx}^{(r)}) \right\|^2 \\
    &\leq \frac{4}{\tilde{\eta} (R+1)} \left( f (\bar{\vz}^{(0)}) - f^\star \right) 
    + \frac{2 L \sigma^2 \tilde{\eta}}{N} 
    + \underbrace{\frac{L^2}{1-\beta} \left( \frac{40 t}{p^2} \left( 10 + \frac{29}{p} + \frac{864}{p^2} \right)
    + \frac{4 \beta^2}{N (1-\beta)} \right)}_{T} \sigma^2 \tilde{\eta}^2,
\end{align*}
where we define $\tilde{\eta} \coloneqq \frac{\eta}{1-\beta}$.
Then, we can bound $T$ from above as follows:
\begin{align*}
    T
    &= \frac{L^2}{1-\beta} \left( \frac{40}{p^2} \left( \frac{2 \beta^2 p}{1 - \beta^2} + 4 \right) \left( 10 + \frac{29}{p} + \frac{864}{p^2} \right)
    + \frac{4 \beta^2}{N (1-\beta)} \right) \\
    &\stackrel{p \in (0, 1]}{\leq} \frac{L^2}{1-\beta} \left( \frac{36120}{p^4} \left( \frac{2 \beta^2 p}{1 - \beta^2} + 4 \right) 
    + \frac{4 \beta^2}{N (1-\beta)} \right) \\
    &\stackrel{p \in (0, 1], \beta \in [0, 1)}{\leq} \frac{36120 L^2}{(1-\beta) p^4} \left( \frac{3 \beta^2 p}{1-\beta} + 4 \right).
\end{align*}
Then, we obtain
\begin{align*}
    &\frac{1}{2 (R+1)} \sum_{r=0}^R \mathbb{E} \left\| \nabla f (\bar{\vx}^{(r)}) \right\|^2 \\
    &\leq \frac{4}{\tilde{\eta} (R+1)} \left( f (\bar{\vz}^{(0)}) - f^\star \right) 
    + \frac{2 L \sigma^2 \tilde{\eta}}{N} 
    + \frac{36120 L^2}{(1-\beta) p^4} \left( \frac{3 \beta^2 p}{1-\beta} + 4 \right) \sigma^2 \tilde{\eta}^2.
\end{align*}
Using Lemma 17 in the previous work \citep{koloskova2020unified}, we obtain the statement.
\end{proof}

\newpage
\section{Hyperparameter Settings}
\label{sec:hyperparameter}

Tables~\ref{table:fashion_setting}, \ref{table:svhn_setting}, \ref{table:cifar_setting}, \ref{table:cifar_vgg_setting}, and \ref{table:cifar_resnet_setting} list the hyperparameter settings for each dataset.
We evaluated the performance of each comparison method for different step sizes and selected the step size that achieved the highest accuracy on the validation dataset.

\begin{table}[!h]
\caption{Experimental settings for FashionMNIST.}
\label{table:fashion_setting}
\vskip -0.1 in
    \centering
    \begin{tabular}{ll}
    \toprule
        Neural network architecture & LeNet \citep{lecun1998gradientbased} \\
        Normalization               & Group normalization \citep{wu2018group} \\
    \midrule
        Step size          & $\{ 0.005,0.001,0.0005 \}$ \\
        L2 penalty         & $0.001$ \\
        Batch size         & $100$ \\
        Data augmentation  & RandomCrop \\
        Total number of epochs & 500 \\
    \bottomrule
    \end{tabular}
\vskip -0.1 in
\end{table}
\begin{table}[!h]
\caption{Experimental settings for SVHN.}
\label{table:svhn_setting}
\vskip -0.1 in
\centering
    \begin{tabular}{ll}
    \toprule
        Neural network architecture & LeNet \citep{lecun1998gradientbased} \\
        Normalization               & Group normalization \citep{wu2018group} \\
    \midrule
        Step size          & $\{ 0.005,0.001,0.0005 \}$ \\
        L2 penalty         & $0.001$ \\
        Batch size         & $100$ \\
        Data augmentation  & RandomCrop \\
        Total number of epochs & 500 \\
    \bottomrule
    \end{tabular}
\vskip -0.1 in
\end{table}
\begin{table}[!h]
\caption{Experimental settings for CIFAR-10.}
\vskip -0.1 in
\label{table:cifar_setting}
    \centering
    \begin{tabular}{ll}
    \toprule
        Neural network architecture & LeNet \citep{lecun1998gradientbased} \\
        Normalization               & Group normalization \citep{wu2018group} \\
    \midrule
        Step size          & $\{ 0.005,0.001,0.0005 \}$ \\
        L2 penalty         & $0.001$ \\
        Batch size         & $100$ \\
        Data augmentation  & RandomCrop, RandomHorizontalFlip \\
        Total number of epochs & 500 \\
    \bottomrule
    \end{tabular}
\vskip -0.1 in
\end{table}
\begin{table}[!h]
\caption{Experimental settings for CIFAR-10 with VGG-11.}
\label{table:cifar_vgg_setting}
\vskip -0.1 in
    \centering
    \begin{tabular}{ll}
    \toprule
        Neural network architecture & VGG-11 \citep{simonyanZ2014very} \\
        Normalization               & Group normalization \citep{wu2018group} \\
    \midrule
        Step size          & $\{ 0.5, 0.1, 0.05, 0.01, 0.005 \}$ \\
        Step size decay    & $/10$ at epoch $500$ and $750$. \\
        L2 penalty         & $0.001$ \\
        Batch size         & $100$ \\
        Data augmentation  & RandomCrop, RandomHorizontalFlip, RandomErasing \\
        Total number of epochs & 1000 \\
    \bottomrule
    \end{tabular}
\vskip -0.1 in
\end{table}

\begin{table}[!h]
\caption{Experimental settings for CIFAR-10 with ResNet-34.}
\label{table:cifar_resnet_setting}
\vskip -0.1 in
    \centering
    \begin{tabular}{ll}
    \toprule
        Neural network architecture & ResNet-34 \citep{he2016deep} \\
        Normalization               & Group normalization \citep{wu2018group} \\
    \midrule
        Step size          & $\{ 0.5, 0.1, 0.05, 0.01, 0.005 \}$ \\
        Step size decay    & $/10$ at epoch $375$ and $563$. \\
        L2 penalty         & $0.001$ \\
        Batch size         & $100$ \\
        Data augmentation  & RandomCrop, RandomHorizontalFlip, RandomErasing \\
        Total number of epochs & $750$ \\
    \bottomrule
    \end{tabular}
\vskip -0.1 in
\end{table}

\end{document}

%% file: math_commands.tex
\usepackage{amsmath,amsfonts,bm}

\def\eqref#1{equation~\ref{#1}}
\def\1{\bm{1}}

\def\vzero{{\bm{0}}}

\def\va{{\bm{a}}}
\def\vb{{\bm{b}}}
\def\vc{{\bm{c}}}
\def\vd{{\bm{d}}}
\def\ve{{\bm{e}}}

\def\vg{{\bm{g}}}

\def\vu{{\bm{u}}}
\def\vv{{\bm{v}}}

\def\vx{{\bm{x}}}
\def\vy{{\bm{y}}}
\def\vz{{\bm{z}}}

\def\mA{{\bm{A}}}

\def\mC{{\bm{C}}}
\def\mD{{\bm{D}}}
\def\mE{{\bm{E}}}

\def\mI{{\bm{I}}}

\def\mU{{\bm{U}}}

\def\mW{{\bm{W}}}
\def\mX{{\bm{X}}}

\DeclareMathAlphabet{\mathsfit}{\encodingdefault}{\sfdefault}{m}{sl}
\SetMathAlphabet{\mathsfit}{bold}{\encodingdefault}{\sfdefault}{bx}{n}

%% file: main.bbl
\begin{thebibliography}{55}
\providecommand{\natexlab}[1]{#1}
\providecommand{\url}[1]{\texttt{#1}}
\expandafter\ifx\csname urlstyle\endcsname\relax
  \providecommand{\doi}[1]{doi: #1}\else
  \providecommand{\doi}{doi: \begingroup \urlstyle{rm}\Url}\fi

\bibitem[Assran et~al.(2019)Assran, Loizou, Ballas, and
  Rabbat]{assran2019stochstic}
Mahmoud Assran, Nicolas Loizou, Nicolas Ballas, and Mike Rabbat.
\newblock Stochastic gradient push for distributed deep learning.
\newblock In \emph{International Conference on Machine Learning}, 2019.

\bibitem[Carnevale et~al.(2022)Carnevale, Farina, Notarnicola, and
  Notarstefano]{carnevale2022gtadam}
Guido Carnevale, Francesco Farina, Ivano Notarnicola, and Giuseppe
  Notarstefano.
\newblock Gtadam: Gradient tracking with adaptive momentum for distributed
  online optimization.
\newblock In \emph{arXiv}, 2022.

\bibitem[Chen et~al.(2020)Chen, Kornblith, Norouzi, and Hinton]{chen2020simple}
Ting Chen, Simon Kornblith, Mohammad Norouzi, and Geoffrey Hinton.
\newblock A simple framework for contrastive learning of visual
  representations.
\newblock In \emph{International Conference on Machine Learning}, 2020.

\bibitem[Cutkosky \& Mehta(2020)Cutkosky and Mehta]{cutkosky2020momentum}
Ashok Cutkosky and Harsh Mehta.
\newblock Momentum improves normalized {SGD}.
\newblock In \emph{International Conference on Machine Learning}, 2020.

\bibitem[Defazio(2021)]{defazio2021momentum}
Aaron Defazio.
\newblock Momentum via primal averaging: Theoretical insights and learning rate
  schedules for non-convex optimization.
\newblock In \emph{arXiv}, 2021.

\bibitem[Devlin et~al.(2019)Devlin, Chang, Lee, and Toutanova]{devlin2019bert}
Jacob Devlin, Ming-Wei Chang, Kenton Lee, and Kristina Toutanova.
\newblock {BERT}: Pre-training of deep bidirectional transformers for language
  understanding.
\newblock In \emph{Conference of the North {A}merican Chapter of the
  Association for Computational Linguistics}, 2019.

\bibitem[Esfandiari et~al.(2021)Esfandiari, Tan, Jiang, Balu, Herron, Hegde,
  and Sarkar]{esfandiari2021cross}
Yasaman Esfandiari, Sin~Yong Tan, Zhanhong Jiang, Aditya Balu, Ethan Herron,
  Chinmay Hegde, and Soumik Sarkar.
\newblock Cross-gradient aggregation for decentralized learning from non-iid
  data.
\newblock In \emph{International Conference on Machine Learning}, 2021.

\bibitem[Gao \& Huang(2020)Gao and Huang]{gao2020periodic}
Hongchang Gao and Heng Huang.
\newblock Periodic stochastic gradient descent with momentum for decentralized
  training.
\newblock In \emph{arXiv}, 2020.

\bibitem[He et~al.(2016)He, Zhang, Ren, and Sun]{he2016deep}
Kaiming He, X.~Zhang, Shaoqing Ren, and Jian Sun.
\newblock Deep residual learning for image recognition.
\newblock In \emph{IEEE Conference on Computer Vision and Pattern Recognition},
  2016.

\bibitem[Hsieh et~al.(2020)Hsieh, Phanishayee, Mutlu, and
  Gibbons]{hsieh2020noniid}
Kevin Hsieh, Amar Phanishayee, Onur Mutlu, and Phillip Gibbons.
\newblock The non-{IID} data quagmire of decentralized machine learning.
\newblock In \emph{International Conference on Machine Learning}, 2020.

\bibitem[Hsu et~al.(2019)Hsu, Qi, and Brown]{hsu2019measuring}
Tzu-Ming~Harry Hsu, Qi, and Matthew Brown.
\newblock Measuring the effects of non-identical data distribution for
  federated visual classification.
\newblock In \emph{ArXiv}, 2019.

\bibitem[Karimireddy et~al.(2021)Karimireddy, Jaggi, Kale, Mohri, Reddi, Stich,
  and Suresh]{karimireddy2021breaking}
Sai~Praneeth Karimireddy, Martin Jaggi, Satyen Kale, Mehryar Mohri, Sashank~J.
  Reddi, Sebastian~U Stich, and Ananda~Theertha Suresh.
\newblock Breaking the centralized barrier for cross-device federated learning.
\newblock In \emph{Advances in Neural Information Processing Systems}, 2021.

\bibitem[Kingma \& Ba(2015)Kingma and Ba]{kingma2015adam}
Diederik~P. Kingma and Jimmy Ba.
\newblock Adam: A method for stochastic optimization.
\newblock In \emph{International Conference on Learning Representations}, 2015.

\bibitem[Koloskova et~al.(2020)Koloskova, Loizou, Boreiri, Jaggi, and
  Stich]{koloskova2020unified}
Anastasia Koloskova, Nicolas Loizou, Sadra Boreiri, Martin Jaggi, and Sebastian
  Stich.
\newblock A unified theory of decentralized {SGD} with changing topology and
  local updates.
\newblock In \emph{International Conference on Machine Learning}, 2020.

\bibitem[Koloskova et~al.(2021)Koloskova, Lin, and Stich]{koloskova2021an}
Anastasia Koloskova, Tao Lin, and Sebastian~U Stich.
\newblock An improved analysis of gradient tracking for decentralized machine
  learning.
\newblock In \emph{Advances in Neural Information Processing Systems}, 2021.

\bibitem[Krizhevsky(2009)]{Krizhevsky09learningmultiple}
Alex Krizhevsky.
\newblock Learning multiple layers of features from tiny images.
\newblock Technical report, 2009.

\bibitem[Le~Bars et~al.(2023)Le~Bars, Bellet, Tommasi, Lavoie, and
  Kermarrec]{le2023refined}
Batiste Le~Bars, Aur\'elien Bellet, Marc Tommasi, Erick Lavoie, and Anne-Marie
  Kermarrec.
\newblock Refined convergence and topology learning for decentralized sgd with
  heterogeneous data.
\newblock In \emph{International Conference on Artificial Intelligence and
  Statistics}, 2023.

\bibitem[LeCun et~al.(1998)LeCun, Bottou, Bengio, and
  Haffner]{lecun1998gradientbased}
Yann LeCun, Léon Bottou, Yoshua Bengio, and Patrick Haffner.
\newblock Gradient-based learning applied to document recognition.
\newblock In \emph{IEEE}, 1998.

\bibitem[Lian et~al.(2017)Lian, Zhang, Zhang, Hsieh, Zhang, and
  Liu]{lian2017can}
Xiangru Lian, Ce~Zhang, Huan Zhang, Cho-Jui Hsieh, Wei Zhang, and Ji~Liu.
\newblock Can decentralized algorithms outperform centralized algorithms? a
  case study for decentralized parallel stochastic gradient descent.
\newblock In \emph{Advances in Neural Information Processing Systems}, 2017.

\bibitem[Lian et~al.(2018)Lian, Zhang, Zhang, and Liu]{lian2018asynchronous}
Xiangru Lian, Wei Zhang, Ce~Zhang, and Ji~Liu.
\newblock Asynchronous decentralized parallel stochastic gradient descent.
\newblock In \emph{International Conference on Machine Learning}, 2018.

\bibitem[Lin et~al.(2021)Lin, Karimireddy, Stich, and Jaggi]{lin21quasi}
Tao Lin, Sai~Praneeth Karimireddy, Sebastian Stich, and Martin Jaggi.
\newblock Quasi-global momentum: Accelerating decentralized deep learning on
  heterogeneous data.
\newblock In \emph{International Conference on Machine Learning}, 2021.

\bibitem[Liu et~al.(2020{\natexlab{a}})Liu, Jiang, He, Chen, Liu, Gao, and
  Han]{liu2020On}
Liyuan Liu, Haoming Jiang, Pengcheng He, Weizhu Chen, Xiaodong Liu, Jianfeng
  Gao, and Jiawei Han.
\newblock On the variance of the adaptive learning rate and beyond.
\newblock In \emph{International Conference on Learning Representations},
  2020{\natexlab{a}}.

\bibitem[Liu et~al.(2021)Liu, Li, Wang, Tang, and Yan]{liu2021linear}
Xiaorui Liu, Yao Li, Rongrong Wang, Jiliang Tang, and Ming Yan.
\newblock Linear convergent decentralized optimization with compression.
\newblock In \emph{International Conference on Learning Representations}, 2021.

\bibitem[Liu et~al.(2020{\natexlab{b}})Liu, Gao, and Yin]{liu2020improved}
Yanli Liu, Yuan Gao, and Wotao Yin.
\newblock An improved analysis of stochastic gradient descent with momentum.
\newblock In \emph{Advances in Neural Information Processing Systems},
  2020{\natexlab{b}}.

\bibitem[Lorenzo \& Scutari(2016)Lorenzo and Scutari]{lorenzo2016next}
Paolo~Di Lorenzo and Gesualdo Scutari.
\newblock {NEXT:} in-network nonconvex optimization.
\newblock In \emph{IEEE Transactions on Signal and Information Processing over
  Networks}, 2016.

\bibitem[Lu \& De~Sa(2020)Lu and De~Sa]{lu2020moniqua}
Yucheng Lu and Christopher De~Sa.
\newblock Moniqua: Modulo quantized communication in decentralized {SGD}.
\newblock In \emph{International Conference on Machine Learning}, 2020.

\bibitem[Lu \& De~Sa(2021)Lu and De~Sa]{lu2021optimal}
Yucheng Lu and Christopher De~Sa.
\newblock Optimal complexity in decentralized training.
\newblock In \emph{International Conference on Machine Learning}, 2021.

\bibitem[Nedi{\'c} et~al.(2017)Nedi{\'c}, Olshevsky, and
  Shi]{nedic2017achieving}
Angelia Nedi{\'c}, Alexander Olshevsky, and Wei Shi.
\newblock Achieving geometric convergence for distributed optimization over
  time-varying graphs.
\newblock In \emph{SIAM Journal on Optimization}, 2017.

\bibitem[Neglia et~al.(2019)Neglia, Calbi, Towsley, and
  Vardoyan]{neglia2019role}
Giovanni Neglia, Gianmarco Calbi, Don Towsley, and Gayane Vardoyan.
\newblock The role of network topology for distributed machine learning.
\newblock In \emph{IEEE Conference on Computer Communications}, 2019.

\bibitem[Netzer et~al.(2011)Netzer, Wang, Coates, Bissacco, Wu, and
  Ng]{betzer2011reading}
Yuval Netzer, Tao Wang, Adam Coates, Alessandro Bissacco, Bo~Wu, and Andrew~Y.
  Ng.
\newblock Reading digits in natural images with unsupervised feature learning.
\newblock In \emph{Advances in Neural Information Processing Systems Workshop},
  2011.

\bibitem[Niwa et~al.(2020)Niwa, Harada, Zhang, and Kleijn]{niwa2020edge}
Kenta Niwa, Noboru Harada, Guoqiang Zhang, and W.~Bastiaan Kleijn.
\newblock Edge-consensus learning: Deep learning on p2p networks with
  nonhomogeneous data.
\newblock In \emph{International Conference on Knowledge Discovery and Data
  Mining}, 2020.

\bibitem[Niwa et~al.(2021)Niwa, Zhang, Kleijn, Harada, Sawada, and
  Fujino]{niwa2021asynchronous}
Kenta Niwa, Guoqiang Zhang, W.~Bastiaan Kleijn, Noboru Harada, Hiroshi Sawada,
  and Akinori Fujino.
\newblock Asynchronous decentralized optimization with implicit stochastic
  variance reduction.
\newblock In \emph{International Conference on Machine Learning}, 2021.

\bibitem[Polyak(1964)]{polyk1941some}
Boris~T. Polyak.
\newblock Some methods of speeding up the convergence of iteration methods.
\newblock In \emph{USSR Computational Mathematics and Mathematical Physics},
  1964.

\bibitem[Pu \& Nedic(2021)Pu and Nedic]{pu2021distributed}
Shi Pu and Angelia Nedic.
\newblock Distributed stochastic gradient tracking methods.
\newblock In \emph{Math. Program.}, 2021.

\bibitem[Qu \& Li(2018)Qu and Li]{qu2018harnessing}
Guannan Qu and Na~Li.
\newblock Harnessing smoothness to accelerate distributed optimization.
\newblock In \emph{IEEE Transactions on Control of Network Systems}, 2018.

\bibitem[Simonyan \& Zisserman(2015)Simonyan and Zisserman]{simonyanZ2014very}
Karen Simonyan and Andrew Zisserman.
\newblock Very deep convolutional networks for large-scale image recognition.
\newblock In \emph{International Conference on Learning Representations}, 2015.

\bibitem[Singh et~al.(2021)Singh, Data, George, and
  Diggavi]{singh2021communication}
Navjot Singh, Deepesh Data, Jemin George, and Suhas Diggavi.
\newblock Squarm-sgd: Communication-efficient momentum sgd for decentralized
  optimization.
\newblock In \emph{IEEE International Symposium on Information Theory}, 2021.

\bibitem[Takezawa et~al.(2022{\natexlab{a}})Takezawa, Niwa, and
  Yamada]{takezawa2022communication}
Yuki Takezawa, Kenta Niwa, and Makoto Yamada.
\newblock Communication compression for decentralized learning with operator
  splitting methods.
\newblock In \emph{arXiv}, 2022{\natexlab{a}}.

\bibitem[Takezawa et~al.(2022{\natexlab{b}})Takezawa, Niwa, and
  Yamada]{takezawa2022theoretical}
Yuki Takezawa, Kenta Niwa, and Makoto Yamada.
\newblock Theoretical analysis of primal-dual algorithm for non-convex
  stochastic decentralized optimization.
\newblock In \emph{arXiv}, 2022{\natexlab{b}}.

\bibitem[Tang et~al.(2018{\natexlab{a}})Tang, Gan, Zhang, Zhang, and
  Liu]{tang2018communication}
Hanlin Tang, Shaoduo Gan, Ce~Zhang, Tong Zhang, and Ji~Liu.
\newblock Communication compression for decentralized training.
\newblock In \emph{Advances in Neural Information Processing Systems},
  2018{\natexlab{a}}.

\bibitem[Tang et~al.(2018{\natexlab{b}})Tang, Lian, Yan, Zhang, and
  Liu]{tang2018d2}
Hanlin Tang, Xiangru Lian, Ming Yan, Ce~Zhang, and Ji~Liu.
\newblock $d^2$: Decentralized training over decentralized data.
\newblock In \emph{International Conference on Machine Learning},
  2018{\natexlab{b}}.

\bibitem[Vogels et~al.(2021)Vogels, He, Koloskova, Karimireddy, Lin, Stich, and
  Jaggi]{vogels2021relaysum}
Thijs Vogels, Lie He, Anastasia Koloskova, Sai~Praneeth Karimireddy, Tao Lin,
  Sebastian~U Stich, and Martin Jaggi.
\newblock Relay{S}um for decentralized deep learning on heterogeneous data.
\newblock In \emph{Advances in Neural Information Processing Systems}, 2021.

\bibitem[Wang et~al.(2019)Wang, Sahu, Yang, Joshi, and Kar]{wang2019matcha}
Jianyu Wang, Anit~Kumar Sahu, Zhouyi Yang, Gauri Joshi, and Soummya Kar.
\newblock Matcha: Speeding up decentralized sgd via matching decomposition
  sampling.
\newblock In \emph{Indian Control Conference}, 2019.

\bibitem[Wang et~al.(2020{\natexlab{a}})Wang, Tantia, Ballas, and
  Rabbat]{Wang2020SlowMo}
Jianyu Wang, Vinayak Tantia, Nicolas Ballas, and Michael Rabbat.
\newblock Slow{M}o: Improving communication-efficient distributed sgd with slow
  momentum.
\newblock In \emph{International Conference on Learning Representations},
  2020{\natexlab{a}}.

\bibitem[Wang et~al.(2020{\natexlab{b}})Wang, Lin, and
  Abernethy]{Wang2020escaping}
Jun-Kun Wang, Chi-Heng Lin, and Jacob Abernethy.
\newblock Escaping saddle points faster with stochastic momentum.
\newblock In \emph{International Conference on Learning Representations},
  2020{\natexlab{b}}.

\bibitem[Wu \& He(2018)Wu and He]{wu2018group}
Yuxin Wu and Kaiming He.
\newblock Group normalization.
\newblock In \emph{European Conference on Computer Vision}, 2018.

\bibitem[Xiao et~al.(2017)Xiao, Rasul, and Vollgraf]{xiao2017/online}
Han Xiao, Kashif Rasul, and Roland Vollgraf.
\newblock Fashion-mnist: a novel image dataset for benchmarking machine
  learning algorithms.
\newblock In \emph{arXiv}, 2017.

\bibitem[Xin \& Khan(2020)Xin and Khan]{xin2020distributed}
Ran Xin and Usman~A. Khan.
\newblock Distributed heavy-ball: A generalization and acceleration of
  first-order methods with gradient tracking.
\newblock In \emph{IEEE Transactions on Automatic Control}, 2020.

\bibitem[Xin et~al.(2022)Xin, Khan, and Kar]{xin2022fast}
Ran Xin, Usman~A. Khan, and Soummya Kar.
\newblock Fast decentralized nonconvex finite-sum optimization with recursive
  variance reduction.
\newblock In \emph{SIAM Journal on Optimization}, 2022.

\bibitem[Yan et~al.(2018)Yan, Yang, Li, Lin, and Yang]{yan2018unified}
Yan Yan, Tianbao Yang, Zhe Li, Qihang Lin, and Yi~Yang.
\newblock A unified analysis of stochastic momentum methods for deep learning.
\newblock In \emph{International Joint Conference on Artificial Intelligence},
  2018.

\bibitem[Ying et~al.(2021)Ying, Yuan, Chen, Hu, Pan, and
  Yin]{ying2021exponential}
Bicheng Ying, Kun Yuan, Yiming Chen, Hanbin Hu, Pan Pan, and Wotao Yin.
\newblock Exponential graph is provably efficient for decentralized deep
  training.
\newblock In \emph{Advances in Neural Information Processing Systems}, 2021.

\bibitem[Yu et~al.(2019)Yu, Jin, and Yang]{yu2019linear}
Hao Yu, Rong Jin, and Sen Yang.
\newblock On the linear speedup analysis of communication efficient momentum
  {SGD} for distributed non-convex optimization.
\newblock In \emph{International Conference on Machine Learning}, 2019.

\bibitem[Yuan et~al.(2021)Yuan, Chen, Huang, Zhang, Pan, Xu, and
  Yin]{yuan202decentlam}
Kun Yuan, Yiming Chen, Xinmeng Huang, Yingya Zhang, Pan Pan, Yinghui Xu, and
  Wotao Yin.
\newblock Decent{L}a{M}: Decentralized momentum sgd for large-batch deep
  training.
\newblock In \emph{International Conference on Computer Vision}, 2021.

\bibitem[Yuan et~al.(2022)Yuan, Huang, Chen, Zhang, Zhang, and
  Pan]{yuan2022revisiting}
Kun Yuan, Xinmeng Huang, Yiming Chen, Xiaohan Zhang, Yingya Zhang, and Pan Pan.
\newblock Revisiting optimal convergence rate for smooth and non-convex
  stochastic decentralized optimization.
\newblock In \emph{Advances in Neural Information Processing Systems}, 2022.

\bibitem[Zhao et~al.(2022)Zhao, Li, Li, Richtarik, and Chi]{zhao2022beer}
Haoyu Zhao, Boyue Li, Zhize Li, Peter Richtarik, and Yuejie Chi.
\newblock Beer: Fast $o(1/t)$ rate for decentralized nonconvex optimization
  with communication compression.
\newblock In \emph{Advances in Neural Information Processing Systems}, 2022.

\end{thebibliography}
